\documentclass{article}

\usepackage{microtype}
\usepackage{graphicx}
\usepackage{subfigure}
\usepackage{booktabs} 

\usepackage{hyperref}

\usepackage[accepted]{icml2023}

\usepackage{hyperref}
\usepackage{makecell}
\usepackage{amsmath}
\usepackage{amssymb}
\usepackage{mathtools}
\usepackage{amsthm}

\usepackage{algorithm}
\usepackage{algorithmic}

\usepackage[capitalize,noabbrev]{cleveref}

\theoremstyle{plain}
\newtheorem{theorem}{Theorem}[section]

\newtheorem{lemma}[theorem]{Lemma}
\newtheorem{corollary}[theorem]{Corollary}
\newtheorem{claim}[theorem]{Claim}
\theoremstyle{definition}
\newtheorem{definition}[theorem]{Definition}
\newtheorem{assumption}[theorem]{Assumption}
\theoremstyle{remark}

\DeclareMathOperator*{\argmax}{argmax}
\DeclareMathOperator*{\argmin}{argmin}

\DeclareMathOperator*{\E}{\mathbb{E}}

\newcommand{\eq}[1]{(\ref{eq:#1})}

\newcommand{\sect}[1]{\hyperref[sect:#1]{Section~\ref*{sect:#1}}}
\newcommand{\append}[1]{\hyperref[append:#1]{Appendix~\ref*{append:#1}}}
\newcommand{\lem}[1]{\hyperref[lem:#1]{Lemma~\ref*{lem:#1}}}
\newcommand{\thm}[1]{\hyperref[thm:#1]{Theorem~\ref*{thm:#1}}}
\newcommand{\algo}[1]{\hyperref[algo:#1]{Algorithm~\ref*{algo:#1}}}
\newcommand{\cor}[1]{\hyperref[cor:#1]{Corollary~\ref*{cor:#1}}}
\newcommand{\figg}[1]{\hyperref[fig:#1]{Figure~\ref*{fig:#1}}}
\newcommand{\tab}[1]{\hyperref[tab:#1]{Table~\ref*{tab:#1}}}
\newcommand{\ex}[1]{\hyperref[ex:#1]{Example~\ref*{ex:#1}}}
\newcommand{\defi}[1]{\hyperref[def:#1]{Definition~\ref*{def:#1}}}
\newcommand{\assump}[1]{\hyperref[assump:#1]{Assumption~\ref*{assump:#1}}}
\newcommand{\cla}[1]{\hyperref[cla:#1]{Claim~\ref*{cla:#1}}}
\theoremstyle{plain}
\newtheorem{innercustomthm}{Theorem}
\newenvironment{customthm}[1]
  {\renewcommand\theinnercustomthm{#1}\innercustomthm}
  {\endinnercustomthm}
  
\newtheorem{innercustomlem}{Lemma}
\newenvironment{customlem}[1]
  {\renewcommand\theinnercustomlem{#1}\innercustomlem}
  {\endinnercustomlem}

\newtheorem{innercustomcoro}{Corollary}
\newenvironment{customcoro}[1]
  {\renewcommand\theinnercustomcoro{#1}\innercustomcoro}
  {\endinnercustomcoro}

\newcommand{\compilehidecomments}{true} 
\ifthenelse{ \equal{\compilehidecomments}{false} }{%
	\newcommand{\zongqi}[1]{{\color{red}  [\text{Zongqi:} #1]}}
\newcommand{\wei}[1]{{\color{blue} [\text{Wei:} #1]}}
\newcommand{\jialin}[1]{{\color{cyan} [\text{Jialin:} #1]}}
\newcommand{\zhijie}[1]{{\color{yellow} [\text{Zhijie:} #1]}}        
}{
	\newcommand{\zongqi}[1]{}
\newcommand{\wei}[1]{}
\newcommand{\jialin}[1]{}
\newcommand{\zhijie}[1]{}
}

\newcommand\myatop[2]{\genfrac{}{}{0pt}{}{#1\hfill}{#2\hfill}}

\newcommand{\compilefullversion}{false} 
\ifthenelse{\equal{\compilefullversion}{false}}{%
	\newcommand{\OnlyInFull}[1]{}
	\newcommand{\OnlyInShort}[1]{#1}
}{
	\newcommand{\OnlyInFull}[1]{#1}%
	\newcommand{\OnlyInShort}[1]{}%
}

\allowdisplaybreaks[4]

\begin{document}

\twocolumn[
\icmltitle{Bandit Multi-linear DR-Submodular Maximization and Its Applications on Adversarial Submodular Bandits}



\icmlsetsymbol{equal}{*}

\begin{icmlauthorlist}
\icmlauthor{Zongqi Wan}{ict,ucas}
\icmlauthor{Jialin Zhang}{ict,ucas}
\icmlauthor{Wei Chen}{micro}
\icmlauthor{Xiaoming Sun}{ict,ucas}
\icmlauthor{Zhijie Zhang}{fzu}
\end{icmlauthorlist}

\icmlaffiliation{ict}{Institute of Computing Technology, Chinese Academy of Sciences}
\icmlaffiliation{ucas}{University of Chinese Academy of Sciences}
\icmlaffiliation{micro}{Microsoft Research}
\icmlaffiliation{fzu}{Center for Applied Mathematics of Fujian Province, School of Mathematics and Statistics, Fuzhou University}

\icmlcorrespondingauthor{Zhijie Zhang}{zzhang@fzu.edu.cn}

\icmlkeywords{DR-submodular, Adversarial Bandits, Combinatorial Bandits}

\vskip 0.3in
]



\printAffiliationsAndNotice{ } 

\begin{abstract}
We investigate the online bandit learning of the monotone multi-linear DR-submodular functions, designing the algorithm $\mathtt{BanditMLSM}$ that attains $O(T^{2/3}\log T)$ of $(1-1/e)$-regret. Then we reduce submodular bandit with partition matroid constraint and bandit sequential monotone maximization to the online bandit learning of the monotone multi-linear DR-submodular functions, attaining $O(T^{2/3}\log T)$ of $(1-1/e)$-regret in both problems, which improve the existing results. To the best of our knowledge, we are the first to give a sublinear regret algorithm for the submodular bandit with partition matroid constraint. A special case of this problem is studied by \citet{streeter2009online}. They prove a $O(T^{4/5})$ $(1-1/e)$-regret upper bound. For the bandit sequential submodular maximization, the existing work proves an $O(T^{2/3})$ regret with a suboptimal $1/2$ approximation ratio~\cite{niazadeh2021online}.
\end{abstract}

\section{Introduction}
Research on multi-armed bandit problems has developed rapidly in the last two decades. After the classical finite-arm bandit and linear bandit were well-studied in both stochastic and adversarial settings, people were starting to consider more general bandit problems. Submodular bandit is such an object being considered due to its ability to characterize the diminishing property of reward function in realistic applications. 

In submodular bandit, an optimizer/decision-maker needs to select a feasible subset each round, then a monotone submodular reward function of this round is determined
stochastically or adversarially, and the decision maker obtains a reward according to the reward function of this round. In this paper, we consider the adversarial setting. In the adversarial submodular bandit literature, the meta-action technique proposed by \citet{streeter2008online} is commonly used to obtain a sublinear regret algorithm. This technique employs online optimizers for each offline step of an offline algorithm to mimic it in an online manner. This technique reaches $O(T^{2/3})$ $(1-1/e)$-regret with the cardinality constraint~\cite{streeter2008online}. Subsequently, Streeter et al. applied this technique to the assignment constraint~\cite{streeter2009online}, which is a special partition matroid where the feasible set can only select one item from each partition. They obtained an $O(T^{4/5})$ $(1-1/e)$-regret in this situation. A recent work~\cite{niazadeh2021online} uses a Blackwell algorithm to turn offline greedy algorithms into online regret minimization algorithms. As an application, they reproduced the $O(T^{2/3})$ $(1-1/e)$-regret for the cardinality constraint.

In this paper, we present a different approach for the adversarial submodular bandit. We reduce the submodular bandit into a bandit multi-linear DR-submodular maximization problem. The DR-submodular function is a kind of non-convex function with theoretical guarantees in optimization, which has received much attention in recent years~\cite{bian2017continuous,bian2017guaranteed,niazadeh2020optimal}. There are several works considering the online full information or bandit feedback learning of the DR-submodular function~\cite{chen2018online,zhang2019online,raut2020online,thang2021online,zhang2022stochastic,zhang2022online}. DR-submodularity is inspired by the submodular set function, and we find it useful for designing submodular bandit algorithms due to its continuity. Specifically, we propose the function class called the multi-linear DR-submodular function. A multi-linear DR-submodular function is a DR-submodular function, and we additionally require it to be a multi-variable polynomial with the degree of each variable not exceeding $1$. We propose the algorithm $\mathtt{BanditMLSM}$ for the bandit maximization of this function class and reach the $(1-1/e)$-regret of $\widetilde{O}(T^{2/3})$, which is far better than the $O(T^{5/6})$ $(1-1/e)$-regret bound achieved on the general bandit DR-submodular maximization problem~\cite{niazadeh2021online}.
Multi-linear DR-submodular function captures the property of the multi-linear extension of a submodular set function. In fact, a multi-linear extension is a special case of multi-linear DR-submodular functions. 

Our next goal is to reduce discrete submodular bandit to bandit multi-linear DR-submodular maximization problem. A natural idea is to run $\mathtt{BanditMLSM}$ on the multi-linear extension of a submodular set function. However, this idea fails to reduce the submodular bandit to bandit multi-linear DR-submodular maximization problem. This is because the function value of the multi-linear extension cannot be estimated unbiasedly while the constraint is not trivial, which is because the value of the multi-linear extension may require obtaining feedback on a set function value $f(S)$ where the set $S$ is outside the constraint (e.g. the cardinality constraint). This is not allowed in the bandit feedback model.
To address this issue, we propose a new kind of continuous extension which is also multi-linear DR-submodular. Then we run $\mathtt{BanditMLSM}$ on that extension. We try our continuous approach on submodular bandit with partition matroid constraint and bandit sequential submodular maximization, generalizing and improving the previous results, see \sect{results}.

\paragraph{More related works} There is also some research on the stochastic submodular bandit. A model named linear submodular bandit has been studied~\cite{yue2011linear,chen2017interactive}. The model assumes that the reward function is a linear combination of several known submodular functions, only the weights of each submodular function are unknown to the decision maker, and the model requires the noisy marginal gain as the stochastic feedback. Many studies focus on the online influence maximization problem~\cite{vaswani2015influence,chen2016combinatorial,wang2017improving,wu2019factorization,li2020online,zhang2022OIM}, where the submodular reward function is induced by an information diffusion process on a social network. 
In this problem, different feedback models are studied. However, all the studies above assume extra information more than a \textit{full-bandit} feedback model where the decision maker can only observe the reward of the action played. We only notice two works studying the full-bandit feedback model: \cite{nie2022explore,nie2023framework}. \citet{nie2022explore} studied the bandit monotone submodular maximization with cardinality constraint, attaining $(1-1/e)$-regret of order $O(T^{2/3})$; \citet{nie2023framework} design a framework which adapts an $\alpha$-approximate offline algorithm into a stochastic bandit algorithm with $O(T^{2/3}(\log (T))^{1/3})$ $\alpha$-regret. The framework needs the offline algorithm to be robust to small errors. Besides the above, \citet{foster2021submodular} studied the submodular contextual bandit.

\paragraph{Remark on the stochastic submodular bandit} While \citet{nie2022explore} make in their paper a weaker assumption that the online reward function need not be a monotone submodular but only need to be monotone submodular in expectation, we find our algorithm can also be applied in this setting even our adversarial submodular bandit model requires the reward function to be monotone submodular. The key observation is, when we apply the adversarial submodular bandit problem on a stochastic submodular bandit environment, we should see the expected submodular function rather than the stochastically realized reward function as the online reward function selected by the adversary, and see the stochastic feedback as an unbiased estimate of the true value of the expected function. We will explain this in \append{remark}.

\subsection{Bandit Optimization Model}

Adversarial bandit optimization problems can be formalized as a repeated game between an optimizer and an adversary. 
The game lasts for $T$ rounds and $T$ is known to both players. In $t$-th round, the optimizer chooses an action $x_t$ from an action set $\mathcal{K}$, then the adversary chooses a reward function $f_t\in \mathcal{F}$. The action set $\mathcal{K}$ and the reward function set $\mathcal{F}$ are determined by specific bandit problems. Generally, $f_t$ maps $\mathcal{K}$ to a bounded interval $[0,M]\subseteq \mathbb{R}$. The optimizer gets reward $f_t(x_t)$ and it can only observe the value $f_t(x_t)$, which is called the \textit{bandit} feedback model. Sometimes people also call it the \textit{full-bandit} model to distinguish it from the semi-bandit model where the optimizer can observe more information, in this paper bandit and full-bandit are the same thing. 

In this paper, we consider \textit{oblivious} adversary, which means the reward functions $f_t$ can not be adaptively selected according to $x_1,x_2,\ldots,x_t$. In other words, we can think the adversary selects $f_t\in \mathcal{F}$ for each $1\leq t\leq T$ before the game starts and these functions are not revealed to the optimizer. Our goal is to design a strategy for the optimizer to minimize its cumulative \textit{$\alpha$-regret} during $T$ rounds,

\[\mathcal{R}_{\alpha}(T) = \max_{\boldsymbol{x}^* \in\mathcal{K}}\E\left[\sum_{t=1}^T\left(\alpha f_t(x^*)- f_t(x_t)\right)\right].\]

The action set $\mathcal{K}$ could be some structured set, maybe infinite or finite size. 
For the convenience of subsequent descriptions, we use $\mathcal{S}$ to denote the finite action set of the optimizer and use $\mathcal{K}$ to denote the infinite action set. Given $\mathcal{K}$ and $\mathcal{F}$, we call the game a $(\mathcal{K},\mathcal{F})$-bandit.


When we consider the bandit multi-linear monotone DR-submodular maximization and bandit DR-submodular maximization in \sect{BMMDSM} and \sect{DRSM}, we focus on the situation that $\mathcal{K}$ satisfies \assump{convex body}.
\begin{assumption}\label{assump:convex body}
 We assume $\mathcal{K}$ is a compact convex subset of $\mathbb{R}^d$ containing $\boldsymbol{0}$ , and $\mathcal{K}\subseteq D \mathbb{B}_d$ for some constant $D$, where $\mathbb{B}_d$ is a $d$-dimensional unit ball.
\end{assumption}

Before further describing the model we are considering, we give several definitions.

\begin{definition}[Monotonicity]
    There is a natural partial order on $\mathbb{R}^d$. For $\boldsymbol{x},\boldsymbol{y}\in\mathbb{R}^d$, if $x_i\geq y_i\ \forall i\in [d]$, then $\boldsymbol{x}\geq \boldsymbol{y}$. For a function $f:\mathbb{R}^d\rightarrow \mathbb{R}$, if for any $\boldsymbol{x}\geq \boldsymbol{y}$, $f(\boldsymbol{x})\geq f(\boldsymbol{y})$, we call $f$ a monotone function.
\end{definition}
\begin{definition}[DR-submodularity]
Let $\mathcal{X}=\prod_{i=1}^d \mathcal{X}_i$ be a subset of $\mathbb{R}^d$, where $\mathcal{X}_i$ is an interval $[0,a_i]$. A continuous function $f:\mathcal{X}\rightarrow \mathbb{R}_+$ is called a DR-submodular function if
for any $\boldsymbol{x}\geq \boldsymbol{y}$, $\lambda\in \mathbb{R}^+$, and the $i$-th base vector $\boldsymbol{e}_i$ for any $i\in [d]$, \[f(\boldsymbol{x}+\lambda \boldsymbol{e}_i )-f(\boldsymbol{x})\leq f(\boldsymbol{y}+\lambda \boldsymbol{e}_i )-f(\boldsymbol{y}).\]
 Moreover, if $f$ is second-order differentiable, then the DR-submodularity is equivalent to
$\frac{\partial^2 f}{\partial x_i \partial x_j}\leq 0, \forall i,j\in [d].$
\end{definition}

\begin{definition}[Multi-linearity]
We say function $f:\mathbb{R}^d \rightarrow \mathbb{R}$ a multi-linear function if $f$ is polynomial of $d$ variables, and for any variable $x_i$, the degree of $x_i$ in each term of $f$ is no more than $1$. 
\end{definition}


\begin{definition}[Lipschitz condition and smoothness]
Let $\|\cdot\|$ be the $L^2$-norm. For continuous differentiable function $f:\mathbb{R}^d \rightarrow \mathbb{R}$, if $|f(\boldsymbol{x})-f(\boldsymbol{y})|\leq L_1\|\boldsymbol{x}-\boldsymbol{y}\|$ for any $\boldsymbol{x},\boldsymbol{y}$, we say $f$ is $L_1$-lipschitz continuous. If $\|\nabla f(\boldsymbol{x})-\nabla f(\boldsymbol{y})\|\leq L_2\|\boldsymbol{x}-\boldsymbol{y}\|$, we say $f$ is $L_2$-smooth.
\end{definition}
With the above definitions, we consider two reward function sets in \sect{BMMDSM} and \sect{DRSM}: 
\begin{itemize}
    \item $\mathcal{F}_{DS}$: The set of monotone DR-submodular functions, which are $L_1$-lipschitz continuous, $L_2$-smooth, and $f(\boldsymbol{0})=0$. 
    \item $\mathcal{F}_{MDS}$: The set of monotone multi-linear DR-submodular functions, which are $L_1$-lipschitz continuous, and $f(\boldsymbol{0})=0$. 
\end{itemize}

\subsection{Our Results}\label{sect:results}

\begin{table*}[t]
\caption{Our results comparing to the previous results.
\wei{Can we put these notes below the table? When I read the caption, I was a bit confused what the caption is saying.}
\zongqi{ICML format instruction requires that the title of a table should be placed above the table.}
\wei{I do not mean the entire title. I mean that the title (first sentence) in above the table, but all others are not title. They are endnotes for the table, and usually are put at the bottom of the table.}\zongqi{I see, i moved the footnote below the table.}
}
\label{tab:results}
\centering
\begin{tabular}{|c|c|c|}
\hline
\textbf{Problem} & \textbf{Our $\alpha$ and regret} & \textbf{Previous $\alpha$ and regret}  \\ \hline
\multicolumn{3}{|c|}{\makecell[c]{Bandit DR-submodular maximization results }}
\\ \hline
\makecell[c]{Bandit Multi-linear Monotone \\ DR-Submodular Maximization} & \makecell[c]{\thm{MLSM} \\$1-1/e,O\left(d^{4/3}T^{2/3}\log (T)\right)$}  & $\backslash $ \\ 
\hline
\makecell[c]{Bandit Monotone \\ DR-Submodular Maximization} & \makecell[c]{\thm{DRSM}\\$1-1/e, O\left(d^{1/2}T^{3/4}\log (T)\right)$} & \makecell[c]{\cite{niazadeh2021online}\\$1-1/e, O(d(\log(d))^{1/6}T^{5/6})$}  \\ 
\hline
\multicolumn{3}{|c|}{Applications on adversarial submodular bandits}\\

\hline

\makecell[c]{Bandit Assignment Problem} & \makecell[c]{\cor{BMSMPM}$^\dag$\\$1-1/e, O\left((|G|)^{5/3} T^{2/3}\log(T)\right)$}& \makecell[c]{\cite{streeter2009online}\\$1-1/e,O\left(T^{4/5}\right)^\ddag$} \\ \hline
\makecell[c]{Bandit Monotone \\Submodular Maximization \\ over Partition Matroid} & \makecell[c]{\cor{BMSMPM}\\ $1-1/e, O\left(\left(\sum_{k=1}^K r_k|G_k|\right)^{5/3}T^{2/3}\log T\right)$} & $\backslash$ \\ \hline
\makecell[c]{Bandit Sequential \\Submodular Maximization$^\P$} & \makecell[c]{\cor{BSSM}\\$1-1/e, O\left(|G|^{10/3}T^{2/3}\log(T)\right)$} & \makecell[c]{\cite{niazadeh2021online}\\$1/2, O(|G|^{5/3} (\log (|G|))^{1/3} T^{2/3})$}  \\ 
\hline 
\end{tabular}
\begin{flushleft}
\footnotesize{$^\dag$ Bandit assignment problem is a special case of Bandit Monotone Submodular Maximization over Partition Matroid where $r_k=1$ and $G_k=G$, so this regret bound can be directly derived from \cor{BMSMPM}. $^\ddag$ In the original paper~\cite{streeter2009online}, the regret is written in the form which contains the optimal cumulative reward value, which will continue to be bounded to $T$ usually, leading to a bad dependent on $T$. So we re-trade off their regret and write it in terms of $T$ so that it can be compared with our regret bound. $^\P$ Compared with the setting in \cite{niazadeh2021online}, we actually add a new assumption that there is a dummy element in the ground set that always has $0$ marginal gain. This assumption can be satisfied easily in realistic applications, see \sect{BSSM}.}
\end{flushleft}

\end{table*}

We are the first to consider the $(\mathcal{K},\mathcal{F}_{MDS})$-bandit, i.e. \underline{B}andit \underline{M}onotone \underline{M}ulti-linear \underline{D}R-\underline{S}ubmodular \underline{M}aximization (BMMDSM). We observe that the gradient of multi-linear functions can be written as a linear combination of finite function values. Therefore, compared to the standard one-point gradient estimator proposed in~\cite{flaxman2005online}
	\wei{The acronym is unclear, and also need a citation here.}
 \zongqi{I've changed the description here to make it clearer.}
 which is used in previous works~\cite{zhang2019online,niazadeh2021online}, we propose a better one-point gradient estimator for monotone multi-linear DR-submodular functions. Along with other techniques including self-concordant barrier and non-oblivious technique, we propose the algorithm $\mathtt{BanditMLSM}$, which achieves $(1-1/e)$-regret of $\widetilde{O}(T^{2/3})$. Here the $\widetilde{O}$ hides the $\log T$ factor.

As a secondary result, we also improved the $(1-1/e)$-regret of general $(\mathcal{K},\mathcal{F}_{DS})$-bandit. This bandit is studied in \cite{zhang2019online,niazadeh2021online}, where they gave the $(1-1/e)$-regret bounds of $O(T^{8/9})$ and $O(T^{5/6})$ respectively. We proposed the algorithm $\mathtt{BanditDRSM}$ which achieves the $(1-1/e)$-regret of $\widetilde{O}(T^{3/4})$.  
Compared with their assumptions on functions, we add a new assumption that $f_t(\boldsymbol{0})=0$. Fortunately, this assumption is satisfied by many applications of DR-submodular maximization, including \emph{optimal budget allocation with continuous assignment}, \emph{senser placement}, \emph{softmax extension} and so on~\cite{bian2017continuous,bian2017guaranteed}. 
For the constraint set, they assume $\mathcal{K}$ is downward closed while we do not make this assumption.

Our main contribution is to propose a continuous approach for combinatorial full-bandit, for example, the case where online functions are submodular set functions and the constraint is a partition matroid. Talking about the continuous approach, a natural idea is reducing combinatorial bandit to Bandit Monotone Multi-linear DR-submodular Maximization using the classical multi-linear extension technique. For a submodular set function $g$ over the ground set $G=\{1,2,\ldots,n\}$, its multi-linear extension $f:[0,1]^{n}\rightarrow \mathbb{R}^+$ is defined as
\[f(\boldsymbol{x})=\sum_{S\subseteq G}g(S) \prod_{i\in S}x_i \prod_{i\notin S}(1-x_i).\]
From the definition of the multi-linear extension, we can see that it needs the value information of set function $g$ over all subsets of $G$. However, in the submodular bandit, one can only take the action which satisfies the constraint, thus the algorithm can not explore the value information outside the constraint. 
As a result, \citet{zhang2019online} proved that it is impossible to construct an unbiased estimate of $f$ and the gradient of $f$. That is to say, classical multi-linear extension is not a good candidate for our goal.

To overcome the above difficulties, we propose other continuous multi-linear DR-submodular extensions which require only the information of the feasible action. We select two submodular bandit problems to clarify our methodology: Bandit Monotone Submodular Maximization with Partition Matroid Constraint (BMSMPM) and Bandit Sequential Submodular Maximization (BSSM). The results are summarized in \tab{results}. Previous works have studied two special cases of BMSMPM: cardinality constraint \cite{streeter2008online} and the assignment problem \cite{streeter2009online}. 
We improve the regret bound of the bandit assignment problem and reproduce an $\widetilde{O}(T^{2/3})$ $(1-1/e)$-regret for cardinality constraint. 
To our best knowledge, we are the first to give a sublinear $(1-1/e)$-regret algorithm for bandit monotone submodular maximization with general partition matroid constraint. BSSM is motivated by maximizing user engagement on online retailing platforms. It is first studied in \cite{niazadeh2021online}, and their algorithm attains $O(T^{2/3})$ 
$1/2$-regret 
while the $1/2$ approximation ratio is not tight. We improve this result to $\widetilde{O}(T^{2/3})$ $(1-1/e)$-regret, leading to the tight approximation ratio.

In summary, we make the following contributions:
\begin{itemize}
    \item We are the first to study the bandit maximization of multi-linear monotone DR-submodular functions and propose a $\widetilde{O}(T^{2/3})$ $(1-1/e)$-regret algorithm.
    \item We improve the previous result of bandit maximization of general monotone DR-submodular functions to $\widetilde{O}(T^{3/4})$ $(1-1/e)$-regret by better exploiting the smoothness.
    \item We propose a continuous approach to reducing combinatorial bandit to multi-linear DR-submodular bandit. Using this continuous approach, we propose the first sublinear regret algorithm for submodular bandit with partition matroid constraint, which also improves the result of a previous work~\cite{streeter2009online} that studied the special case of this problem. We also improve the previous approximation ratio of Bandit Sequential Submodular Maximization from $1/2$ to tight $1-1/e$.
\end{itemize}

\section{Preliminary}
\subsection{Regularized Follow the Leader and Self-Concordant Functions}\label{sect:RFTL}

\underline{R}egularized \underline{F}ollow \underline{T}he \underline{L}eader(RFTL) is a commonly used algorithm for online optimization. While applying on a sequence of vector $\{\boldsymbol{g}_q\}_{q=1}^Q$ with constraint $\mathcal{K}$, RFTL outputs a sequence of point $\{\boldsymbol{x}_q\}_{q=1}^Q$, where
\begin{align*}
    \boldsymbol{x}_1 &= \argmin_{\boldsymbol{x}\in\mathcal{K}} \Phi(\boldsymbol{x})
    \\\boldsymbol{x}_{q+1} &= \argmin_{\boldsymbol{x}\in \mathcal{K}}\left(  \eta \sum_{s=1}^q \langle - \boldsymbol{g}_s,\boldsymbol{x}\rangle + \Phi(\boldsymbol{x})\right).
\end{align*}  
Here $\Phi(\boldsymbol{x})$ is an arbitrary regularizer, $\eta$ is a parameter. In this paper, we use a self-concordant barrier of $\mathcal{K}$ as the regularizer of RFTL. Self-concordant barrier was first proposed in convex optimization literature, and it was introduced to the bandit optimization problem in \cite{abernethy2008competing}.
\begin{definition}[Self-concordant Barrier \cite{hazan2016introduction}]
\label{def:self concordant}
    Let $\mathcal{K}\in \mathbb{R}^d$ be a convex set with non empty interior $\mbox{int}(\mathcal{K})$. We call the function $\Phi : \mbox{int} (\mathcal{K}) \longrightarrow \mathbb{R}$ a $\nu$-self-concordant barrier of $\mathcal{K}$ if:\\ (1) $\Phi$ is three-times continuously differentiable, convex, and approaches infinity along any sequence of points approaching the boundary of $\mathcal{K}$; \\(2) For every $\boldsymbol{h}\in\mathbb{R}^d$ and $\boldsymbol{x}\in\mbox{int}(\mathcal{K})$, the following holds:$|\nabla^3 \Phi(\boldsymbol{x})[\boldsymbol{h},\boldsymbol{h},\boldsymbol{h}]| \leq 2 (\nabla^2 \Phi(\boldsymbol{x})[\boldsymbol{h},\boldsymbol{h}])^{3/2}$, $|\nabla \Phi(x)[\boldsymbol{h}]|\leq \nu^{1/2} (\nabla^2 \Phi (\boldsymbol{x})[\boldsymbol{h},\boldsymbol{h}])^{1/2}$.
    where the third-order differential is defined as
    $\nabla^3 \Phi(\boldsymbol{x})[\boldsymbol{h},\boldsymbol{h},\boldsymbol{h}]:=\frac{\partial^3}{\partial t_1 \partial t_2 \partial t_3} \Phi (x+t_1 \boldsymbol{h} +t_2 \boldsymbol{h}+t_3 \boldsymbol{h})|_{t_1=t_2=t_3=0}$.
    \wei{What is the convention on boldface variables? Are they vectors? If so, should $x$ above be $\boldsymbol{x}$? Also should $\boldsymbol{h}$ be $\boldsymbol{h}$?}\zongqi{Yes, they are vectors. I will check this globally.}
\end{definition}

\begin{definition}[Local norm]
    The Hessian of self-concordant barrier induces a local norm at every $x\in \mbox{int}(\mathcal{K})$, denoted as $\|\cdot\|_{\Phi,x}$. We denote its dual norm as $\|\cdot \|_{\Phi,\boldsymbol{x},*}$. For any $\boldsymbol{v}\in \mathbb{R}^d$, 
    \begin{align*}
        \|\boldsymbol{v}\|_{\Phi,\boldsymbol{x}} &= \sqrt{\boldsymbol{v}^T \nabla^2 \Phi (\boldsymbol{x}) \boldsymbol{v}}
        \\\|\boldsymbol{v}\|_{\Phi,\boldsymbol{x},*} &=\sqrt{\boldsymbol{v}^T (\nabla^2 \Phi(\boldsymbol{x}))^{-1} \boldsymbol{v}}.
    \end{align*}
\end{definition}

The following theorem is proved in \cite{abernethy2008competing}. It shows that, if we set the regularizer to be a self-concordant barrier of $\mathcal{K}$ and the algorithm can access the unbiased estimator of $g_t$, then the regret of the generated solution sequence $\{x_q\}_{q=1}^Q$ can be bounded in terms of the local norm of the estimator.

\begin{theorem}[\cite{abernethy2008competing}]
    \label{thm:AHR}
    Let $\mathcal{K}$ be a convex set, $\Phi(x)$ be a self-concordant barrier on $\mathcal{K}$, $\{\widetilde{\boldsymbol{g}}_q\}_{q=1}^Q$ be a vector sequence. If $\widetilde{g}_q$ is an unbiased estimation of $g_q$, then running RFTL on vector sequence $\widetilde{g}_q$ with $\Phi(x)$ as the regularizer will produce a sequence of point $\{\boldsymbol{x}_q\}_{q=1}^Q, x_q\in \mathcal{K}$. For $\{\boldsymbol{x}_q\}_{q=1}^Q$ and any $\boldsymbol{y}\in\mathcal{K}$, we have
    \begin{align*}
        &\sum_{q=1}^Q \E \left[\langle \boldsymbol{g}_q,\boldsymbol{y}-\boldsymbol{x}_q\rangle\right]\\&\quad\leq \eta \sum_{q=1}^Q \E\left[\|\widetilde{\boldsymbol{g}}_q\|_{\Phi,\boldsymbol{x}_q,*}^2\right]
        +\frac{\Phi(\boldsymbol{y})-\Phi(\boldsymbol{x}_1)}{\eta}
    \end{align*}
\end{theorem}

\subsection{Ellipsoid Gradient Estimator}
Ellipsoid gradient estimator is proposed in \cite{abernethy2008competing}, where the authors use it along with the tool from \sect{RFTL} to design an $\widetilde{O}(\sqrt{T})$ regret algorithm for bandit linear optimization. For a continuous function $f:\mathbb{R}^d \rightarrow \mathbb{R}$ and an invertible matrix $\boldsymbol{H}\in \mathbb{R}^{d\times d}$, we define the $\boldsymbol{H}$-smoothed version of $f$.
\begin{definition}[$\boldsymbol{H}$-smoothed function]
    For function $f(\boldsymbol{x}):\mathbb{R}^{d}\rightarrow \mathbb{R}$ and invertible matrix $\boldsymbol{H}\in\mathbb{R}^{d\times d}$, we call $f^{\boldsymbol{H}}(\boldsymbol{x})$ an $\boldsymbol{H}$-smoothed version of $f(\boldsymbol{x})$, where
    \[f^{\boldsymbol{H}}(\boldsymbol{x}) = \E_{\boldsymbol{v}\sim \mathbb{B}_{d}}\left[f(\boldsymbol{x}+\boldsymbol{H}\boldsymbol{v})\right].\]
    Here $\boldsymbol{v}\sim \mathbb{B}_{d}$ means that $\boldsymbol{v}$ is sampled from the unit ball $\mathbb{B}_d$ uniformly at random. 
\end{definition}
There is a surprising fact that there is an unbiased estimator of $\nabla f^{\boldsymbol{H}}(\boldsymbol{x})$ for any $\boldsymbol{x}$, and the estimator uses only one query to the value oracle of $f$.
\begin{lemma}[Ellipsoid estimator~\cite{abernethy2008competing}]
    \label{lem:ellipsoid estimator}
        Let $\boldsymbol{H}\in\mathcal{R}^{d\times d}$ be an invertible matrix, $f(\boldsymbol{x}):\mathbb{R}^d \rightarrow \mathbb{R}$ be an arbitrary function. Then 
        \[\nabla f^{\boldsymbol{H}}(\boldsymbol{x}) = d\E_{\boldsymbol{v}\sim \mathbb{S}_{d-1}} \left[f(\boldsymbol{x}+\boldsymbol{H}\boldsymbol{v})\boldsymbol{H}^{-1}\boldsymbol{v}\right].\]
        Here $\boldsymbol{v}\sim \mathbb{S}_{d-1}$ means that $\boldsymbol{v}$ is sampled from the $(d-1)$-dimensional unit sphere $\mathbb{S}_{d-1}$ uniformly at random. 
\end{lemma}
For linear $f$, $f^{\boldsymbol{H}}(\boldsymbol{x})=f(\boldsymbol{x})$, so \lem{ellipsoid estimator} gives a one-sample unbiased estimator of the gradient of the linear function.
The ellipsoid gradient estimator is usually used along with RFTL with a self-concordant regularizer $\Phi$ of $\mathcal{K}$. When the invertible matrix $\boldsymbol{H}$ is set to be $(\nabla^2 \Phi(\boldsymbol{x}))^{-1/2}$ and $\boldsymbol{x}\in\mbox{int}(\mathcal{K})$, the sampled action $\boldsymbol{x}+\boldsymbol{H}\boldsymbol{v}$ is located in the surface of a so-called \textbf{Dikin ellipsoid} centered at $\boldsymbol{x}$, i.e. $\{\boldsymbol{x}'\mid \|\boldsymbol{x}'-\boldsymbol{x}\|_{\Phi,\boldsymbol{x}}\leq 1\}$. The fact that Dikin ellipsoid is entirely contained in $\mathcal{K}$ is useful for reducing regret.

\subsection{Non-oblivious Techniques for Monotone DR-Submodular Maximization}
The non-oblivious technique was first proposed to improve the approximation ratio of the solution returned by a local search algorithm. The idea is to run a local search on an auxiliary function rather than the original objective, and the local optima of the auxiliary function have a higher approximation ratio, thus the search algorithm will return a better solution.


In monotone submodular maximization literature, \citet{filmus2014monotone} improved the approximation ratio of the greedy algorithm to $1-1/e$ using the non-oblivious technique. \citet{zhang2022stochastic} generalized this result to the continuous DR-submodular maximization problem, improving the approximation ratio of projected gradient ascent to $1-1/e$. For a monotone DR-submodular function $f(\boldsymbol{x})$ satisfying $f(\boldsymbol{0})=0$, they consider following auxiliary function,
\begin{align}\label{eq:auxiliary function}
    F(\boldsymbol{x})=\int_0^1 \frac{e^{z-1}}{z} f(z\cdot \boldsymbol{x})dz.
\end{align}
We need the following lemma about the auxiliary function proved in their paper.
\begin{lemma}[Auxiliary function\cite{zhang2022stochastic}]\label{lem:auxiliary function}
    Let $f$ be a monotone DR-submodular function defined on $\mathcal{X}$ and $f(\boldsymbol{0})=0$, $\boldsymbol{x},\boldsymbol{y}\in \mathcal{X}$. Let $F$ be defined as \eq{auxiliary function}. Then 
    \begin{align}\label{eq:auxiliary gradient}
        \nabla F(\boldsymbol{x})= \int_0^1 e^{z-1}\nabla f(z\cdot \boldsymbol{x})dz
    \end{align}
    and the following inequality holds,
    \[\left\langle \boldsymbol{y}-\boldsymbol{x}, \nabla F(\boldsymbol{x})\right\rangle\geq (1-1/e)f(\boldsymbol{y})-f(\boldsymbol{x}).\]
\end{lemma}

\section{Bandit Monotone Multi-linear DR-Submodular Maximization}
\label{sect:BMMDSM}

\begin{algorithm}[t]
	\caption{$\mathtt{BanditMLSM}(\eta,L,\Phi)$}
	\label{algo:MLSM}

	\textbf{Input}: block size $L$, block number $Q=T/L$, learning rate $\eta$, self-concordant barrier $\Phi$
	
	\begin{algorithmic}[1]
	   \STATE initiate $\boldsymbol{x}_1\in \mbox{int} (\mathcal{K})$ such that $\nabla \Phi(\boldsymbol{x}_1) = 0$ 
	   \FOR{$q=1,2,\ldots,Q$}
	        \STATE Draw $t_q\sim \mbox{Unif}\{(q-1)L+1,(q-1)L+2,\ldots,qL\}$
	        \FOR{$t=(q-1)L+1, (q-1)L+2,\ldots, qL$}
	            \IF{$t=t_q$}
	                \STATE $\boldsymbol{H}_q = \left(\nabla^2 \Phi (\boldsymbol{x}_q)\right)^{-1/2}$
                        \STATE sample $z_q$ from $\mathbf{Z}$ where $P(\mathbf{Z}\leq z) =\int_{0}^{z} \frac{e^{u-1}}{1-e^{-1}}\mathbb{I}\left[u\in [0,1]\right]d u$
	                \STATE draw $\boldsymbol{v}_q \sim \mathbb{S}_{d-1}$
	                \STATE draw $\boldsymbol{u}_q$ from $\{\boldsymbol{0},\boldsymbol{e}_1,\boldsymbol{e}_2,\ldots,\boldsymbol{e}_d\}$
	                with probability: $\Pr(\boldsymbol{u}_q = \boldsymbol{0})= \frac{1}{2}$, $\Pr(\boldsymbol{u}_q = \boldsymbol{e}_i) = \frac{1}{2d}$
                    \STATE play $\boldsymbol{y}_{t_q} = z_q\cdot \boldsymbol{x}_q+z_q\langle \boldsymbol{H}_q\boldsymbol{v}_q,\boldsymbol{u}_q\rangle\boldsymbol{u}_q$
                    \STATE Set $\widetilde{l}_q(\boldsymbol{H}_q\boldsymbol{v}_q)$ as \eq{lq} 
                    \STATE $\widetilde{\nabla} \overline{F}_q(\boldsymbol{x}_q) \gets d\cdot \widetilde{l}_q(\boldsymbol{H}_q\boldsymbol{v}_q)\boldsymbol{H}_q^{-1}\boldsymbol{v}_q$  
	                \STATE $\boldsymbol{x}_{q+1} \gets \argmin\limits_{\boldsymbol{x}\in\mathcal{K}} \sum_{s=1}^q\langle -\eta \widetilde{\nabla}F_s(\boldsymbol{x}_s),\boldsymbol{x}\rangle + \Phi(\boldsymbol{x})$
	            \ELSE
	                \STATE play $\boldsymbol{y}_t = \boldsymbol{x}_q$
	            \ENDIF
	        \ENDFOR
	   \ENDFOR
	\end{algorithmic}
\end{algorithm}

In this section, we present our algorithm $\mathtt{BanditMLSM}$ for BMMDSM. The pseudo-code is shown in \algo{MLSM}. For some technical reason we will explain later, we divide the whole $T$ rounds into $Q$ equal-size blocks, and each block has $L$ consecutive rounds. Here $Q$ and $L$ are to be determined later, $L=T/Q$. without loss of generality, we assume both $L$ and $Q$ are integers. We define the average function $\overline{f}_q(\boldsymbol{x})$ of each block,
\begin{align}\label{eq:average}
    \overline{f}_q(\boldsymbol{x}) = \frac{1}{L}\sum_{t=(q-1)L+1}^{qL} f_t(\boldsymbol{x}).
\end{align}
Let $\overline{F}_q(\boldsymbol{x})$ be the auxiliary function of $\overline{f}_q(\boldsymbol{x})$,
\begin{align}\label{eq:average nonoblivious}
    \overline{F}_q(\boldsymbol{x}) = \int_{0}^1 \frac{e^{z-1}}{zL}\sum_{t=(q-1)L+1}^{qL} f_t(z\cdot\boldsymbol{x})dz.
\end{align}

In high level, $\mathtt{BanditMLSM}$ runs RFTL with a self-concordant regularizer $\Phi({\boldsymbol{x}})$ on the vector sequence $\{\nabla \overline{F}_q(\boldsymbol{x}_q)\}_{q=1}^Q$ and controls the regret w.r.t.~the linear function sequence $\{l_q\}_{q=1}^Q$ where $l_q(\boldsymbol{u}):=\langle \boldsymbol{u},\nabla \overline{F}_q(\boldsymbol{x}_q)\rangle$. Now the question is how to estimate $\nabla \overline{F}_q(\boldsymbol{x}_q)$. Recall \lem{ellipsoid estimator}, we can estimate $\nabla \overline{F}_q(\boldsymbol{x}_q)=\nabla l_q(\boldsymbol{0})$ with the ellipsoid estimator by querying one function value of $l_q(\boldsymbol{u})$. That is, we fix an invertible matrix $\boldsymbol{H}_q=(\nabla^2 \Phi(\boldsymbol{x}_q))^{-1/2}$, sample a random direction $\boldsymbol{v}_q$ in the $(d-1)$-dimensional sphere, then query $l_q(\boldsymbol{H}_q\boldsymbol{v}_q)$, and return $\widetilde{\nabla} \overline{F}_q(\boldsymbol{x}_q) := d\cdot l_q(\boldsymbol{H}_q\boldsymbol{v}_q)\boldsymbol{H}^{-1}_q\boldsymbol{v}_q$ as the estimate. 

The problem here is that we cannot query $l_q$ directly.
The algorithm can only query the function value of $f_t$ by playing the corresponding action in round $t$. We construct the unbiased estimator of $l_q(\boldsymbol{H}_q\boldsymbol{v}_q)$ as follows. First, we sample a uniformly random $t_q\in [(q-1)L+1, qL]\cap \mathbb{Z}$ and sample $z_q$ from the distribution $Z$ where $\Pr(Z\leq z) = \int_{0}^{z} \frac{e^{u-1}}{1-e^{-1}}\mathbb{I}\left[u\in [0,1]\right]d u$. Then we pick a vector $\boldsymbol{u}_q$ from the set $\{\boldsymbol{0},\boldsymbol{e}_1,\boldsymbol{e}_2,\ldots,\boldsymbol{e}_d\}$ following the distribution: $\Pr(\boldsymbol{u}_q = \boldsymbol{0})= \frac{1}{2}$, $\Pr(\boldsymbol{u}_q = \boldsymbol{e}_i) = \frac{1}{2d}$. Then we play $\boldsymbol{y}_{t_q}:=z_q \boldsymbol{x}_q + z_q\langle \boldsymbol{H}_q\boldsymbol{v}_q,\boldsymbol{u}_q\rangle\boldsymbol{u}_q$ in round $t_q$ to obtain the feedback $f_{t_q}(\boldsymbol{y}_{t_q})$. We replace $l_q(\boldsymbol{H}_q\boldsymbol{v}_q)$ with an estimate 
\begin{align}\label{eq:lq}
    \widetilde{l}_q(\boldsymbol{H}_q\boldsymbol{v}_q):=\left\{
        \begin{aligned}
            &-2(1-1/e)\frac{d}{z_q}\cdot f_{t_q}(\boldsymbol{y}_{t_q}) \quad \mbox{if }\boldsymbol{u}_q=\boldsymbol{0},\\
            &2(1-1/e)\frac{d}{z_q}\cdot f_{t_q}(\boldsymbol{y}_{t_q}) \quad \mbox{if } \boldsymbol{u}_q\neq \boldsymbol{0}.
        \end{aligned}
        \right.
\end{align}
If $z_q=0$, we define $ \widetilde{l}_q(\boldsymbol{H}_q\boldsymbol{v}_q):=0$. The following Lemma shows that $\widetilde{l}_q(\boldsymbol{H}_q\boldsymbol{v}_q)$ is an unbiased estimator of $l_q(\boldsymbol{H}_q\boldsymbol{v}_q)$. Its proof is deferred to \append{BMMDSM}. In the rounds other than $t_q$ in block $q$, we play $\boldsymbol{y}_t:=\boldsymbol{x}_q$ output by RFTL at the end of $(q-1)$-th block to exploit the regret bound of RFTL.
\begin{lemma}
\label{lem:multi-linear estimator}
    Let $\mathcal{H}_{q-1}$ be the history of the algorithm in the first $q$ blocks, that is, the realization of $t_s,z_s,\boldsymbol{v}_s,\boldsymbol{u}_s,\forall s\leq q$. Then $\E[\widetilde{l}_q(\boldsymbol{H}_q\boldsymbol{v}_q)\mid \mathcal{H}_{q-1},\boldsymbol{v}_q]= l_q(\boldsymbol{H}_q\boldsymbol{v}_q)$.
\end{lemma}

So $\widetilde{\nabla} \overline{F}_q(\boldsymbol{x})$ is actually defined as
\begin{align}
    \widetilde{\nabla} \overline{F}_q(\boldsymbol{x}) := d\cdot \widetilde{l}_q(\boldsymbol{H}_q\boldsymbol{v}_q)\boldsymbol{H}^{-1}_q\boldsymbol{v}_q.
\end{align}

We show that $\widetilde{\nabla} \overline{F}_q(\boldsymbol{x}_q)$ is an unbiased estimator of $\nabla \overline{F}_q(\boldsymbol{x}_q)$, and its dual local norm is $O(d^4)$ in the following lemma.
The proof is deferred to \append{BMMDSM}.
\begin{lemma}\label{lem:linear estimator}
   The following properties hold for $\widetilde{\nabla} \overline{F}_q(\boldsymbol{x}_q)$:
    \begin{itemize}
        \item[(i)] $\mathbb{E}\left[\widetilde{\nabla} \overline{F}_q(\boldsymbol{x}_q)\mid \mathcal{H}_{q-1}\right]=\nabla \overline{F}_q(\boldsymbol{x}_q)$,
        \item[(ii)] $\E\left[\|\widetilde{\nabla} \overline{F}_q(\boldsymbol{x}_q)\|_{\boldsymbol{x}_q,*}^2\mid \mathcal{H}_{q-1}\right]\leq 4(1-1/e)^2 L_1^2 D^2 d^4$.
    \end{itemize}
\end{lemma}

Note that, to estimate $\nabla \overline{F}_q(\boldsymbol{x}_q)$, we must sample an action that is far from $\boldsymbol{x}_q$.
This means we cannot do exploration and exploitation at the same time, 
	which is different from the linear bandit. 
This is the reason why previous works on bandit submodular maximization and our work divide rounds into blocks. We need to do the exploitation in most of the rounds of a block to maintain the regret bound. 

Now recall \thm{AHR}, running RFTL with a self-concordant barrier of $\mathcal{K}$ will generate a series of action $\{x_q\}_{q=1}^Q$, which has low regret w.r.t.~the linear function sequence $\langle \cdot\ ,\ \nabla \overline{F}_q(\boldsymbol{x}_q)\rangle$. $\overline{F}_q$ is the auxiliary function of the block average of $\{f_t\}_{t=1}^T$. In block $q$, our algorithm plays $\boldsymbol{y}_t=\boldsymbol{x}_q$ most of the time. Intuitively, the rerget of $\boldsymbol{y}_t$ w.r.t.~function sequence $\langle \cdot , \nabla F_t(\boldsymbol{y_t})\rangle$ is low, where $F_t$ is the auxiliary function of $f_t$. By \lem{auxiliary function}, we can bound the $(1-1/e)$-regret of $\mathtt{BanditMLSM}$. The proof of \thm{MLSM} is deferred to \append{BMMDSM}.
\begin{theorem}\label{thm:MLSM}
    Set $\eta = d^{-4}T^{-2/3}$, $L=d^{-2}T^{1/3}$ in \algo{MLSM}, if $\Phi$ is a $\nu$-self-concordant barrier of $\mathcal{K}$, then the expected $(1-1/e)$-regret of \algo{MLSM} can be bounded as
    \begin{align*}
        \mathcal{R}_{1-1/e}(T)&\leq O(\nu d^{4/3} T^{2/3}\log T ).
    \end{align*}
\end{theorem}

\paragraph{About the computational complexity} The computational cost mainly comes from two tasks: (1) Calculating the inverse and square root of the Hessian matrix of the regularizer; (2) Minimizing the convex function over a convex body. These tasks are commonly performed, so $\mathtt{BanditMLSM}$ can be implemented efficiently.

\section{Bandit DR-submodular Maximization}\label{sect:DRSM}

\OnlyInShort{
Combining RFTL with a self-concordant barrier and non-oblivious technique, we can also improve the result of the general bandit DR-submodular maximization problem where the online reward functions are not required to be multi-linear functions. Due to the space limitation, the algorithmic details and the proof are deferred to the \append{DRSM}. Here we only give the regret bound of our algorithm.
\begin{theorem}\label{thm:DRSM}
    If there is a $\nu$-self-concordant barrier of $\mathcal{K}$. Then there is an algorithm that attains the following regret upper bound in any $(\mathcal{K},\mathcal{F}_{DS})$-bandit instance:
    \[\mathcal{R}_{1-1/e}(T)\leq O(\nu d^{1/2}T^{3/4}\log T).\]
\end{theorem}
}

\OnlyInFull{
\begin{algorithm}[t]
	\caption{ $\mathtt{BanditDRSM}(\eta,\delta,L,\Phi)$}
	\label{algo:DRSM}

	\textbf{Input}: Smoothing radius $\delta$, block size $L$, block number $Q=T/L$, learning rate $\eta$, potential function $\Phi$
	
	\begin{algorithmic}[1]
	   \STATE initiate $\boldsymbol{x}_1\in \mbox{int} (\mathcal{K})$ such that $\nabla \Phi(\boldsymbol{x}_1) = 0$ 
	   \FOR{$q=1,2,\ldots,Q$}
	        \STATE Draw $t_q\sim \mbox{Unif}\{(q-1)L+1,(q-1)L+2,\ldots,qL\}$
	        \FOR{$t=(q-1)L+1, (q-1)L+2,\ldots, qL$}
	            \IF{$t=t_q$}
	                \STATE sample $z_q$ from $\mathbf{Z}$ where $P(\mathbf{Z}<z) =\int_{0}^{z} \frac{e^{u-1}}{1-e^{-1}}\mathbb{I}\left[u\in [0,1]\right]d u$
	                \STATE $\boldsymbol{H}_q = \left(\nabla^2 \Phi (z_q\cdot \boldsymbol{x}_q)\right)^{-1/2}$
	                \STATE draw $\boldsymbol{v}_q \sim \mathbb{S}_{d-1}$
	                \STATE play $\boldsymbol{y}_t = z_q\cdot \boldsymbol{x}_q+\delta z_q\cdot\boldsymbol{H}_q\boldsymbol{v}_q$
	                \STATE $\widetilde{\nabla} \overline{F}_q(\boldsymbol{x}_q) \gets (1-1/e)\frac{d}{\delta z_q}f_{t_q}(\boldsymbol{y}_t)\boldsymbol{H}_q^{-1}\boldsymbol{v}_q $
	                \STATE $\boldsymbol{x}_{q+1} \gets \argmin\limits_{\boldsymbol{x}} \sum_{s=1}^q\langle -\eta \widetilde{\nabla}\overline{F}_s(\boldsymbol{x}_s),\boldsymbol{x}\rangle + \Phi(\boldsymbol{x})$
	            \ELSE
	                \STATE play $\boldsymbol{y}_t = \boldsymbol{x}_q$
	            \ENDIF
	        \ENDFOR
	   \ENDFOR
	\end{algorithmic}
\end{algorithm}

In this section we present our algorithm $\mathtt{BanditDRSM}$ for general bandit monotone DR-submodular maximization, the pseudocode is shown in \algo{DRSM}. $\mathtt{BanditDRSM}$ is very similar from $\mathtt{BanditMLSM}$, it also divides $T$ rounds into $Q$ equal size blocks. We use again $\overline{f}_q(\boldsymbol{x})$ and $\overline{F}_q(\boldsymbol{x})$ to denote the average function of $q$-th block and the auxiliary function of it, defined as \eq{average} and \eq{average nonoblivious}. $\mathtt{BanditDRSM}$ runs RFTL with self-concordant regularizer on vector sequence $\{\nabla\overline{F}_q(\boldsymbol{x}_q)\}_{q=1}^Q$. Here the difference compared with $\mathtt{BanditMLSM}$ is, we cannot find an unbiased estimator for $\nabla\overline{F}_q(\boldsymbol{x}_q)$. We use the ellipsoid estimator directly to estimate $ \nabla\overline{F}^{\delta\boldsymbol{H}_q}_q(\boldsymbol{x}_q)$, the gradient of the $\delta\boldsymbol{H}_q$-smoothed function, here $\boldsymbol{H}_q=(\nabla^2 \Phi(z_q\cdot \boldsymbol{x}_q))^{-1/2}$ as the same as $\mathtt{BanditMLSM}$, $\delta$ is a parameter to be determined. Specifically, in block $q$, we select an uniform random exploration round $t_q\in [(q-1)L+1,qL]\cap \mathbb{Z}$, a random direction $\boldsymbol{v}_q\in\mathbb{S}_{d-1}$, $z_q\sim Z$ where $\Pr(Z\leq z)=\int_0^z \frac{e^{u-1}}{1-e^{-1}}\mathbb{I}[u\in\{0,1\}]du$. In round $t_q$, we play $\boldsymbol{y}_{t_q}=z_q\cdot \boldsymbol{x}_q+\delta z_q\cdot \boldsymbol{H}_q^{-1}\boldsymbol{v}_q$ and feedback the gradient estimate as follow to RFTL,
\begin{align}\label{eq:general estimator}
    \widetilde{\nabla}\overline{F}(\boldsymbol{x}_q):= (1-1/e)d\cdot f_{t_q}(\boldsymbol{y}_{t_q})\boldsymbol{H}^{-1}_q\boldsymbol{v}_q.
\end{align}
We prove that $\widetilde{\nabla}\overline{F}(\boldsymbol{x}_q)$ is an unbiased gradient estimator for the $\delta\boldsymbol{H}_q$-smoothed function $\overline{F}_q^{\delta\boldsymbol{H}_q}(\boldsymbol{x}_q)$. Moreover, the dual local norm of the estimator can be bounded as $O(\frac{d^2}{\delta^2})$. To formalize the above arguments, we have the following lemma.
\begin{lemma}
    \label{lem:estimator}
    Let $\widetilde{\nabla}\overline{F}_q(\boldsymbol{x}_q)$ be defined as \eq{general estimator}. Assume $f_t$ for $t\in [(q-1)L+1,qL]$ is $L_1$-lipschitz, $f_t(\boldsymbol{0})=0$ and $\|\boldsymbol{x}+\delta \boldsymbol{H}_q \boldsymbol{v}\|\leq D$, then following holds,
    \begin{itemize}
        \item[(i)] $\E \left[\widetilde{\nabla}\overline{F}_q(\boldsymbol{x}_q)\mid \mathcal{H}_{q-1}\right]=\nabla \overline{F}_q^{\delta\boldsymbol{H}_q}(\boldsymbol{x}_q) $.
        \item[(ii)] $\|\widetilde{\nabla}\overline{F}_q(\boldsymbol{x}_q)\|_{\boldsymbol{x}_q,*}^2\leq \frac{(1-e)^2d^2 L_1^2 D^2}{\delta^2}$
    \end{itemize}
\end{lemma}
Intuitively, we can control the regret of $\{\boldsymbol{x}_q\}_{q=1}^Q$ w.r.t.~the linear function sequence $\{\langle \cdot ,\nabla \overline{F}_q^{\delta\boldsymbol{H}_q}(\boldsymbol{x}_q) \rangle \}_{q=1}^Q$ by using \thm{AHR}. We further prove that $\overline{f}_q^{\delta\boldsymbol{H}_q}$ is also DR-submodular, and $\overline{F}_q^{\delta\boldsymbol{H}_q}(\boldsymbol{x}_q)$ is the auxiliary function of $\overline{f}_q^{\delta\boldsymbol{H}_q}$. This allows us to control the $(1-1/e)$-regret of $\{\boldsymbol{x}_q\}$ w.r.t.~$\{\overline{f}_q^{\delta \boldsymbol{H}_q}\}$ by using \lem{auxiliary function}. A key observation here is $\|\overline{f}_q^{\delta \boldsymbol{H}_q}-\overline{f}_q\|_{\infty}\leq O(\delta^2)$ assuming the online functions are smooth, which means we can bound the $(1-1/e)$-regret of $\{\boldsymbol{x}_q\}$ w.r.t.~$\{\overline{f}_q\}$ in term of the $(1-1/e)$-regret w.r.t.~$\{\overline{f}_q^{\delta \boldsymbol{H}_q}\}$ with an extra $O(\delta^2)$ additive term. Previous works \cite{zhang2019online,niazadeh2021online} use the FKM estimator, where the sample sphere is fixed(which can be seen as a special case of the ellipsoid estimator when $\boldsymbol{H}_q=I$), to prevent the sample action jump out $\mathcal{K}$, they must run their algorithm on a smaller interior $\mathcal{K}_{\delta}$ which is $\delta$-far from $\partial \mathcal{K}$. So this only guarantees the regret competing with the point in $\mathcal{K}_{\delta}$, this adds a $O(\delta)$ term to the overall regret, which is bigger than $O(\delta^2)$ since the $\delta$ is set to $o(1)$ latter.

With this improved gradient estimator and non-oblivious technique, we prove a $\widetilde{O}(T^{3/4})$ $(1-1/e)$-regret of $\mathtt{BanditDRSM}$.

\begin{theorem}\label{thm:DRSM}
    Set $\eta = D^{-2}d^{-1}T^{-1/2}$, $\delta = d^{1/4}T^{-1/8}$, $L=d^{-1/2}T^{1/4}$, $Q=T/L=d^{1/2}T^{3/4}$ in \algo{DRSM}. If $\Phi$ is a $\nu$-self concordant function of $\mathcal{K}$, then the expected $(1-1/e)$-regret of \algo{DRSM} can be bounded as
    \begin{align*}
        \mathcal{R}_{1-1/e}(T)&\leq O(\nu d^{1/2}T^{3/4}\log(T))
    \end{align*}
\end{theorem}

The idea of using a self-concordant regularizer RFTL on smooth online functions is motivated by \cite{saha2011improved}. Where the authors studied the bandit convex optimization problem, and they find that RFTL with a self-concordant regularizer works well when the convex functions are smooth. We find this idea also works here in the bandit DR-submodular maximization problem, while the previous works all assume the smoothness of online functions, they do not make good use of this assumption.}

\section{A Continuous Approach for Submodular Full-Bandit}\label{sect:reduction}
In this section, we show reductions from two selected submodular full-bandit problems to the bandit multi-linear DR-submodular maximization problem. All proofs in this section are deferred to \append{app} due to space limitations.
\subsection{Reduction Framework}
A natural reduction for our task is to consider the multi-linear extension of the submodular function. That is, we consider the multi-linear extension of each submodular set function, running the $\mathtt{BanditMLSM}$ on the function sequence of the multi-linear extensions. If we could estimate the function value of the multi-linear extension unbiasedly by using only one query to the corresponding discrete submodular function, then we would complete the reduction successfully. This idea is already considered in \cite{zhang2019online}. 
However, it does not work in the full-bandit setting here. The main reason is that the definition of multi-linear extension uses information of the values of the submodular set function on all subsets, including those not satisfying the constraint. 
This makes it impossible to find an unbiased estimator for the multi-linear extension under bandit feedback setting. 
To address this problem, \citet{zhang2019online} consider 
\wei{When the citation is the subject, use \citet{zhang2019online} instead of \cite{zhang2019online} so that the authors are the subject. Please check this globally.}
\zongqi{I've already checked the citations.}
a relaxed responsive bandit model, where they allow the algorithm to query the function value of an infeasible action and gain zero reward. Through this relaxation, they prove a $O(T^{8/9})$ $(1-1/e)$-regret upper bound for bandit submodular maximization with a matroid constraint. We do not make this relaxation and consider the original full-bandit model, that is, the algorithm must play a feasible action each round.

Assume we want to transform a $(\mathcal{S},\mathcal{G})$-bandit to a bandit multi-linear DR-submodular maximization instance, where $\mathcal{S}$ is a finite set and we use $g_t\in\mathcal{G}$ to denote the online reward function. The central component of our reduction framework is a mapping from a product of standard simplexes, denoted as $\mathcal{K}$, to the set of all distributions over $\mathcal{S}$, denoted as $\Delta(\mathcal{S})$. The $d$-dimensional standard simplex is a set $\{(x_1,x_2,\ldots,x_d) \mid x_1+\cdots+x_d\leq 1, x_i\geq 0, \forall i\}$.

\begin{algorithm}[t]
        \caption{$\mathtt{MLSMWrapper}(\eta,L,\Phi,\mbox{EXT})$}
        \label{algo:MLSMW}
    
        \textbf{Input}: learning rate $\eta$, block size $L$, self-concordant barrier $\Phi$, an extension mapping $\mbox{EXT}$
        \begin{algorithmic}[1]
           \FOR{$t=1,2,\ldots,T$}
                \STATE Get $\boldsymbol{y}_t$ from $\mathtt{BanditMLSM4PS}(\eta,L,\Phi)$
                \STATE Sample $S_t$ from distribution $\mbox{EXT}(\boldsymbol{y}_t)$
                \STATE Play $S_t$ and feed $g_t(S_t)$ back to $\mathtt{BanditMLSM4PS}(\eta,L,\Phi)$
           \ENDFOR
        \end{algorithmic}
\end{algorithm}

We denote the extension mapping as $\mbox{EXT}: \mathcal{K}\rightarrow \Delta(\mathcal{S})$. The dimension $d$ of the set $\mathcal{K}$ varies with different $\mathcal{S}$ and $\mathcal{G}$.
The extension mapping naturally defines an extension of any function $g\in G$, that is,
$f(\boldsymbol{x})=\E_{S\in \mbox{EXT}(\boldsymbol{x})}[g(S)]$.
This extension has a good property, if we sample an element $S\in\mathcal{S}$ according to the distribution $\mbox{EXT}(\boldsymbol{x})$, then $g(S)$ is an unbiased estimator of $f(\boldsymbol{x})$. The idea is to run $\mathtt{Bandit MLSM}$ on such extensions $\{f_t\}_{t=1}^T$ of the online functions sequence $\{g_t\}_{t=1}^T$. When we received an action $\boldsymbol{y}_t$ from $\mathtt{Bandit MLSM}$, we sample an action $S_t\in \mathcal{S}$ from $\mbox{EXT}(\boldsymbol{y}_t)$, and feed $g_t(S_t)$ back to $\mathtt{Bandit MLSM}$. However, if we replace the $f_t(\boldsymbol{y}_{t_q})$ with $g_t(S_t)$ in \eq{lq}, the estimator $\widetilde{l}(\boldsymbol{H}_q\boldsymbol{v}_q)$ can be unbounded when $z_q$ is very small which makes the regret uncontrollable. Fortunately, when $\mathcal{K}$ is a product of simplexes, we can slightly modify $\mathtt{BanditMLSM}$ to address this problem. We denote the modified algorithm as $\mathtt{BanditMLSM4PS}$. In this algorithm, we use another estimator to substitute \eq{lq} when $z_q< \frac{1}{2}$. That is, we draw $\boldsymbol{u}_q \in \{\boldsymbol{e}_1,\ldots,\boldsymbol{e}_d\}$ uniformly at random. Then we let $\boldsymbol{y}_{t_q}= z_q\boldsymbol{x}_q$ or $\boldsymbol{y}_{t_q}= z_q\boldsymbol{x}_q+\frac{1}{2}\boldsymbol{u}_q$ with equal probability. The estimator is set to be 
$
    \widetilde{l}_q(\boldsymbol{H}_q\boldsymbol{v}_q):=\left\{
        \begin{aligned}
            &-4(1-1/e)d\langle \boldsymbol{H}_q\boldsymbol{v}_q,\boldsymbol{u}_q\rangle f_{t_q}(\boldsymbol{y}_{t_q}) \quad \mbox{if }\boldsymbol{y}_{t_q}=z_q\boldsymbol{x}_q,\\
            &4(1-1/e)d\langle \boldsymbol{H}_q\boldsymbol{v}_q,\boldsymbol{u}_q\rangle f_{t_q}(\boldsymbol{y}_{t_q}) \quad \mbox{if } \boldsymbol{y}_{t_q}=z_q\boldsymbol{x}_q+\frac{1}{2}\boldsymbol{u}_q.
        \end{aligned}
        \right.
$
To show the estimator is feasible, we need to prove that $z_q\boldsymbol{x}_q+\frac{1}{2}\boldsymbol{u}_q\in \mathcal{K}$ such that the value $f_{t_q}(z_q\boldsymbol{x}_q+\frac{1}{2}\boldsymbol{u}_q)$ can be observed in bandit feedback model. Consider the simplex to which the basis vector $\boldsymbol{u}_q$ belongs, without loss of generality, we assume that $\boldsymbol{u}_q=\boldsymbol{e}_1$ and $\boldsymbol{e}_1,\boldsymbol{e}_2,\ldots,\boldsymbol{e}_{d_1}$ form the basis of the simplex. Then $x_1+\ldots+x_{d_1}\leq 1$, which means $z_q\sum_{i=1}^{d_1}x_i<\frac{1}{2}$, therefore $\frac{1}{2}+z_q\sum_{i=1}^{d_1}x_i\leq 1$, $z_q\boldsymbol{x}_q+\frac{1}{2}\boldsymbol{u}_q\in \mathcal{K}$.

The reduction algorithm is shown in \algo{MLSMW} and the detailed pseudo-code of $\mathtt{BanditMLSM4PS}$ can be found in \algo{MLSM4PS} of \append{app}. For product simplexes, we give an $O(d)$-self-concordant barrier in \append{self concordant}. To obtain the regret guarantee, we need to make sure that the extension induced by the extension mapping satisfies the assumption $\mathtt{BanditMLSM4PS}$ requires. Formally, we prove the following lemma.

\begin{lemma}\label{lem:reduction}
    For a finite set $\mathcal{S}$, and a function family $\mathcal{G}\subseteq \mathcal{S}^{\mathbb{R}_+}$, where $\mathcal{S}^{\mathbb{R}_+}$ is the set of all functions that map element in $\mathcal{S}$ to $\mathbb{R}^+$. If there is an extension mapping $\mbox{EXT}:\mathcal{K}\rightarrow \Delta(\mathcal{S})$  satisfying following conditions:
    \begin{enumerate}
        \item $\mathcal{K}\subseteq \mathbb{R}^d$ is a product of standard simplexes.
        \item For any $g\in \mathcal{G}$, $f(\boldsymbol{x})=\E_{S\in \mbox{EXT}(\boldsymbol{x})}[g(S)]$ is a multi-linear, monotone, DR-submodular function, and $f$ is $L_1$-lipschitz continuous, $f(\boldsymbol{0})=0$.
        \item For any $S\in \mathcal{S}$, there exist $\boldsymbol{x}\in\mathcal{K}$ such that $\mbox{EXT}(\boldsymbol{x}) = \boldsymbol{1}_{s}$. Where $\boldsymbol{1}_{S}$ assign probability $1$ to $S$ and $0$ to other elements of $\mathcal{S}$.
    \end{enumerate}
     then the algorithm $\mathtt{MLSMWrapper}$ attains expected 
     $(1-1/e)$-regret \[\mathcal{R}_{1-1/e}(T)\leq O\left(d^{5/3}T^{2/3}\log (T)\right)\] on $(\mathcal{S},\mathcal{G})$-bandit.
\end{lemma}

\subsection{Bandit Monotone Submodular Maximization with Partition Matroid Constraint}
\label{sect:BMSMPM}
We consider a $(\mathcal{S}_{PM},\mathcal{G}_{MS})$-bandit this section, here $\mathcal{S}_{PM}$ is a partition matroid, and $\mathcal{G}_{MS}$ is the family of monotone submodular set function. We assume the functions in $\mathcal{G}_{MS}$ take value $0$ on the empty set.
\begin{definition}[Partition Matroid]
    Let $G$ be a finite ground set. A set system $\mathcal{S}\subseteq 2^G$ is called a partition matroid if there exist $K>0$ and positive integers $r_1,r_2,\ldots,r_K$ such that $G$ can be partitioned into $K$ subsets $G = \bigcup_{k=1}^K G_k$, and $\mathcal{S} = \{A \mid A\in 2^G \mbox{ and } |A\cap G_k|\leq r_k \mbox{ $\forall\ k $}\}$.
\end{definition}
By \lem{reduction}, all we need is to find an appropriate extension mapping.  Let $\Delta_{d}$ be a $d$-dimensional standard simplex, Let $\mathcal{K} = \prod_{k=1}^K \left(\prod_{i=1}^{r_k} \Delta_{|G_k|}^{k,i}\right)$ be the product of standard simplexes.
Here $\Delta_{|G_k|}^{k,i}$ is a $|G_k|$-dimensional standard simplex and $(k,i)$ is the index of this simplex. Next, we construct an extension mapping $\mbox{EXT}_{PM}:\mathcal{K}\rightarrow\mathcal{S}_{PM}$.

For $\boldsymbol{x}\in\mathcal{K}$, write $\boldsymbol{x}=(x_{k,i,s})_{(k,i,s)\in \Lambda}$, $\Lambda=\{(k,i,s)\mid 1\leq k\leq K, 1\leq i\leq r_k, s\in G_k, k,i\in \mathbb{N}\}$ is the index set. $x_{k,i,s}$ means the coordinate of the simplex $\Delta_{|G_k|}^{k,i}$, and
\wei{missing a reference!}
\zongqi{Revised.}
$\boldsymbol{x}\in \mathbb{R}_+^{\sum_{k=1}^K r_k|G_k|}$ satisfies $\sum_{s\in G_k} x_{k,i,s}\leq 1, \ \forall k,i$. We see the point in the standard simplex $\Delta^{k,i}_{|G_k|}$ as a probability distribution over $G_k\cup \{\circ\}$ where $\circ\notin G_k$ is an extra element which means no element in $G_k$ is chosen. We sample elements according to the coordinate of each simplex independently, then $\boldsymbol{x}$ can be seen as a probability distribution over the set $\Omega := \prod_{k=1}^K \left(G_k\cup \{\circ\}\right)^{r_k}$, we use $\mbox{pre-EXT}_{PM}(\boldsymbol{x})$ to denote this distribution on $\Omega$. 
We now define a mapping $\rho:\Omega \longrightarrow \mathcal{S}$ as follows. 
For $\omega\in\Omega$, assume $\omega$ can be represented as $\omega = (\omega_{k,i})_{(k,i)\in \Gamma}$, where $\omega_{k,i} \in G_k\cup \{\circ\}$ 
\wei{I saw both $\{o\}$ and $\{\circ\}$. Please be consistent.}\zongqi{Revised.}
and $\Gamma=\{(k,i)\mid 1\leq k\leq K, 1\leq i\leq r_k, k,i\in\mathbb{N}\}$ is the index set. Then $\rho(\omega)= \{\omega_{k,i}\mid (k,i)\in\Gamma\}\backslash \{\circ\}$.

It's easy to check $\rho(\omega)\in \mathcal{S}$. Thus, for $\boldsymbol{x}$, we first sample an $\omega\sim \mbox{pre-EXT}_{PM}(\boldsymbol{x})$, then map the sample to $\rho(\omega)\in \mathcal{S}$. This process defines a distribution over $\mathcal{S}$. We let this distribution be $\mbox{EXT}_{PM}(\boldsymbol{x})$.
\begin{lemma}\label{lem:EXTPM}
    For $\mathcal{G}_{MS}$, the extension mapping $\mbox{EXT}_{PM}:\mathcal{K}\rightarrow \Delta(\mathcal{S}_{PM})$ satisfies the conditions in \lem{reduction}. Moreover, $\mathcal{K}$ is in a $\sum_{k=1}^K r_k $ dimensional real vector space. For any $g\in \mathcal{G}_{MS}$, the continuous extension $f(\boldsymbol{x})=\E_{S\in \mbox{EXT}_{PM}(\boldsymbol{x})}[g(S)]$ is $M\sqrt{\sum_{k=1}^K r_k|G_k|}$-lipschitz. 
\end{lemma}

\begin{corollary}\label{cor:BMSMPM}
    There is an algorithm attaining the expected $(1-1/e)$-regret of
    $\mathcal{R}_{1-1/e}(T)\leq O\left(\left(\sum_{k=1}^K r_k|G_k|\right)^{5/3}T^{2/3}\log T\right)$
    on any $(\mathcal{S}_{PM},\mathcal{G}_{MS})$-bandit.
\end{corollary}
\subsection{Bandit Sequential Submodular Maximization}\label{sect:BSSM}
Bandit sequential submodular maximization is first studied in \cite{niazadeh2021online}. It is motivated by online retailing platforms where the platform needs to show its products in sequence. There are many types of customers who have different patience and preference. 
Rarely customers will see all the products in the list. 
They will stop browsing the product after seeing some products according to their patience, and the click probability after a customer sees a set of products is submodular. 
This situation can be formalized into an $(\mathcal{S}_{OL},\mathcal{G}_{SS})$-bandit. Let $G$ be the ground set of all products, and the constraint $\mathcal{S}_{OL}$ is the set of all ordered lists of length $|G|$ consisting of elements in $G$. $\mathcal{G}_{SS}$ consists of function $g:\mathcal{S}_{OL} \rightarrow [0, M]$ in this form,
\[g(S) = \sum_{i=1}^{|G|} \lambda_i g_i(\{S_j \mid j\leq i\}), \]
where $\lambda_i$'s with $\lambda_i\geq 0$ are positive weights, $g_i$'s are monotone submodular set functions, and $S_j$ is the $i$-th element in the ordered list $S$. 
If we interpret $g(S)$ as a click probability, then $M=1$.

For technical reasons, we assume that there is a dummy element $\circ$ in $G$, which has $0$ marginal gain for all $g_i$. That is, $\forall i, \forall S\subseteq G$, we have $g_i(S\cup \{\circ \})=g_i(S)$. We denote $G'=G\backslash \{\circ\}$ This assumption can be satisfied by adding a non-clickable item that is not related to the products to $G'$.

Next, we construct an extension mapping $\mbox{EXT}_{SS}$. Let $\mathcal{K}$ be the cartesian product of standard simplexes, $\mathcal{K} = \prod_{i=1}^{|G|} \Delta_{|G'|}^{i}$. We see $\boldsymbol{x}\in\mathcal{K}$ as $|G|$ probability distributions over $G$, the component of $\boldsymbol{x}\in\mathcal{K}$ in $\Delta_{|G'|}^i$ represents the distribution of the $i$-th element of the ordered list, all these distributions are independent. Any $\boldsymbol{x}\in\mathcal{K}$ can be seen as a distribution over $\mathcal{S}_{OL}$. Let this distribution be $\mbox{EXT}_{SS}(\boldsymbol{x})$.

\begin{lemma}\label{lem:BSSM}
    For $\mathcal{G}_{SS}$, the extension mapping $\mbox{EXT}_{SS}:\mathcal{K}\rightarrow \Delta(\mathcal{S}_{OL})$ satisfies the conditions in \lem{reduction}. Moreover, $\mathcal{K}$ is in a $|G|^2-|G| $ dimensional real vector space. For any $g\in \mathcal{G}_{SS}$, the continuous extension $f(\boldsymbol{x})=\E_{S\in \mbox{EXT}_{SS}(\boldsymbol{x})}[g(S)]$ is $M|G|$-lipschitz. 
\end{lemma}

\begin{corollary}\label{cor:BSSM}
    There is an algorithm for attaining the expected $(1-1/e)$-regret of
    $\mathcal{R}_{1-1/e}(T)\leq O\left((|G|)^{10/3}T^{2/3}\log T\right)$ on any $(\mathcal{S}_{OL},\mathcal{G}_{SS})$-bandit.
\end{corollary}

\section{Conclusion}
In this paper, we propose two bandit algorithms, $\mathtt{BanditMLSM}$ for monotone multilinear DR-submodular functions and $\mathtt{BanditDRSM}$ for general monotone DR-submodular functions. We then show an approach to design the $\widetilde{O}(T^{2/3})$ $(1-1/e)$-regret algorithm for two special combinatorial full-bandits submodular maximization problems, that is, reducing the combinatorial bandits to a multilinear DR-submodular bandit.

There are some remaining open problems that need to be studied. Firstly, we notice that $\widetilde{O}(T^{2/3})$-type regret bounds show up frequently in the submodular bandit literature. However, as far as we know no one has proved or disproved the optimality of this bound, which may be an interesting and challenging problem. Secondly, we still know less about the relationship between the combinatorial constraint and sublinear regret. For submodular set functions, we show that one can achieve sublinear regret with partition matroid constraint in this paper. However, we conjecture that it does not hold for all matroid constraints. How to characterize such a relationship is also a fascinating open question. 

\section*{Acknowledgements}
We thank the anonymous reviewers for their suggestions in the presentation of the article. This work was supported in part by the National Natural Science Foundation of China Grants No. 61832003, 62272441.

\bibliography{main}
\bibliographystyle{icml2023}

\newpage
\appendix
\onecolumn
\section{Technical Lemmas}
This section provides some technical lemmas that will be used in the proofs of this appendix later.
\begin{definition}[Minkowski function and Minkowski set]\label{def:Minkowski}
    Let $\mathcal{K}$ be a compact convex set, the Minkowski function $\pi_{\boldsymbol{x}}:\mathcal{K} \rightarrow \mathbb{R}$ parameterized by a pole $\boldsymbol{x}\in \mbox{int} (\mathcal{K})$ is defined as $\pi_{\boldsymbol{x}}(\boldsymbol{y})\triangleq \mbox{inf}\{t\geq 0 \mid x+t^{-1}(y-x)\in \mathcal{K}\}$. Given $\delta\in \mathbb{R}^+$ and $\boldsymbol{x}_1\in \mbox{int}(\mathcal{K})$, we define the Minkowski set $\mathcal{K}_{\gamma,\boldsymbol{x}_1}\triangleq\{\boldsymbol{x}\in \mathcal{K}\mid \pi_{\boldsymbol{x}_1}(\boldsymbol{x})\leq (1+\gamma)^{-1}\}$.
\end{definition}

The following lemma provides an upper bound of the difference between the function value of a self-concordant barrier at two different points.
\begin{lemma}[\cite{nesterov1994interior}]
\label{lem:potential}
Let $\Phi$ be a $\nu$-self-concordant barrier over a compact convex set $\mathcal{K}$, then for all $x,y\in \mbox{int}(\mathcal{K})$:
\[\Phi(y)-\Phi(x)\leq \nu \log \frac{1}{1-\pi_x(y)}.\]
\end{lemma}

The following lemma is already proved in \cite{abernethy2008competing}, we include the proof for completeness.
\begin{lemma}[\cite{abernethy2008competing}]\label{lem:Minkowski projection}
    Let $\mathcal{K}$ be a compact convex set, $\boldsymbol{x}\in \mbox{int}(\mathcal{K})$ with diameter $D$, $\boldsymbol{x}^*\in \mathcal{K}$ and $\hat{\boldsymbol{x}}^*\triangleq \argmin_{\boldsymbol{z}\in\mathcal{K}_{\gamma,\boldsymbol{x}}}\|\boldsymbol{z}-\boldsymbol{x}^*\|$ be the projection of $\boldsymbol{x}^*$ onto the Minkowski set $\mathcal{K}_{\gamma,\boldsymbol{x}}$, then
    \[\|\boldsymbol{x}^*-\hat{\boldsymbol{x}}^*\|\leq \gamma D\]
\end{lemma}
\begin{proof}
    Consider the point $\boldsymbol{y}$ in the segment $[\boldsymbol{x},\boldsymbol{x}^*]$ satisfying $\frac{\|\boldsymbol{y}-\boldsymbol{x}\|}{\|\boldsymbol{x}^*-\boldsymbol{x}\|}=\frac{1}{1+\gamma}$. Since $\boldsymbol{x}+(1+\gamma)(\boldsymbol{y}-\boldsymbol{x})=\boldsymbol{x}^*\in \mathcal{K}$, we can deduce that $\boldsymbol{y}\in \mathcal{K}_{\gamma,\boldsymbol{x}}$. Thus,
    \[\|\hat{\boldsymbol{x}}^*-\boldsymbol{x}^*\|\leq \|\boldsymbol{y}-\boldsymbol{x}^*\|=\left(1-\frac{1}{1+\gamma}\right)\|\boldsymbol{x}^*-\boldsymbol{x}\|\leq \gamma D.\]
\end{proof}

The following two lemmas show that the average auxiliary functions and the $\boldsymbol{H}$-smoothed functions both inherent good properties of the original online functions. And they will be used later.

\begin{lemma}\label{lem:average function}
If $\forall t\in [(q-1)L+1,qL]$, $f_t$ is twice differentiable, $L_1$-lipschitz and $L_2$-smooth, monotone, DR-submodular, then following holds for the average functions $\overline{f}_q$, $\overline{F}_q$.
\begin{itemize}
    \item[(i)] $\overline{f}_q$ is $L_1$-lipschitz and $L_2$-smooth.
    \item[(ii)] $\overline{f}_q$ is a monotone DR-submodular function.
    \item[(iii)] $\overline{F}_q$ is $\frac{L_2}{e}$-smooth.
    \item[(iv)] $\overline{F}_q$ is a monotone DR-submodular function.
\end{itemize}
\end{lemma}

\begin{proof}
\begin{itemize}
    \item[(i)] 
    \begin{align*}
        \|\overline{f}_q(\boldsymbol{x})-\overline{f}_q(\boldsymbol{y})\|&=\frac{1}{L}\left\|\sum_{t=(q-1)L+1}^{qL} f_t(\boldsymbol{x})-\sum_{t=(q-1)L+1}^{qL} f_t(\boldsymbol{y})\right\|
        \\&\leq \frac{1}{L}\sum_{t=(q-1)L+1}^{qL}\|f_t(\boldsymbol{x})-f_t(\boldsymbol{y})\|
        \\&\leq \frac{1}{L}\sum_{t=(q-1)L+1}^{qL}L_1\|\boldsymbol{x}-\boldsymbol{y}\|=L_1\|\boldsymbol{x}-\boldsymbol{y}\|
    \end{align*}
    \begin{align*}
        \|\nabla\overline{f}_q(\boldsymbol{x})-\nabla\overline{f}_q(\boldsymbol{y})\|&= \frac{1}{L}\left\|\sum_{t=(q-1)L+1}^{qL} \nabla f_t(\boldsymbol{x})-\sum_{t=(q-1)L+1}^{qL} \nabla f_t(\boldsymbol{y})\right\|
        \\&\leq \frac{1}{L}\sum_{t=(q-1)L+1}^{qL}\|\nabla f_t(\boldsymbol{x})-\nabla f_t(\boldsymbol{y})\|
        \\&\leq \frac{1}{L}\sum_{t=(q-1)L+1}^{qL}L_2 \|\boldsymbol{x}-\boldsymbol{y}\|\leq L_2\|\boldsymbol{x}-\boldsymbol{y}\|
    \end{align*}
    \item[(ii)]
        For any $i\in [d]$, 
        \begin{align*}
            \frac{\partial \overline{f}_q}{\partial x_i}(\boldsymbol{x}) &= \frac{1}{L}\sum_{t=(q-1)L+1}^{qL}\frac{\partial f_t}{\partial x_i}(\boldsymbol{x})\geq 0
        \end{align*}
        For any $i\in [d],j\in [d]$, 
        \begin{align*}
            \frac{\partial^2}{\partial x_i\partial x_j}\overline{f}_q(\boldsymbol{x})=\frac{1}{L}\sum_{t=(q-1)L+1}^{qL}\frac{\partial^2 f_t}{\partial x_i\partial x_j}(\boldsymbol{x})\leq 0
        \end{align*}
        Thus $\overline{f}_q$ is monotone DR-submodular.
    \item[(iii)] 
    \begin{align*}
        \|\nabla\overline{F}_q(\boldsymbol{x})-\nabla\overline{F}_q(\boldsymbol{y})\| &= \left\|\nabla\int_{0}^1 \frac{e^{z-1}}{zL}\sum_{t=(q-1)L+1}^{qL} f_t(z\cdot\boldsymbol{x})dz-\nabla\int_{0}^1 \frac{e^{z-1}}{zL}\sum_{t=(q-1)L+1}^{qL} f_t(z\cdot\boldsymbol{y})dz\right\|
        \\&=\left\|\int_{0}^1 \frac{e^{z-1}}{L}\sum_{t=(q-1)L+1}^{qL} \nabla f_t(z\cdot\boldsymbol{x})dz-\int_{0}^1 \frac{e^{z-1}}{L}\sum_{t=(q-1)L+1}^{qL} \nabla f_t(z\cdot\boldsymbol{y})dz\right\|
        \\&\leq \int_{0}^1 \frac{e^{z-1}}{L}\sum_{t=(q-1)L+1}^{qL} \|\nabla f_t(z\cdot\boldsymbol{x})dz- \nabla f_t(z\cdot\boldsymbol{y})\| dz
        \\&\leq\int_{0}^1 \frac{e^{z-1}}{L}\sum_{t=(q-1)L+1}^{qL} L_2 z \|\boldsymbol{x}-\boldsymbol{y}\| dz
        \\&= L_2\int_{0}^1ze^{z-1}dz  \|\boldsymbol{x}-\boldsymbol{y}\|= \frac{L_2}{e}\|\boldsymbol{x}-\boldsymbol{y}\|
    \end{align*}
    \item[(iv)] For any $i\in [d]$,
    \begin{align*}
        \frac{\partial}{\partial \boldsymbol{x}_i} \overline{F}_q(\boldsymbol{x})&=\frac{\partial}{\partial \boldsymbol{x}_i}\int_{0}^1 \frac{e^{z-1}}{zL}\sum_{t=(q-1)L+1}^{qL} f_t(z\cdot\boldsymbol{x})dz
        \\&=\int_{0}^1 \frac{e^{z-1}}{L}\sum_{t=(q-1)L+1}^{qL} \frac{\partial}{\partial \boldsymbol{x}_i}f_t(z\cdot\boldsymbol{x})dz
        \\&\geq 0
    \end{align*}
    For any $i\in [d], j\in [d]$,
    \begin{align*}
        \frac{\partial}{\partial \boldsymbol{x}_i\boldsymbol{x}_j}\overline{F}_q(\boldsymbol{x})&=\int_{0}^1 \frac{ze^{z-1}}{L}\sum_{t=(q-1)L+1}^{qL} \frac{\partial}{\partial \boldsymbol{x}_i\boldsymbol{x}_j}f_t(z\cdot\boldsymbol{x})dz
        \\&\leq 0
    \end{align*}
    Thus $\overline{F}_q$ is monotone and DR-submodular.
\end{itemize}
\end{proof}

\begin{lemma}\label{lem:submodular smooth}
    Following properties hold for $\boldsymbol{H}$-smoothed version of a twice differentiable function $f(\boldsymbol{x})$.
    \begin{itemize}
        \item[(i)] If $f(\boldsymbol{x})$ is a monotone DR-submodular function, then for any invertible matrix $\boldsymbol{H}$, its $\boldsymbol{H}$-smoothed version $f^{\boldsymbol{H}}(\boldsymbol{x})$ is a monotone DR-submodular function. 
        \item[(ii)] If $f(\boldsymbol{x})$ is $L_1$-lipschitz continuous and $L_2$-smooth, then $f^{\boldsymbol{H}}(\boldsymbol{x})$ is $L_1$-lipschitz continuous and $L_2$-smooth. 
    \end{itemize}
\end{lemma}

\begin{proof}
    \begin{itemize}
        \item[(i)] By Leibnez integral rule, for any $i\in [d]$,
    \begin{align*}
        \frac{\partial}{\partial \boldsymbol{x}_i} f^{\boldsymbol{H}}(\boldsymbol{x}) &= \int_{\boldsymbol{v}\in\mathbb{B}_{d}}\frac{1}{\mbox{Vol}(\mathbb{B}_d)}\frac{\partial}{\partial \boldsymbol{x}_i }f(\boldsymbol{x}+\boldsymbol{H} \boldsymbol{v})d\boldsymbol{v}
        \\&\geq 0
    \end{align*}
    The last inequality is because $\frac{\partial}{\partial \boldsymbol{x}_i}f(\boldsymbol{x}+\boldsymbol{H}\boldsymbol{v})\geq 0$ for any $i\in [d]$.
    \begin{align*}
        \frac{\partial}{\partial \boldsymbol{x}_i \boldsymbol{x}_j}f^{\boldsymbol{H}}(\boldsymbol{x}) &= \int_{\boldsymbol{v}\in\mathbb{B}_{d}}\frac{1}{\mbox{Vol}(\mathbb{B}_d)}\frac{\partial}{\partial \boldsymbol{x}_i \boldsymbol{x}_j}f(\boldsymbol{x}+\boldsymbol{H} \boldsymbol{v})d\boldsymbol{v}
        \\&\leq 0
    \end{align*}
    The last inequality is because $\frac{\partial}{\partial \boldsymbol{x}_i \boldsymbol{x}_j}f(\boldsymbol{x}+\boldsymbol{H}\boldsymbol{v})\leq 0$ for any $i,j\in [d]$.
    
        \item[(ii)] 
            \begin{align*}
                f^{\boldsymbol{H}}(\boldsymbol{x})- f^{\boldsymbol{H}}(\boldsymbol{y})&= \int_{\boldsymbol{v}\in\mathbb{B}_d} \frac{1}{\mbox{Vol}(\mathbb{B}_d)}\left(f(\boldsymbol{x}+\boldsymbol{H}\boldsymbol{v})-f(\boldsymbol{y}+\boldsymbol{H}\boldsymbol{v})\right)d\boldsymbol{v}
                \\&\leq \int_{\boldsymbol{v}\in\mathbb{B}_d} \frac{1}{\mbox{Vol}(\mathbb{B}_d)}L_1\|\boldsymbol{x}+\boldsymbol{H}\boldsymbol{v}-\boldsymbol{y}-\boldsymbol{H}\boldsymbol{v}\|d\boldsymbol{v}
                \\&=L_1 \|\boldsymbol{x}-\boldsymbol{y}\|
            \end{align*}
            Thus, $f^{\boldsymbol{H}}(\boldsymbol{x})$ is $L_1$-lipschitz continuous.
            \begin{align*}
                \nabla f^{\boldsymbol{H}}(\boldsymbol{x})- \nabla f^{\boldsymbol{H}}(\boldsymbol{y})&= \nabla \int_{\boldsymbol{v}\in\mathbb{B}_d} \frac{1}{\mbox{Vol}(\mathbb{B}_d)}\left(f(\boldsymbol{x}+\boldsymbol{H}\boldsymbol{v})-f(\boldsymbol{y}+\boldsymbol{H}\boldsymbol{v})\right)d\boldsymbol{v}
                \\&=\int_{\boldsymbol{v}\in\mathbb{B}_d} \frac{1}{\mbox{Vol}(\mathbb{B}_d)}\nabla\left(f(\boldsymbol{x}+\boldsymbol{H}\boldsymbol{v})-f(\boldsymbol{y}+\boldsymbol{H}\boldsymbol{v})\right)d\boldsymbol{v}
                \\&\leq \int_{\boldsymbol{v}\in\mathbb{B}_d} \frac{1}{\mbox{Vol}(\mathbb{B}_d)}L_2 \|\boldsymbol{x}-\boldsymbol{y}\|d\boldsymbol{v}
                \\&=L_2\|\boldsymbol{x}-\boldsymbol{y}\|
            \end{align*}
    \end{itemize}
    
\end{proof}

\section{Missing Proofs in \sect{BMMDSM}}
\label{append:BMMDSM}

We first show a property of multi-linear functions, which is the key observation of our estimator for the gradient of multi-linear functions.
\begin{lemma}\label{lem:multi-linear}
    If $f:\mathcal{K}\rightarrow \mathbb{R}$ is a multi-linear function, where $\mathcal{K}\subseteq \mathbb{R}^d$, then for any basis vector $\boldsymbol{e}_i, i\in [d]$ and any $\boldsymbol{x}\in\mathcal{K}$ and $\lambda>0$ satisfying $\boldsymbol{x}+\lambda \boldsymbol{e}_i \in\mathcal{K}$. We have,
    \begin{align}\label{eq:multi-linear}
        \frac{\partial f(\boldsymbol{x})}{\partial \boldsymbol{x}_i}=\frac{f(\boldsymbol{x}+\lambda\boldsymbol{e}_i)-f(\boldsymbol{x})}{\lambda}
    \end{align}
\end{lemma}
\begin{proof}
    When we fix the components of $\boldsymbol{x}$ except $\boldsymbol{x}_i$, $f$ is a linear function of $\boldsymbol{x}_i$. Thus \eq{multi-linear} directly comes from the linearity.
\end{proof}

\begin{customlem}{3.1}
     Let $\mathcal{H}_{q}$ be the history of the algorithm in the first $q$ blocks, that is, the realization of $t_s,z_s,\boldsymbol{v}_s,\boldsymbol{u}_s,\forall s\leq q$. Then $\E[\widetilde{l}_q(\boldsymbol{H}_q\boldsymbol{v}_q)\mid \mathcal{H}_{q-1},\boldsymbol{v}_q]= l_q(\boldsymbol{H}_q\boldsymbol{v}_q)$.
\end{customlem}
\begin{proof}[Proof of \lem{multi-linear estimator}]
    When $z_q>0$, we have
    \begin{align*}
        \E[\widetilde{l}_q(\boldsymbol{H}_q\boldsymbol{v}_q)\mid \mathcal{H}_{q-1},\boldsymbol{v}_q,t_q,z_q] &=\frac{1}{2}(-2(1-1/e)\frac{d}{z_q}\cdot f_{t_q}(z_q\boldsymbol{x}_q))+\sum_{i=1}^d\frac{1}{2d}(2(1-1/e)\frac{d}{z_q}\cdot f_{t_q}(z_q\boldsymbol{x}_q+z_q\langle\boldsymbol{H}_q\boldsymbol{v}_q,\boldsymbol{e}_i\rangle\boldsymbol{e}_i)
        \\&=(1-1/e)\sum_{i=1}^d \frac{1}{z_q}\left(f_{t_q}(z_q\boldsymbol{x}_q+z_q\langle\boldsymbol{H}_q\boldsymbol{v}_q,\boldsymbol{e}_i\rangle\boldsymbol{e}_i)-f_{t_q}(z_q\boldsymbol{x}_q)\right)
        \\&=(1-1/e)\sum_{i=1}^d \frac{1}{z_q} z_q\langle \boldsymbol{H}_q \boldsymbol{v}_q,\boldsymbol{e}_i\rangle \frac{\partial f_{t_q}}{\partial \boldsymbol{x}_i}({z_q\cdot \boldsymbol{x}_q})
        \\&=(1-1/e)\sum_{i=1}^d \langle\boldsymbol{H}_q \boldsymbol{v}_q,\boldsymbol{e}_i\rangle \langle\boldsymbol{e}_i, \nabla f_{t_q}(z_q\cdot\boldsymbol{x}_q)\rangle
        \\&=(1-1/e)\langle \boldsymbol{H}_q\boldsymbol{v}_q,\nabla f_{t_q}(z_q\cdot \boldsymbol{x}_q)\rangle
    \end{align*}
    Then we take the expectations over $t_q$ and $z_q$, note the value of $\widetilde{l}(\boldsymbol{H}_q\boldsymbol{v}_q)$ when $z_q=0$ does not affect the result of the integral since $z_q=0$ is a zero measured event.
    \begin{align*}
        \E[\widetilde{l}_q(\boldsymbol{H}_q\boldsymbol{v}_q)\mid \mathcal{H}_{q-1},\boldsymbol{v}_q] &= \sum_{t_q=(q-1)L+1}^{qL} \int_{0}^1 \Pr(t_q,z_q) \E[\widetilde{l}_q(\boldsymbol{H}_q\boldsymbol{v}_q)\mid \mathcal{H}_{q-1},\boldsymbol{v}_q,t_q,z_q] dz_q
        \\&=\sum_{t_q=(q-1)L+1}^{qL} \int_{0}^1 \frac{1}{L}\frac{e^{z_q-1}}{1-1/e} (1-1/e) \langle \boldsymbol{H}_q\boldsymbol{v}_q,\nabla f_{t_q}(z_q\cdot \boldsymbol{x}_q)\rangle dz_q
        \\&=\left\langle \boldsymbol{H}_q\boldsymbol{v}_q, \int_0^1  \frac{e^{z_q-1}}{L}\sum_{t_q=(q-1)L+1}^{qL}\nabla f_{t_q}(z_q\cdot \boldsymbol{x}_q)dz_q\right\rangle 
    \end{align*}
    Since
    \begin{align*}
        \nabla \overline{F}_q(\boldsymbol{x}_q)&= \nabla\int_{0}^1 \frac{e^{z-1}}{zL}\sum_{t=(q-1)L+1}^{qL} f_t(z\cdot\boldsymbol{x}_q)dz
        \\&=\int_{0}^1 \frac{e^{z-1}}{L}\sum_{t=(q-1)L+1}^{qL} \nabla f_t(z\cdot\boldsymbol{x}_q)dz
    \end{align*}
    Therefore,
    \begin{align*}
        \E[\widetilde{l}_q(\boldsymbol{H}_q\boldsymbol{v}_q)\mid \mathcal{H}_{q-1},\boldsymbol{v}_q] &= \langle \boldsymbol{H}_q \boldsymbol{v}_q, \nabla \overline{F}_q(\boldsymbol{x}_q)\rangle 
        \\& = l_q(\boldsymbol{H}_q \boldsymbol{v}_q)
    \end{align*}
\end{proof}

\begin{customlem}{3.2}
    The following properties hold for $\widetilde{\nabla} \overline{F}_q(\boldsymbol{x}_q)$
    \begin{itemize}
        \item[(i)] $\mathbb{E}\left[\widetilde{\nabla} \overline{F}_q(\boldsymbol{x}_q)\mid \mathcal{H}_{q-1}\right]=\nabla \overline{F}_q(\boldsymbol{x}_q)$
        \item[(ii)] $\E\left[\|\widetilde{\nabla} \overline{F}_q(\boldsymbol{x}_q)\|_{\boldsymbol{x}_q,*}^2\mid \mathcal{H}_{q-1}\right]\leq 4(1-1/e)^2 L_1^2 D^2 d^4$
    \end{itemize}
\end{customlem}

\begin{proof}[Proof of \lem{linear estimator}]
    ~
\begin{itemize}
    \item[(i)] 
        \begin{align}
            \E\left[\widetilde{\nabla} \overline{F}_q(\boldsymbol{x}_q)\mid \mathcal{H}_{q-1}\right]&= \int_{\boldsymbol{v}_q\in\mathbb{S}_{d-1}} d\cdot \E[\widetilde{l}_q(\boldsymbol{H}_q\boldsymbol{v}_q)\mid \mathcal{H}_{q-1},\boldsymbol{v}_q]\boldsymbol{H}_q^{-1}\boldsymbol{v}_q d\boldsymbol{v}_q \notag
            \\&=\int_{\boldsymbol{v}_q\in\mathbb{S}_{d-1}} \frac{1}{\mbox{Vol}(\mathbb{S}_{d-1})} d\cdot l_q(\boldsymbol{H}_q\boldsymbol{v}_q) \boldsymbol{H}_q^{-1}\boldsymbol{v}_q d\boldsymbol{v}_q \label{eq:linear estimator 1} 
            \\&=\E_{\boldsymbol{v}_q\sim \mathbb{S}_{d-1}}[d\cdot l_q(\boldsymbol{\boldsymbol{H}_q\boldsymbol{v}_q})\boldsymbol{H}_q^{-1}\boldsymbol{v}_q]\notag
            \\&=\nabla l_q^{\boldsymbol{H}_q}(\boldsymbol{0}) \label{eq:linear estimator 2} 
            \\&=\nabla l_q(\boldsymbol{0})\label{eq:linear estimator 3} 
            \\&=\nabla \overline{F}_q(\boldsymbol{x}_q)\notag
        \end{align}
    \eq{linear estimator 1} is due to \lem{multi-linear estimator}, \eq{linear estimator 2} is due to \lem{ellipsoid estimator}, \eq{linear estimator 3} is because that $l_q$ is a linear function.
    
    \item[(ii)]
    \begin{align}
        &\E\left[\|\widetilde{\nabla} \overline{F}_q (\boldsymbol{x}_q)\|^2_{\boldsymbol{x}_q,*}\mid \mathcal{H}_{q-1}\right]\notag
        \\&=\E\left[\frac{1}{2} (1-1/e)^2(2d^2)^2 \frac{1}{z_q^2}f_{t_q}(z_q\cdot\boldsymbol{x}_q)^2\cdot \boldsymbol{v}_q^T \boldsymbol{H}_q\Phi(\boldsymbol{x}_q)^{-1} \boldsymbol{H}_q^{-1}\boldsymbol{v}_q\mid \mathcal{H}_{q-1}\right] \notag
        \\&\quad +\sum_{i=1}^d \E\left[\frac{1}{2d}(1-1/e)^2 4d^4 \frac{1}{z_q^2}f_{t_q}\left(z_q\cdot\boldsymbol{x}_q+z_q\langle \boldsymbol{H}_q\boldsymbol{v}_q,\boldsymbol{e}_i\rangle\boldsymbol{e}_i\right)^2\boldsymbol{v}_q^T\boldsymbol{H}_q\Phi(\boldsymbol{x}_q)\boldsymbol{H}^{-1}_q\boldsymbol{v}_q\mid \mathcal{H}_{q-1}\right] \notag
        \\&\leq 2(1-1/e)^2d^4 \frac{L_1^2 z_q^2\|\boldsymbol{x}_q\|^2}{z_q^2} \E\left[\boldsymbol{v}_q^T \boldsymbol{H}_q^{-1}\Phi(\boldsymbol{x}_q)^{-1} \boldsymbol{H}_q^{-1}\boldsymbol{v}_q\mid \mathcal{H}_{q-1}\right] \notag
        \\&\quad +2(1-1/e)^2d^4 \frac{L_1^2 z_q^2\|\boldsymbol{x}_q+\langle\boldsymbol{H}_q \boldsymbol{v}_q,\boldsymbol{e}_i\rangle \boldsymbol{e}_i\|^2}{z_q^2}  \E\left[\boldsymbol{v}_q^T \boldsymbol{H}_q^{-1}\Phi(\boldsymbol{x}_q)^{-1} \boldsymbol{H}_q^{-1}\boldsymbol{v}_q\mid \mathcal{H}_{q-1}\right] \label{eq:estimator lipschitz bound}
        \\&=4(1-1/e)^2d^4L_1^2 D^2\E\left[\boldsymbol{v}_q^T (\Phi(\boldsymbol{x}_q)^{-1/2})^{-1}\Phi(\boldsymbol{x}_q)^{-1} (\Phi(\boldsymbol{x}_q)^{-1/2})^{-1}\boldsymbol{v}_q\mid \mathcal{H}_{q-1}\right]\notag
        \\& \leq 4(1-1/e)^2d^4 L_1^2 D^2 \|\boldsymbol{v}_q\|^2_2 \notag
        \\&= 4(1-1/e)^2d^4 L_1^2 D^2 \label{eq:multilinear estimator last}
    \end{align}
    Inequality \eq{estimator lipschitz bound} is because $f_{t_q}$ is $L_1$-lipschitz continuous and $\boldsymbol{x}_q+\langle \boldsymbol{H}_q\boldsymbol{v}_q,\boldsymbol{e}_i\rangle\boldsymbol{e}_i$ is in the Dikin ellipsoid $\{\boldsymbol{x} \mid \|\boldsymbol{x}-\boldsymbol{x}_q\|_{\Phi,\boldsymbol{x}_q}\leq 1\}$, which is contained in $\mathcal{K}$. \eq{multilinear estimator last} is because $\boldsymbol{v}_q\in \mathbb{S}_{d-1}$, thus $\|\boldsymbol{v}_q\|=1$.
\end{itemize}
\end{proof}

\begin{customthm}{3.3}
    Set $\eta = d^{-4}T^{-2/3}$, $L=d^{-2}T^{1/3}$, $Q=T/L =d^2 T^{2/3}$ in \algo{MLSM}, if $\Phi$ is a $\nu$-self-concordant barrier of $\mathcal{K}$, then the expected $(1-1/e)$-regret of \algo{MLSM} can be bounded as
    \begin{align*}
        \mathcal{R}_{1-1/e}(T)&\leq (4(1-1/e)^2 L_1^2 D^2+M) d^{4/3} T^{2/3}+(1-1/e)L_1 D  + \nu d^{4/3} T^{2/3}\log (T)
    \end{align*}
\end{customthm}

\begin{proof}[Proof of \thm{MLSM}]
Set $\widetilde{\boldsymbol{g}}_q = \widetilde{\nabla} \overline{F}_q (\boldsymbol{x}_q)$, $\boldsymbol{g}_q =\nabla \overline{F}_q (\boldsymbol{x}_q) $ in \thm{AHR}. We have proved $\E\left[ \widetilde{\nabla} \overline{F}_q (\boldsymbol{x}_q)\mid \mathcal{H}_{q-1}\right] =\nabla \overline{F}_q (\boldsymbol{x}_q) $. Let $\hat{\boldsymbol{x}}^* \triangleq \argmin_{\boldsymbol{x}\in \mathcal{K}_{\gamma,\boldsymbol{x}_1}} \|\boldsymbol{x}^*-\boldsymbol{x}\|$ be the projection of $\boldsymbol{x}^*$ onto the Minkowski set $\mathcal{K}_{\gamma,\boldsymbol{x}_1}$ defined in \defi{Minkowski}, here the pole is $\boldsymbol{x}_1$, and $\gamma$ is a parameter to be determined later, $\boldsymbol{x}^* = \argmin_{\boldsymbol{x}\in\mathcal{K}}\sum_{t=1}^T f_t(\boldsymbol{x})$.  We have
\begin{equation}\label{eq:regret1}
\begin{aligned}
    \sum_{q=1}^Q \E\left[\langle \nabla \overline{F}_q(\boldsymbol{x}_q),\hat{\boldsymbol{x}}^*-\boldsymbol{x}_q \rangle\mid \mathcal{H}_{q-1}\right]&\leq \eta \sum_{q=1}^Q \E\left[\|\widetilde{\nabla} \overline{F}_q (\boldsymbol{x}_q)\|^2_{\boldsymbol{x}_q,*}\mid \mathcal{H}_{q-1}\right]+\frac{\Phi(\boldsymbol{\hat{\boldsymbol{x}}}^*)-\Phi(\boldsymbol{x}_1)}{\eta}
    \\&\leq 4(1-1/e)^2 L_1^2 D^2 \eta d^4 Q +\frac{\Phi(\hat{\boldsymbol{x}}^*)-\Phi(\boldsymbol{x}_1)}{\eta}
    \\&\leq 4(1-1/e)^2L_1^2 D^2 \eta d^4 Q+\frac{\nu \log (\frac{1}{1-(1+\gamma)^{-1}})}{\eta}
\end{aligned}
\end{equation}

The last inequality is because $\pi_{\boldsymbol{x}_1}(\hat{\boldsymbol{x}}^*)\leq (1+\gamma)^{-1}$ and \lem{potential}. Since $\overline{f}_q$ is a monotone DR-submodular function by \lem{average function} and $\overline{F}_q$ is its auxiliary function, we lower bound the left hand side of \eq{regret1} by \lem{auxiliary function},
\begin{align*}
    \sum_{q=1}^Q\E\left[\langle \nabla \overline{F}_q(\boldsymbol{x}_q),\hat{\boldsymbol{x}}^*-\boldsymbol{x}_q\rangle\mid \mathcal{H}_{q-1}\right]&\geq \sum_{q=1}^Q \E\left[\left((1-1/e)\overline{f}_q(\hat{\boldsymbol{x}}^*)-\overline{f}_q(\boldsymbol{x}_q)\right)\mid \mathcal{H}_{q-1}\right] 
    \\&=\sum_{q=1}^Q \sum_{t=(q-1)L+1}^{qL} \frac{1}{L}\E\left[\left((1-1/e)f_t(\hat{\boldsymbol{x}}^*)-f_t(\boldsymbol{x}_q)\right)\mid \mathcal{H}_{q-1}\right]
    \\&=\frac{1}{L}\sum_{t=1}^T \E\left[(1-1/e)f_t(\hat{\boldsymbol{x}}^*)-f_t\left(\boldsymbol{x}_{\lceil \frac{t}{L}\rceil}\right)\mid \mathcal{H}_{\lceil \frac{t}{L}\rceil - 1}\right]
    \\&=\underbrace{\frac{1}{L}\sum_{t=1}^T \E\left[(1-1/e)f_t(\hat{\boldsymbol{x}}^*)-(1-1/e)f_t(\boldsymbol{x}^*)\mid \mathcal{H}_{\lceil \frac{t}{L}\rceil-1}\right]}_{(A)}
    \\&\quad+\frac{1}{L}\sum_{t=1}^T\E\left[(1-1/e)f_t(\boldsymbol{x}^*)-f_t\left(\boldsymbol{y}_t\right)\mid \mathcal{H}_{\lceil \frac{t}{L}\rceil-1}\right] \\&\quad\quad +\underbrace{\frac{1}{L}\sum_{q=1}^Q\E\left[f_{t_q}(\boldsymbol{y}_{t_q})-f_{t_q}\left(\boldsymbol{x}_{\lceil \frac{t_q}{L}\rceil}\right)\mid \mathcal{H}_{q-1}\right]}_{(B)}
\end{align*}
Since $|f_t(\hat{\boldsymbol{x}}^*)-f_t(\boldsymbol{x}^*)|\leq L_1 \|\hat{\boldsymbol{x}}^*-\boldsymbol{x}^*\|\leq L_1\gamma D$ and $|f_{t_q}(\boldsymbol{y}_{t_q})-f_{t_q}\left(\boldsymbol{x}_{\lceil \frac{t_q}{L}\rceil}\right)|\leq M$, we have
\begin{align}
\label{eq:3.3.1}
    |(A)|\leq (1-1/e)\frac{L_1}{L}\gamma D T \quad \quad |(B)|\leq \frac{MQ}{L}
\end{align}

Therefore,
\begin{align*}
    \E\left[\sum_{t=1}^T(1-1/e)f_t(\boldsymbol{x}^*)-f_t(\boldsymbol{y}_t)\right]
    &\leq L\sum_{q=1}^Q\E\left[\langle \nabla \overline{F}_q(\boldsymbol{x}_q),\hat{\boldsymbol{x}}^*-\boldsymbol{x}_q\rangle\mid \mathcal{H}_{q-1}\right]-L\times(A) - L\times(B)
    \\&\leq  4(1-1/e)^2 L_1^2 D^2 \eta d^4 T +\frac{\nu L\log (\frac{1}{1-(1+\gamma)^{-1}})}{\eta}+(1-1/e)L_1\gamma D T  +MQ
\end{align*}
The last inequality is because of \eq{regret1} and \eq{3.3.1}. set $\eta = d^{-8/3}T^{-1/3}$, $L=d^{-4/3}T^{1/3}$, $Q=T/L =d^{4/3} T^{2/3}$, $\gamma=\frac{1}{T}$, we have,
\begin{align*}
    \E\left[\sum_{t=1}^T(1-1/e)f_t(\boldsymbol{x}^*)-f_t(\boldsymbol{y}_t)\right]&\leq  (4(1-1/e)^2 L_1^2 D^2+M) d^{4/3} T^{2/3}+(1-1/e)L_1 D  + \nu d^{4/3} T^{2/3}\log (T+1)\\& = O(\nu d^{4/3} T^{2/3}\log (T) )
\end{align*}
\end{proof}

\section{Bandit DR-submodular Maximization}
\label{append:DRSM}
\begin{algorithm}[t]
	\caption{ $\mathtt{BanditDRSM}(\eta,\delta,L,\Phi)$}
	\label{algo:DRSM}

	\textbf{Input}: Smoothing radius $\delta$, block size $L$, block number $Q=T/L$, learning rate $\eta$,self-concordant barrier $\Phi$
	
	\begin{algorithmic}[1]
	   \STATE initiate $\boldsymbol{x}_1\in \mbox{int} (\mathcal{K})$ such that $\nabla \Phi(\boldsymbol{x}_1) = 0$ 
	   \FOR{$q=1,2,\ldots,Q$}
	        \STATE Draw $t_q\sim \mbox{Unif}\{(q-1)L+1,(q-1)L+2,\ldots,qL\}$
	        \FOR{$t=(q-1)L+1, (q-1)L+2,\ldots, qL$}
	            \IF{$t=t_q$}
                        \STATE $\boldsymbol{H}_q = \left(\nabla^2 \Phi (\boldsymbol{x}_q)\right)^{-1/2}$
	                \STATE sample $z_q$ from $\mathbf{Z}$ where $P(\mathbf{Z}\leq z) =\int_{0}^{z} \frac{e^{u-1}}{1-e^{-1}}\mathbb{I}\left[u\in [0,1]\right]d u$
	                \STATE draw $\boldsymbol{v}_q \sim \mathbb{S}_{d-1}$
	                \STATE play $\boldsymbol{y}_t = z_q\cdot \boldsymbol{x}_q+\delta z_q\cdot\boldsymbol{H}_q\boldsymbol{v}_q$
	                \STATE $\widetilde{\nabla} \overline{F}_q(\boldsymbol{x}_q) \gets (1-1/e)\frac{d}{\delta z_q}f_{t_q}(\boldsymbol{y}_t)\boldsymbol{H}_q^{-1}\boldsymbol{v}_q $
	                \STATE $\boldsymbol{x}_{q+1} \gets \argmin\limits_{\boldsymbol{x}\in\mathcal{K}} \sum_{s=1}^q\langle -\eta \widetilde{\nabla}\overline{F}_s(\boldsymbol{x}_s),\boldsymbol{x}\rangle + \Phi(\boldsymbol{x})$
	            \ELSE
	                \STATE play $\boldsymbol{y}_t = \boldsymbol{x}_q$
	            \ENDIF
	        \ENDFOR
	   \ENDFOR
	\end{algorithmic}
\end{algorithm}

In this section we present our algorithm $\mathtt{BanditDRSM}$ for general bandit monotone DR-submodular maximization, the pseudocode is shown in \algo{DRSM}. $\mathtt{BanditDRSM}$ is very similar to $\mathtt{BanditMLSM}$, it also divides $T$ rounds into $Q$ equal size blocks. We use again $\overline{f}_q(\boldsymbol{x})$ and $\overline{F}_q(\boldsymbol{x})$ to denote the average function of $q$-th block and the auxiliary function of it, defined as \eq{average} and \eq{average nonoblivious}. $\mathtt{BanditDRSM}$ runs RFTL with self-concordant regularizer on vector sequence $\{\nabla\overline{F}_q(\boldsymbol{x}_q)\}_{q=1}^Q$. Here the difference compared with $\mathtt{BanditMLSM}$ is, we cannot find an unbiased estimator for $\nabla\overline{F}_q(\boldsymbol{x}_q)$. We use the ellipsoid estimator directly to estimate $ \nabla\overline{F}^{\delta\boldsymbol{H}_q}_q(\boldsymbol{x}_q)$, the gradient of the $\delta\boldsymbol{H}_q$-smoothed function, here $\boldsymbol{H}_q=(\nabla^2 \Phi( \boldsymbol{x}_q))^{-1/2}$ is the same as $\mathtt{BanditMLSM}$, $\delta$ is a parameter to be determined. Specifically, in block $q$, we select a uniform random exploration round $t_q\in [(q-1)L+1,qL]\cap \mathbb{Z}$, a random direction $\boldsymbol{v}_q\in\mathbb{S}_{d-1}$, $z_q\sim Z$ where $\Pr(Z\leq z)=\int_0^z \frac{e^{u-1}}{1-e^{-1}}\mathbb{I}[u\in\{0,1\}]du$. In round $t_q$, we play $\boldsymbol{y}_{t_q}=z_q\cdot \boldsymbol{x}_q+\delta z_q\cdot \boldsymbol{H}_q\boldsymbol{v}_q$ and feedback the gradient estimate as follow to RFTL and define the estimator $\widetilde{\nabla}\overline{F}(\boldsymbol{x}_q)$.
\begin{align}\label{eq:general estimator}
    \widetilde{\nabla}\overline{F}(\boldsymbol{x}_q):= (1-1/e)\frac{d}{z_q\delta}\cdot f_{t_q}(\boldsymbol{y}_{t_q})\boldsymbol{H}^{-1}_q\boldsymbol{v}_q.
\end{align}
$\boldsymbol{y}_{t_q}=z_q(\boldsymbol{x}_q+\delta \boldsymbol{H}_q\boldsymbol{v}_q)$, if we let $\delta\leq 1$, then $\boldsymbol{x}_q+\delta\boldsymbol{H}_q\boldsymbol{v}_q$ is in the Dikin ellipsoid $\{\boldsymbol{x}\mid \|x-x_q\|_{\Phi,\boldsymbol{x}_q}\leq 1\}$, therefore $\boldsymbol{x}_q+\delta\boldsymbol{H}_q^{-1}\boldsymbol{v}_q\in \mathcal{K}$. Since $\boldsymbol{0}\in \mathcal{K}$, $z_q\in [0,1]$ and $\mathcal{K}$ is convex, $\boldsymbol{y}_{t_q}\in\mathcal{K}$. When $z_q=0$, we define $\widetilde{\nabla}\overline{F}(\boldsymbol{x}_q):=0$, since $\Pr(z_q=0)=0$, the value of $\widetilde{\nabla}\overline{F}(\boldsymbol{x}_q)$ when $z_q=0$ does not matter.

We prove that $\widetilde{\nabla}\overline{F}(\boldsymbol{x}_q)$ is an unbiased gradient estimator for the $\delta\boldsymbol{H}_q$-smoothed function $\overline{F}_q^{\delta\boldsymbol{H}_q}(\boldsymbol{x}_q)$. Moreover, the dual local norm of the estimator can be bounded as $O(\frac{d^2}{\delta^2})$. To formalize the above arguments, we have the following lemma.
\begin{lemma}
    \label{lem:estimator}
    Let $\widetilde{\nabla}\overline{F}_q(\boldsymbol{x}_q)$ be defined as \eq{general estimator}. Assume $f_t$ for $t\in [(q-1)L+1,qL]$ is $L_1$-lipschitz, $f_t(\boldsymbol{0})=0$, then the following hold,
    \begin{itemize}
        \item[(i)] $\E \left[\widetilde{\nabla}\overline{F}_q(\boldsymbol{x}_q)\mid \mathcal{H}_{q-1}\right]=\nabla \overline{F}_q^{\delta\boldsymbol{H}_q}(\boldsymbol{x}_q) $.
        \item[(ii)] $\|\widetilde{\nabla}\overline{F}_q(\boldsymbol{x}_q)\|_{\boldsymbol{x}_q,*}^2\leq \frac{(1-e)^2d^2 L_1^2 D^2}{\delta^2}$
    \end{itemize}
\end{lemma}

\begin{proof}[Proof]
    \begin{itemize}
        \item[(i)] Let $\boldsymbol{H}=\delta \boldsymbol{H}_q$ in \lem{ellipsoid estimator}, let $f_{t_q,z_q}(\boldsymbol{x}) \triangleq f_{t_q}(z_q\cdot \boldsymbol{x})$,we have 
        \[\E \left[\widetilde{\nabla}\overline{F}_q(\boldsymbol{x}_q)\mid \mathcal{H}_{q-1},t_q,z_q\right]=\frac{1-1/e}{z_q}\nabla f^{\delta\boldsymbol{H}_q}_{t_q,z_q}(\boldsymbol{x}_q)\]
        Thus,
        \begin{align*}
            \E \left[\widetilde{\nabla}\overline{F}_q(\boldsymbol{x}_q)\mid \mathcal{H}_{q-1}\right]& = \sum_{t_q=(q-1)L+1}^{qL}\int_{0}^1 \Pr\left(t_q,z_q\mid \mathcal{H}_{q-1}\right)\frac{1-1/e}{z_q}\nabla f^{\delta\boldsymbol{H}_q}_{t_q,z_q}(\boldsymbol{x}_q)dz_q
            \\&=\sum_{t_q=(q-1)L+1}^{qL}\int_{0}^1 \frac{e^{z_q-1}}{(1-1/e)L}\frac{1-1/e}{z_q}\nabla f^{\delta\boldsymbol{H}_q}_{t_q,z_q}(\boldsymbol{x}_q)dz_q
            \\&=\sum_{t_q=(q-1)L+1}^{qL}\int_{0}^1 \frac{e^{z_q-1}}{z_qL}\nabla f^{\delta\boldsymbol{H}_q}_{t_q,z_q}(\boldsymbol{x}_q)dz_q
            \\&=\nabla \overline{F}_q^{\delta\boldsymbol{H}_q}(\boldsymbol{x}_q)
        \end{align*}
    
        \item[(ii)] 
        \begin{align*}
            \|\widetilde{\nabla}\overline{F}_q(\boldsymbol{x}_q)\|_{\boldsymbol{x}_q,*}^2&=(1-1/e)^2\frac{d^2}{\delta^2 z_q^2} f_{t_q}^2(z_q\cdot \boldsymbol{x}_q+\delta z_q\cdot\boldsymbol{H_q}\boldsymbol{v}_q) \boldsymbol{v}_q^T\boldsymbol{H}_q^{-1}\left(\nabla^2\Phi(\boldsymbol{x}_q)\right)^{-1}\boldsymbol{H}_q^{-1}\boldsymbol{v}_q
            \\&\leq (1-1/e)^2\frac{d^2}{\delta^2 z_q^2}z_q^2 L_1^2 \|\boldsymbol{x}_q + \delta\boldsymbol{H}_q \boldsymbol{v}_q\|^2\|\boldsymbol{v}_q\|^2
            \\&\leq \frac{(1-e)^2d^2 L_1^2 D^2}{\delta^2}
        \end{align*}
        The first inequality is because $f_{t_q}$ is $L_1$-lipschitz continuous.
    \end{itemize}
\end{proof}

Intuitively, we can control the regret of $\{\boldsymbol{x}_q\}_{q=1}^Q$ w.r.t.~the linear function sequence $\{\langle \cdot ,\nabla \overline{F}_q^{\delta\boldsymbol{H}_q}(\boldsymbol{x}_q) \rangle \}_{q=1}^Q$ by using \thm{AHR}. In \lem{submodular smooth}, We proved that $\overline{f}_q^{\delta\boldsymbol{H}_q}$ is also DR-submodular. Since $\overline{F}_q^{\delta\boldsymbol{H}_q}(\boldsymbol{x}_q)$ is the auxiliary function of $\overline{f}_q^{\delta\boldsymbol{H}_q}$, this allows us to control the $(1-1/e)$-regret of $\{\boldsymbol{x}_q\}$ w.r.t.~$\{\overline{f}_q^{\delta \boldsymbol{H}_q}\}$ by using \lem{auxiliary function}. A key observation here is $\|\overline{f}_q^{\delta \boldsymbol{H}_q}-\overline{f}_q\|_{\infty}\leq O(\delta^2)$ assuming the online functions are smooth, which means we can bound the $(1-1/e)$-regret of $\{\boldsymbol{x}_q\}$ w.r.t.~$\{\overline{f}_q\}$ in term of the $(1-1/e)$-regret w.r.t.~$\{\overline{f}_q^{\delta \boldsymbol{H}_q}\}$ with an extra $O(\delta^2)$ additive term. Previous works \cite{zhang2019online,niazadeh2021online} use the FKM estimator proposed in \cite{flaxman2005online}, where the sample sphere is fixed(which can be seen as a special case of the ellipsoid estimator when $\boldsymbol{H}_q=I$), to prevent the sample action jump out $\mathcal{K}$, they must run their algorithm on a smaller interior $\mathcal{K}_{\delta}$ which is $\delta$-far from $\partial \mathcal{K}$. So this only guarantees the regret competing with the point in $\mathcal{K}_{\delta}$, this adds an $O(\delta)$ term to the overall regret, which is bigger than $O(\delta^2)$ since the $\delta$ is set to $o(1)$ latter.

With this improved gradient estimator and non-oblivious technique, we prove a $\widetilde{O}(T^{3/4})$ $(1-1/e)$-regret of $\mathtt{BanditDRSM}$.

\begin{customthm}{4.1}[restatement]
        Set $\eta = D^{-2}d^{-1}T^{-1/2}$, $\delta = d^{1/4}T^{-1/8}$, $L=d^{-1/2}T^{1/4}$, $Q=T/L=d^{1/2}T^{3/4}$ in \algo{DRSM}. If $\Phi$ is a $\nu$-self concordant function of $\mathcal{K}$, then the expected $(1-1/e)$-regret of \algo{DRSM} can be bounded as
    \begin{align*}
        \mathcal{R}_{1-1/e}(T)&\leq O(\nu d^{1/2}T^{3/4}\log(T))
    \end{align*}
\end{customthm}
\begin{proof}[Proof of \thm{DRSM}]
     Let $\hat{\boldsymbol{x}}^* = \argmin_{\boldsymbol{x}\in \mathcal{K}_{\gamma,\boldsymbol{x}_1}} \|\boldsymbol{x}^*-\boldsymbol{x}\|$, where $\boldsymbol{x}^* = \argmax_{\boldsymbol{x}\in\mathcal{K}}\sum_{t=1}^T f_t(\boldsymbol{x}^*)$. Let $\boldsymbol{g}_q=\nabla \overline{F}_q^{\delta \boldsymbol{H}_q}$, $\widetilde{\boldsymbol{g}}_q=\widetilde{\nabla} \overline{F}_q(\boldsymbol{x}_q)$, $\boldsymbol{y}=\hat{\boldsymbol{x}}^*$  in \thm{AHR}. Since we proved $\widetilde{\nabla} \overline{F}_q (\boldsymbol{x}_q)$ is an unbiased estimate of $\nabla \overline{F}_q^{\delta \boldsymbol{H}_q}(\boldsymbol{x}_{q})$ in \lem{estimator}, we have 
    \begin{align*}
        \sum_{q=1}^Q \E \left[\langle \nabla \overline{F}_q^{\delta \boldsymbol{H}_{q}} (\boldsymbol{x}_q), \boldsymbol{\hat{\boldsymbol{x}}^*}-\boldsymbol{x}_q\rangle\mid \mathcal{H}_{q-1}\right] &\leq \eta \sum_{q=1}^Q \E \left[\|\widetilde{\nabla}\overline{F}_q(\boldsymbol{x}_q)\|_{\Phi,\boldsymbol{x}_q,*}^2\mid \mathcal{ H}_{q-1}\right] + \frac{\Phi(\hat{\boldsymbol{x}}^*) - \Phi(\boldsymbol{x}_1)}{\eta}
        \\&\leq \eta\sum_{q=1}^Q \frac{(1-e)^2d^2 L_1^2 D^2}{\delta^2}+\frac{\Phi(\hat{\boldsymbol{x}}^*) - \Phi(\boldsymbol{x}_1)}{\eta}
        \\&\leq \frac{(1-e)^2\eta Q d^2 L_1^2 D^2}{\delta^2}+\frac{\nu \log (\frac{1}{1-(1+\gamma)^{-1}})}{\eta}
    \end{align*}
    Since $\nabla \overline{F}_q^{\delta \boldsymbol{H}_q} (\boldsymbol{x})$ is monotone DR-submodular due to \lem{submodular smooth}, and it is the auxiliary function of $\overline{f}_q^{\delta\boldsymbol{H}_q}(\boldsymbol{x})$, by \lem{auxiliary function}, we have  
    \[ \sum_{q=1}^Q \E \left[\langle \nabla \overline{F}_q^{\delta \boldsymbol{H}_q} (\boldsymbol{x}_q), \hat{\boldsymbol{x}}^*-\boldsymbol{x}_q\rangle\mid \mathcal{H}_{q-1}\right]\geq \sum_{q=1}^Q \E\left[(1-1/e)\overline{f}_q^{\delta\boldsymbol{H}_q}(\hat{\boldsymbol{x}}^*)-\overline{f}_q^{\delta\boldsymbol{H}_q}(\boldsymbol{x}_q)\mid \mathcal{H}_{q-1}\right]\]
    The RHS can be further decomposed into several terms,
    \begin{equation}\label{eq:ABC}
        \begin{aligned}
        &\sum_{q=1}^Q \E\left[(1-1/e)\overline{f}_q^{\delta\boldsymbol{H}_q}(\hat{\boldsymbol{x}}^*)-\overline{f}_q^{\delta\boldsymbol{H}_q}(\boldsymbol{x}_q)\mid \mathcal{H}_{q-1}\right]
        \\&=\underbrace{\sum_{q=1}^Q \E\left[(1-1/e)\overline{f}_q^{\delta\boldsymbol{H}_q}(\hat{\boldsymbol{x}}^*)-(1-1/e)\overline{f}_q^{\delta\boldsymbol{H}_q}(\boldsymbol{x}^*)\mid \mathcal{H}_{q-1}\right]}_{(A)} +\underbrace{\sum_{q=1}^Q\E\left[(1-1/e)\overline{f}_q^{\delta\boldsymbol{H}_q}(\boldsymbol{x}^*)-(1-1/e)\overline{f}_q(\boldsymbol{x}^*)\mid \mathcal{H}_{q-1}\right]}_{(B)}
        \\&\quad \quad +\sum_{q=1}^Q \E\left[(1-1/e)\overline{f}_q(\boldsymbol{x}^*)-\overline{f}_q(\boldsymbol{x}_q)\mid \mathcal{H}_{q-1}\right] + \underbrace{\sum_{q=1}^Q \E\left[\overline{f}_q(\boldsymbol{x}_q)-\overline{f}_q^{\delta \boldsymbol{H}_q}(\boldsymbol{x}_q)\mid \mathcal{H}_{q-1}\right]}_{(C)}
        \end{aligned}
    \end{equation}

    \noindent \textbf{Bounding $(A)$}: Since $f_t(\boldsymbol{x})$ is $L_1$-lipschitz continuous for any $t$, $\overline{f}_q$ is also $L_1$-lipschitz continuous by \lem{average function}, thus $\overline{f}_q^{\delta \boldsymbol{H}_q}$ is $L_1$-lipschitz continuous by \lem{submodular smooth}. Since $\|\hat{\boldsymbol{x}}^*-\boldsymbol{x}^*\|\leq \gamma D$ by \lem{Minkowski projection},
    \begin{equation}\label{eq:A}
        \begin{aligned}
        \sum_{q=1}^Q \E\left[(1-1/e)\overline{f}_q^{\delta\boldsymbol{H}_q}(\hat{\boldsymbol{x}}^*)-(1-1/e)\overline{f}_q^{\delta\boldsymbol{H}_q}(\boldsymbol{x}^*)\mid \mathcal{H}_{q-1}\right]
        &\geq -\sum_{q=1}^Q (1-1/e)\E\left[|\overline{f}_q^{\delta\boldsymbol{H}_q}(\hat{\boldsymbol{x}}^*)-\overline{f}_q^{\delta\boldsymbol{H}_q}(\boldsymbol{x}^*)|\mid \mathcal{H}_{q-1}\right]
        \\&\geq -\sum_{q=1}^Q (1-1/e)L_1 \gamma D=-(1-1/e)L_1 \gamma D Q
        \end{aligned}
    \end{equation}
    
    \noindent \textbf{Bounding $(B)$}:  Since $f_t(\boldsymbol{x})$ is $L_2$-smooth for any $t$, by \lem{average function} and \lem{submodular smooth}, $\overline{f}_q^{\delta \boldsymbol{H}_q}$ is $L_2$-smooth. Thus,
    \begin{align*}
        \overline{f}_q^{\delta \boldsymbol{H}_q}(\boldsymbol{x}^*)-\overline{f}_q(\boldsymbol{x}^*)&=\frac{1}{\mbox{Vol}(\mathbb{B}_d)}\int_{\boldsymbol{v}\in\mathbb{B}_d} \overline{f}_q(\boldsymbol{x}^*+\delta \boldsymbol{H}_q\boldsymbol{v})-\overline{f}_q(\boldsymbol{x}^*)d\boldsymbol{v}
        \\&\geq \frac{1}{\mbox{Vol}(\mathbb{B}_d)}\int_{\boldsymbol{v}\in\mathbb{B}_d} \langle \nabla \overline{f}_q(\boldsymbol{x}^*),\delta\boldsymbol{H}_q\boldsymbol{v}\rangle -\frac{L_2}{2}\|\delta \boldsymbol{H}_q\boldsymbol{v}\|^2 d\boldsymbol{v}
        \\&=\frac{1}{\mbox{Vol}(\mathbb{B}_d)}\left\langle \nabla \overline{f}_q(\boldsymbol{x}^*),\delta\boldsymbol{H}_q\int_{v\in\mathbb{B}_d}\boldsymbol{v}d\boldsymbol{v}\right\rangle - \frac{1}{\mbox{Vol}(\mathbb{B}_d)}\int_{v\in\mathbb{B}_d}\frac{L_2}{2}\|\delta \boldsymbol{H}_q\boldsymbol{v}\|^2 d\boldsymbol{v}
        \\&\geq -\frac{1}{\mbox{Vol}(\mathbb{B}_d)}\int_{\boldsymbol{v}\in\mathbb{B}_d}\frac{L_2}{2}\delta^2 D^2 d\boldsymbol{v}
        \\&\geq -\frac{L_2\delta^2 D^2}{2}
    \end{align*}
    Therefore,
    \begin{align}\label{eq:B}
        \sum_{q=1}^Q\E\left[(1-1/e)\overline{f}_q^{\delta\boldsymbol{H}_q}(\hat{\boldsymbol{x}}^*)-(1-1/e)\overline{f}_q(\hat{\boldsymbol{x}}^*)\mid \mathcal{H}_{q-1}\right]\geq -\frac{(1-1/e)L_2\delta^2 D^2 Q}{2}
    \end{align}
    \noindent \textbf{Bounding $(C)$}: Similarly,
    \begin{align*}
        \overline{f}_q(\boldsymbol{x}_q)-\overline{f}_q^{\delta \boldsymbol{H}_q}(\boldsymbol{x}_q) &= \frac{1}{\mbox{Vol}(\mathbb{B}_d)}\int_{\boldsymbol{v}\in\mathbb{B}_d} \overline{f}_q(\boldsymbol{x}_q)-\overline{f}_q(\boldsymbol{x}_q+\delta \boldsymbol{H}_q \boldsymbol{v}) d\boldsymbol{v}
        \\&\geq  \frac{1}{\mbox{Vol}(\mathbb{B}_d)}\int_{\boldsymbol{v}\in\mathbb{B}_d} \left\langle \nabla \overline{f}_q(\boldsymbol{x}_q),\delta \boldsymbol{H}_q\boldsymbol{v}\right\rangle-\frac{L_2}{2}\|\delta \boldsymbol{H}_q\boldsymbol{v}\|^2 d\boldsymbol{v}
        \\&\geq -\frac{L_2\delta^2 D^2}{2}
    \end{align*}
    Therefore,
    \begin{align}\label{eq:C}
        \sum_{q=1}^Q \E\left[\overline{f}_q(\boldsymbol{x}_q)-\overline{f}_q^{\delta \boldsymbol{H}_q}(\boldsymbol{x}_q)\mid \mathcal{H}_{q-1}\right]\geq -\frac{L_2\delta^2 D^2Q}{2}
    \end{align}
    Put \eq{A},\eq{B},\eq{C} in \eq{ABC} and rearrange it,
    \begin{align*} 
        &\sum_{q=1}^Q \E\left[(1-1/e)\overline{f}_q(\boldsymbol{x}^*)-\overline{f}_q(\boldsymbol{x}_q)\mid \mathcal{H}_{q-1}\right]
        \\&\leq  \sum_{q=1}^Q \E \left[\langle \nabla \overline{F}_q^{\delta \boldsymbol{H}_q} (\boldsymbol{x}_q), \hat{\boldsymbol{x}}^*-\boldsymbol{x}_q\rangle\mid \mathcal{H}_{q-1}\right] +(1-1/e)L_1 \gamma D Q+\frac{(1-1/e)L_2\delta^2 D^2 Q}{2} +\frac{L_2\delta^2 D^2 Q}{2}
        \\&\leq  \frac{(1-e)^2\eta Q d^2 L_1^2 D^2}{\delta^2}+\frac{\nu \log (\frac{1}{1-(1+\gamma)^{-1}})}{\eta}+(1-1/e)L_1 \gamma D Q +\frac{(2-1/e)L_2\delta^2 D^2 Q}{2}
    \end{align*}
    Then we bound the expected regret,
    \begin{align*}
        \mathcal{R}_{1-1/e}(T)&= \sum_{t=1}^T \E\left[(1-1/e)f_t(\boldsymbol{x}^*)-f_t(\boldsymbol{y}_t)\mid \mathcal{H}_{\lceil\frac{t}{L}\rceil-1}\right]
        \\&= \sum_{t=1}^T\E\left[(1-1/e)f_t(\boldsymbol{x}^*)-f_t(\boldsymbol{x}_{\lceil \frac{t}{L} \rceil})\mid \mathcal{H}_{\lceil\frac{t}{L}\rceil-1}\right]+\sum_{q=1}^Q \E\left[f_{t_q}(\boldsymbol{x}_{q})-f_{t_q}(\boldsymbol{y}_{t_q})\mid \mathcal{H}_{q-1}\right]
        \\&\leq L\sum_{q=1}^Q \E\left[(1-1/e)\overline{f}_q(\boldsymbol{x}^*)-\overline{f}_q(\boldsymbol{x}_q)\mid \mathcal{H}_{q-1}\right]+MQ
        \\&\leq\frac{(1-e)^2\eta LQ d^2 L_1^2 D^2}{\delta^2}+\frac{\nu L\log (\frac{1}{1-(1+\gamma)^{-1}})}{\eta}+(1-1/e)L_1 \gamma D LQ +\frac{(2-1/e)L_2\delta^2 D^2 LQ}{2} +MQ
        \\&= \frac{(1-e)^2\eta d^2 L_1^2 D^2 T}{\delta^2}+\frac{\nu \log (\frac{1}{1-(1+\gamma)^{-1}})L}{\eta}+(1-1/e)L_1 \gamma D T+\frac{(2-1/e)L_2\delta^2 D^2 T}{2}+MQ
    \end{align*}
    Set $\eta = D^{-2}d^{-1}T^{-1/2}$, $\delta = D^{-1/2}d^{1/4}T^{-1/8}$, $L=D^{-1}d^{-1/2}T^{1/4}$, $Q=T/L=Dd^{1/2}T^{3/4}$, $\gamma=\frac{1}{T}$
    \begin{align*}
        R_{1-1/e}(T)&\leq (1-e)^2  L_1^2 Dd^{1/2} T^{3/4}+\nu Dd^{1/2}T^{3/4}\log (T+1)+(1-1/e)L_1D \\&\quad\quad + \frac{(2-1/e)L_2\delta^2 D d^{1/2}T^{3/4}}{2}+MDd^{1/2}T^{3/4}
        \\&=O(\nu d^{1/2}T^{3/4}\log(T))
    \end{align*}
\end{proof}

The idea of using a self-concordant regularizer RFTL on smooth online functions is motivated by \cite{saha2011improved}. Where the authors studied the bandit convex optimization problem, and they find that RFTL with the self-concordant regularizer works well when the convex functions are smooth. We find this idea also works here in the bandit DR-submodular maximization problem.

\OnlyInFull{
\section{Missing Proofs in \sect{DRSM}}

\begin{customlem}{15}
    Let $\widetilde{\nabla}\overline{F}_q(\boldsymbol{x}_q)$ be defined as \eq{general estimator}. Assume $f_t$ for $t\in [(q-1)L+1,qL]$ is $L_1$-lipschitz, $f_t(\boldsymbol{0})=0$ and $\|\boldsymbol{x}+\delta \boldsymbol{H}_q \boldsymbol{v}\|\leq D$, then following holds,
    \begin{itemize}
        \item[(i)] $\E \left[\widetilde{\nabla}\overline{F}_q(\boldsymbol{x}_q)\mid \mathcal{H}_{q-1}\right]=\nabla \overline{F}_q^{\delta\boldsymbol{H}_q}(\boldsymbol{x}_q) $.
        \item[(ii)] $\|\widetilde{\nabla}\overline{F}_q(\boldsymbol{x}_q)\|_{\boldsymbol{x}_q,*}^2\leq \frac{(1-e)^2d^2 L_1^2 D^2}{\delta^2}$
    \end{itemize}
\end{customlem}
\begin{proof}[Proof of \lem{estimator}]
    \begin{itemize}
        \item[(i)] Let $\boldsymbol{H}=\delta \boldsymbol{H}_q$ in \lem{ellipsoid estimator}, let $f_{t_q,z_q}(\boldsymbol{x}) \triangleq f_{t_q}(z_q\cdot \boldsymbol{x})$,we have 
        \[\E \left[\widetilde{\nabla}\overline{F}_q(\boldsymbol{x}_q)\mid \mathcal{H}_{q-1},t_q,z_q\right]=(1-1/e)\nabla f^{\delta\boldsymbol{H}_q}_{t_q,z_q}(\boldsymbol{x}_q)\]
        Thus,
        \begin{align*}
            \E \left[\widetilde{\nabla}\overline{F}_q(\boldsymbol{x}_q)\mid \mathcal{H}\right]& = \sum_{t_q=(q-1)L+1}^{qL}\int_{z_q=0}^1 \Pr\left(t_q,z_q\mid \mathcal{H}\right)(1-1/e)\nabla f^{\delta\boldsymbol{H}_q}_{t_q,z_q}(\boldsymbol{x}_q)dz_q
            \\&=\sum_{t_q=(q-1)L+1}^{qL}\int_{z_q=0}^1 \frac{e^{z_q-1}}{(1-1/e)L}(1-1/e)\nabla f^{\delta\boldsymbol{H}_q}_{t_q,z_q}(\boldsymbol{x}_q)dz_q
            \\&=\nabla \overline{F}_q^{\delta\boldsymbol{H}_q}(\boldsymbol{x}_q)
        \end{align*}
    
        \item[(ii)] 
        \begin{align*}
            \|\widetilde{\nabla}\overline{F}_q(\boldsymbol{x}_q)\|_{\boldsymbol{x}_q,*}^2&=(1-1/e)^2\frac{d^2}{\delta^2 z_q^2} f_{t_q}^2(z_q\cdot \boldsymbol{x}_q+\delta z_q\cdot\boldsymbol{H_q}\boldsymbol{v}_q) \boldsymbol{v}_q^T\boldsymbol{H}_q^{-1}\left(\nabla^2\Phi(\boldsymbol{x}_q)\right)^{-1}\boldsymbol{H}_q^{-1}\boldsymbol{v}_q
            \\&\leq (1-1/e)^2\frac{d^2}{\delta^2 z_q^2}z_q^2 L_1^2 \|\boldsymbol{x}_q + \delta\boldsymbol{H}_q \boldsymbol{v}_q\|^2\|\boldsymbol{v}_q\|^2
            \\&\leq \frac{(1-e)^2d^2 L_1^2 D^2}{\delta^2}
        \end{align*}
        The first inequality is because $f_{t_q}$ is $L_1$-lipschitz continuous.
    \end{itemize}
\end{proof}

\begin{customthm}{16}
    Set $\eta = D^{-2}d^{-1}T^{-1/2}$, $\delta = D^{-1/2}d^{1/4}T^{-1/8}$, $L=D^{-1}d^{-1/2}T^{1/4}$, $Q=T/L=Dd^{1/2}T^{3/4}$ in \algo{DRSM}. If $\Phi$ is a $\nu$-self concordant function of $\mathcal{K}$, then the expected $(1-1/e)$-regret of \algo{DRSM} can be bounded as
    \begin{align*}
        \E\left[\mathcal{R}_{1-1/e}(T)\right]&\leq O(\nu Dd^{1/2}T^{3/4}\log(T))
    \end{align*}    
\end{customthm}
\begin{proof}[Proof of \thm{DRSM}]
     Let $\hat{\boldsymbol{x}}^* = \argmin_{\boldsymbol{x}\in \mathcal{K}_{\gamma,\boldsymbol{x}_1}} \|\boldsymbol{x}^*-\boldsymbol{x}\|$, where $\boldsymbol{x}^* = \argmax_{\boldsymbol{x}\in\mathcal{K}}\sum_{t=1}^T f_t(\boldsymbol{x}^*)$. Let $\boldsymbol{g}_q=-\nabla \overline{F}_q^{\delta \boldsymbol{H}_q}$, $\widetilde{\boldsymbol{g}}_q=-\widetilde{\nabla} \overline{F}_q(\boldsymbol{x}_q)$, $\boldsymbol{y}=\hat{\boldsymbol{x}}^*$  in \thm{AHR}. Since we proved $\widetilde{\nabla} \overline{F}_q (\boldsymbol{x}_q)$ is an unbiased estimate of $\nabla \overline{F}_q^{\delta \boldsymbol{H}_q}(\boldsymbol{x}_{q})$ in \lem{estimator}, we have 
    \begin{align*}
        \sum_{q=1}^Q \E \left[\langle \nabla \overline{F}_q^{\delta \boldsymbol{H}_q} (\boldsymbol{x}_q), \boldsymbol{\hat{\boldsymbol{x}}^*}-\boldsymbol{x}_q\rangle\mid \mathcal{H}_{q-1}\right] &\leq \eta \sum_{q=1}^Q \E \left[\|\widetilde{\nabla}\overline{F}_q(\boldsymbol{x}_q)\|_{\boldsymbol{x}_q,*}^2\mid \mathcal{ H}_q\right] + \frac{\Phi(\hat{\boldsymbol{x}}^*) - \Phi(\boldsymbol{x}_1)}{\eta}
        \\&\leq \eta\sum_{q=1}^Q \frac{(1-e)^2d^2 L_1^2 D^2}{\delta^2}+\frac{\Phi(\hat{\boldsymbol{x}}^*) - \Phi(\boldsymbol{x}_1)}{\eta}
        \\&\leq \frac{(1-e)^2\eta Q d^2 L_1^2 D^2}{\delta^2}+\frac{\nu \log (\frac{1}{1-(1+\gamma)^{-1}})}{\eta}
    \end{align*}
    Since $\nabla \overline{F}_q^{\delta \boldsymbol{H}_q} (\boldsymbol{x})$ is monotone DR-submodular due to \lem{submodular smooth}, and it is the auxiliary function of $\overline{f}_q^{\delta\boldsymbol{H}_q}(\boldsymbol{x})$, by \lem{auxiliary function}, we have  
    \[ \sum_{q=1}^Q \E \left[\langle \nabla \overline{F}_q^{\delta \boldsymbol{H}_q} (\boldsymbol{x}_q), \hat{\boldsymbol{x}}^*-\boldsymbol{x}_q\rangle\mid \mathcal{H}_{q-1}\right]\geq \sum_{q=1}^Q \E\left[(1-1/e)\overline{f}_q^{\delta\boldsymbol{H}_q}(\hat{\boldsymbol{x}}^*)-\overline{f}_q^{\delta\boldsymbol{H}_q}(\boldsymbol{x}_q)\mid \mathcal{H}_{q-1}\right]\]
    The RHS can be further decomposed into several terms,
    \begin{equation}\label{eq:ABC}
        \begin{aligned}
        &\sum_{q=1}^Q \E\left[(1-1/e)\overline{f}_q^{\delta\boldsymbol{H}_q}(\hat{\boldsymbol{x}}^*)-\overline{f}_q^{\delta\boldsymbol{H}_q}(\boldsymbol{x}_q)\mid \mathcal{H}_{q-1}\right]
        \\&=\underbrace{\sum_{q=1}^Q \E\left[(1-1/e)\overline{f}_q^{\delta\boldsymbol{H}_q}(\hat{\boldsymbol{x}}^*)-(1-1/e)\overline{f}_q^{\delta\boldsymbol{H}_q}(\boldsymbol{x}^*)\mid \mathcal{H}_{q-1}\right]}_{(A)} +\underbrace{\sum_{q=1}^Q\E\left[(1-1/e)\overline{f}_q^{\delta\boldsymbol{H}_q}(\boldsymbol{x}^*)-(1-1/e)\overline{f}_q(\boldsymbol{x}^*)\mid \mathcal{H}_{q-1}\right]}_{(B)}
        \\&\quad \quad +\sum_{q=1}^Q \E\left[(1-1/e)\overline{f}_q(\boldsymbol{x}^*)-\overline{f}_q(\boldsymbol{x}_q)\mid \mathcal{H}_{q-1}\right] + \underbrace{\sum_{q=1}^Q \E\left[\overline{f}_q(\boldsymbol{x}_q)-\overline{f}_q^{\delta \boldsymbol{H}_q}(\boldsymbol{x}_q)\mid \mathcal{H}_{q-1}\right]}_{(C)}
        \end{aligned}
    \end{equation}

    \noindent \textbf{Bounding $(A)$}: Since $f_t(\boldsymbol{x})$ is $L_1$-lipschitz continuous for any $t$, $\overline{f}_q$ is also $L_1$-lipschitz continuous by \lem{average function}, thus $\overline{f}_q^{\delta \boldsymbol{H}_q}$ is $L_1$-lipschitz continuous by \lem{submodular smooth}. Since $\|\hat{\boldsymbol{x}}^*-\boldsymbol{x}^*\|\leq \gamma D$ by \lem{Minkowski projection},
    \begin{equation}\label{eq:A}
        \begin{aligned}
        \sum_{q=1}^Q \E\left[(1-1/e)\overline{f}_q^{\delta\boldsymbol{H}_q}(\hat{\boldsymbol{x}}^*)-(1-1/e)\overline{f}_q^{\delta\boldsymbol{H}_q}(\boldsymbol{x}^*)\mid \mathcal{H}_{q-1}\right]
        &\geq -\sum_{q=1}^Q (1-1/e)\E\left[|\overline{f}_q^{\delta\boldsymbol{H}_q}(\hat{\boldsymbol{x}}^*)-\overline{f}_q^{\delta\boldsymbol{H}_q}(\boldsymbol{x}^*)|\mid \mathcal{H}_{q-1}\right]
        \\&\geq -\sum_{q=1}^Q (1-1/e)L_1 \gamma D=-(1-1/e)L_1 \gamma D Q
        \end{aligned}
    \end{equation}
    
    \noindent \textbf{Bounding $(B)$}:  Since $f_t(\boldsymbol{x})$ is $L_2$-smooth for any $t$, by \lem{average function} and \lem{submodular smooth}, $\overline{f}_q^{\delta \boldsymbol{H}_q}$ is $L_2$-smooth. Thus,
    \begin{align*}
        \overline{f}_q^{\delta \boldsymbol{H}_q}(\boldsymbol{x}^*)-\overline{f}_q(\boldsymbol{x}^*)&=\frac{1}{\mbox{Vol}(\mathbb{B}_d)}\int_{\boldsymbol{v}\in\mathbb{B}_d} \overline{f}_q(\boldsymbol{x}^*+\delta \boldsymbol{H}_q\boldsymbol{v})-\overline{f}_q(\boldsymbol{x}^*)d\boldsymbol{v}
        \\&\geq \frac{1}{\mbox{Vol}(\mathbb{B}_d)}\int_{\boldsymbol{v}\in\mathbb{B}_d} \langle \nabla \overline{f}_q(\boldsymbol{x}^*),\delta\boldsymbol{H}_q\boldsymbol{v}\rangle -\frac{L_2}{2}\|\delta \boldsymbol{H}_q\boldsymbol{v}\|^2 d\boldsymbol{v}
        \\&=\frac{1}{\mbox{Vol}(\mathbb{B}_d)}\left\langle \nabla \overline{f}_q(\boldsymbol{x}^*),\delta\boldsymbol{H}_q\int_{v\in\mathbb{B}_d}\boldsymbol{v}d\boldsymbol{v}\right\rangle - \frac{1}{\mbox{Vol}(\mathbb{B}_d)}\int_{v\in\mathbb{B}_d}\frac{L_2}{2}\|\delta \boldsymbol{H}_q\boldsymbol{v}\|^2 d\boldsymbol{v}
        \\&\geq -\frac{1}{\mbox{Vol}(\mathbb{B}_d)}\int_{\boldsymbol{v}\in\mathbb{B}_d}\frac{L_2}{2}\delta^2 D^2 d\boldsymbol{v}
        \\&\geq -\frac{L_2\delta^2 D^2}{2}
    \end{align*}
    Therefore,
    \begin{align}\label{eq:B}
        \sum_{q=1}^Q\E\left[(1-1/e)\overline{f}_q^{\delta\boldsymbol{H}_q}(\hat{\boldsymbol{x}}^*)-(1-1/e)\overline{f}_q(\hat{\boldsymbol{x}}^*)\mid \mathcal{H}_{q-1}\right]\geq -\frac{(1-1/e)L_2\delta^2 D^2 Q}{2}
    \end{align}
    \noindent \textbf{Bounding $(C)$}: Similarly,
    \begin{align*}
        \overline{f}_q(\boldsymbol{x}_q)-\overline{f}_q^{\delta \boldsymbol{H}_q}(\boldsymbol{x}_q) &= \frac{1}{\mbox{Vol}(\mathbb{B}_d)}\int_{\boldsymbol{v}\in\mathbb{B}_d} \overline{f}_q(\boldsymbol{x}_q)-\overline{f}_q(\boldsymbol{x}_q+\delta \boldsymbol{H}_q \boldsymbol{v}) d\boldsymbol{v}
        \\&\geq  \frac{1}{\mbox{Vol}(\mathbb{B}_d)}\int_{\boldsymbol{v}\in\mathbb{B}_d} \left\langle \nabla \overline{f}_q(\boldsymbol{x}_q),\delta \boldsymbol{H}_q\boldsymbol{v}\right\rangle-\frac{L_2}{2}\|\delta \boldsymbol{H}_q\boldsymbol{v}\|^2 d\boldsymbol{v}
        \\&\geq -\frac{L_2\delta^2 D^2}{2}
    \end{align*}
    Therefore,
    \begin{align}\label{eq:C}
        \sum_{q=1}^Q \E\left[\overline{f}_q(\boldsymbol{x}_q)-\overline{f}_q^{\delta \boldsymbol{H}_q}(\boldsymbol{x}_q)\mid \mathcal{H}_{q-1}\right]\geq -\frac{L_2\delta^2 D^2}{2}
    \end{align}
    Put \eq{A},\eq{B},\eq{C} in \eq{ABC} and rearrange it,
    \begin{align*} 
        &\sum_{q=1}^Q \E\left[(1-1/e)\overline{f}_q(\boldsymbol{x}^*)-\overline{f}_q(\boldsymbol{x}_q)\mid \mathcal{H}_{q-1}\right]
        \\&\leq  \sum_{q=1}^Q \E \left[\langle \nabla \overline{F}_q^{\delta \boldsymbol{H}_q} (\boldsymbol{x}_q), \hat{\boldsymbol{x}}^*-\boldsymbol{x}_q\rangle\mid \mathcal{H}_{q-1}\right] +(1-1/e)L_1 \gamma D Q+\frac{(1-1/e)L_2\delta^2 D^2 Q}{2} +\frac{L_2\delta^2 D^2 Q}{2}
        \\&\leq  \frac{(1-e)^2\eta Q d^2 L_1^2 D^2}{\delta^2}+\frac{\nu \log (\frac{1}{1-(1+\gamma)^{-1}})}{\eta}+(1-1/e)L_1 \gamma D Q +\frac{(2-1/e)L_2\delta^2 D^2 Q}{2}
    \end{align*}
    Then we bound the expected regret,
    \begin{align*}
        R_{1-1/e}(T)&\leq \sum_{t=1}^T \E\left[(1-1/e)f_t(\boldsymbol{x}^*)-f_t(\boldsymbol{y}_t)\mid \mathcal{H}_{\lceil\frac{t}{L}\rceil-1}\right]
        \\&= \sum_{t=1}^T\E\left[(1-1/e)f_t(\boldsymbol{x}^*)-f_t(\boldsymbol{x}_{\lceil \frac{t}{L} \rceil})\mid \mathcal{H}_{\lceil\frac{t}{L}\rceil-1}\right]+\sum_{q=1}^Q \E\left[f_{t_q}(\boldsymbol{x}_{q})-f_{t_q}(\boldsymbol{y}_{t_q})\mid \mathcal{H}_{q-1}\right]
        \\&\leq L\sum_{q=1}^Q \E\left[(1-1/e)\overline{f}_q(\boldsymbol{x}^*)-\overline{f}_q(\boldsymbol{x}_q)\mid \mathcal{H}_{q-1}\right]+MQ
        \\&\leq\frac{(1-e)^2\eta LQ d^2 L_1^2 D^2}{\delta^2}+\frac{\nu L\log (\frac{1}{1-(1+\gamma)^{-1}})}{\eta}+(1-1/e)L_1 \gamma D LQ +\frac{(2-1/e)L_2\delta^2 D^2 LQ}{2} +MQ
        \\&= \frac{(1-e)^2\eta d^2 L_1^2 D^2 T}{\delta^2}+\frac{\nu \log (\frac{1}{1-(1+\gamma)^{-1}})L}{\eta}+(1-1/e)L_1 \gamma D T+\frac{(2-1/e)L_2\delta^2 D^2 T}{2}+MQ
    \end{align*}
    Set $\eta = D^{-2}d^{-1}T^{-1/2}$, $\delta = D^{-1/2}d^{1/4}T^{-1/8}$, $L=D^{-1}d^{-1/2}T^{1/4}$, $Q=T/L=Dd^{1/2}T^{3/4}$, $\gamma=\frac{1}{T}$
    \begin{align*}
        R_{1-1/e}(T)&\leq (1-e)^2  L_1^2 Dd^{1/2} T^{3/4}+\nu Dd^{1/2}T^{3/4}\log (T)+(1-1/e)L_1D \\&\quad\quad + \frac{(2-1/e)L_2\delta^2 D d^{1/2}T^{3/4}}{2}+MDd^{1/2}T^{3/4}
        \\&=O(Dd^{1/2}T^{3/4}\log(T))
    \end{align*}
\end{proof}
}

\section{Self-Concordant Barrier of Product Simplexes}
\label{append:self concordant}
In this section, we give a self-concordant barrier for the product simplex, which is a cartesian product of several simplexes. 

Let $\mathcal{K}$ be the product of $n$ simplexes, and their dimensions are $d_1,d_2,\ldots,d_n$ respectively. We write $\mathcal{K}$ as
\[\mathcal{K}=\prod_{i=1}^n \Delta_{d_i}.\]
For $\boldsymbol{x}\in\mathcal{K}$, we represent it as $\boldsymbol{x}=(x_{1,1},x_{1,2},\ldots,x_{1,d_1},x_{2,1},\ldots,x_{2,d_2},\ldots,x_{n,d_n})$. $\boldsymbol{x}\in \mathcal{K}$ iff 
\[
    \left\{
    \begin{aligned}
    &x_{i,j}\geq 0, \quad \quad &\forall 1\leq i\leq n \mbox{ and } 1\leq j\leq d_i
    \\ &\sum_{j=1}^{d_i} x_{i,j}\leq 1, & \forall 1\leq i\leq n
    \end{aligned}
    \right.
\]
Define the function $\Phi:\mbox{int}(\mathcal{K})\rightarrow \mathbb{R}$,
\begin{align*}
    \Phi(\boldsymbol{x}) = -\sum_{i=1}^n \log(1-\vec{1}^T_{d_i}\cdot \boldsymbol{x}_i) - \sum_{i=1}^n\sum_{j=1}^{d_i} \log (x_{i,j}).
\end{align*}
Here $\vec{1}_{d_i}={\underbrace{(1,1,\ldots,1)}_{d_i}}^T$, $\boldsymbol{x}_i=(x_{i,1},x_{i,2},\ldots,x_{i,d_i})$.
We prove that $\Phi$ is a $n$-self-concordant barrier of $\mathcal{K}$.
\begin{lemma}
\label{lem:self concordant}
    $\Phi(\boldsymbol{x})$ is a $\sum_{i=1}^n (d_i+1)$-self-concordant barrier of $\mathcal{K}$.
\end{lemma}
\begin{proof}
    It's easy to see that $\Phi(\boldsymbol{x})$ is three-times continuously differentiable and approaches infinity alone any sequence of points approaching the boundary of $\mathcal{K}$. We first calculate the gradient and the hessian matrix of $\Phi$.

    \begin{align*}
        \frac{\partial \Phi}{\partial x_{i,j}}(\boldsymbol{x}) &= -\sum_{i=1}^n \frac{\partial\log(1-\vec{1}^T_{d_i}\cdot \boldsymbol{x}_i)}{\partial x_{i,j}} - \sum_{i=1}^n\sum_{j=1}^{d_i} \frac{\partial \log (x_{i,j})}{\partial x_{i,j}}
        \\& = \frac{1}{1-\vec{1}^T\cdot\boldsymbol{x}_{i}} - \frac{1}{x_{i,j}}
    \end{align*}

    \begin{align*}
        \frac{\partial^2 \Phi(\boldsymbol{x})}{\partial x_{i_1,j_1}\partial x_{i_2,j_2}}(\boldsymbol{x}) &= \frac{\partial (1-\vec{1}^T\cdot\boldsymbol{x}_{i_1})^{-1}}{\partial x_{i_1,j_1}}-\frac{\partial x_{i_2,j_2}^{-1}}{\partial x_{i_1,j_1}} 
        \\&= \frac{1}{(1-\vec{1}^T\cdot\boldsymbol{x}_{i_1})^2} \mathbb{I}[i_1=i_2] + \frac{1}{x^2_{i_1,j_1}}\mathbb{I}[i_1=i_2,j_1=j_2]
    \end{align*}
    For any direction $\boldsymbol{h}= (h_{1,1},\ldots,h_{1,d_1},h_{2,1},\ldots,h_{n,d_n})^T$, 
    \begin{align*}
        \boldsymbol{h}^T \nabla^2 \Phi(\boldsymbol{x})\boldsymbol{h}&=\sum_{i_1=1}^n\sum_{j_1=1}^{d_i}\sum_{i_2=1}^n\sum_{j_2=1}^{d_i}h_{i_1,j_1}h_{i_2,j_2}\frac{\partial^2 \Phi(\boldsymbol{x})}{\partial x_{i_1,j_1}\partial x_{i_2,j_2}}(\boldsymbol{x})
        \\&=\sum_{i_1=1}^n\sum_{j_1=1}^{d_i}\sum_{i_2=1}^n\sum_{j_2=1}^{d_i}\left(\frac{h_{i_1,j_1}h_{i_2,j_2}}{(1-\vec{1}^T\cdot\boldsymbol{x}_{i_1})^2} \mathbb{I}[i_1=i_2] + \frac{h_{i_1,j_1}h_{i_2,j_2}}{x^2_{i_1,j_1}}\mathbb{I}[i_1=i_2,j_1=j_2]\right)
        \\&=\sum_{i=1}^n \frac{(\sum_{j=1}^{d_i}h_{i,j})^2}{(1-\vec{1}^T\cdot\boldsymbol{x}_{i_1})^2}+\sum_{i=1}^n\sum_{j=1}^{d_i}\frac{h_{i,j}^2}{x_{i,j}^2}\geq 0
    \end{align*}
    Therefore, $\Phi$ is convex. Next, we check the condition $2$ of \defi{self concordant}. Let $\boldsymbol{h}=(h_{1,1},h_{1,d_1},h_{2,1},\ldots,h_{2,d_2},\ldots,h_{n,d_n})$, $\boldsymbol{h}_i = (h_{i,1},\ldots,h_{i,d_i})$. Then,
    \begin{align*}
        &\nabla^3\Phi(\boldsymbol{x})[\boldsymbol{h},\boldsymbol{h},\boldsymbol{h}]
        \\&=\left.\frac{\partial^3 }{\partial t_1\partial t_2 \partial t_3}\left(-\sum_{i=1}^n \log(1-\vec{1}^T_{d_i}\cdot \boldsymbol{x}_i-t_1\vec{1}^T_{d_i}\cdot \boldsymbol{h}_i-t_2\vec{1}^T_{d_i}\cdot\boldsymbol{h}_i-t_3\vec{1}^T_{d_i}\cdot\boldsymbol{h}_i) - \sum_{i=1}^n\sum_{j=1}^{d_i} \log (x_{i,j}+(t_1+t_2+t_3)h_{i,j})\right)\right|_{t_1=t_2=t_3=0}
        \\&=\left.\sum_{i=1}^n\frac{2(\vec{1}^T_{d_i}\cdot\boldsymbol{h}_i)^3}{(1-\vec{1}^T_{d_i}\cdot \boldsymbol{x}_i-t_1\vec{1}^T_{d_i}\cdot \boldsymbol{h}_i-t_2\vec{1}^T_{d_i}\cdot\boldsymbol{h}_i-t_3\vec{1}^T_{d_i}\cdot\boldsymbol{h}_i)^3}-\sum_{i=1}^n\sum_{j=1}^{d_i}\frac{2h_{i,j}^3}{(x_{i,j}+(t_1+t_2+t_3)h_{i,j})^3}\right|_{t_1=t_2=t_3=0}
        \\&=\sum_{i=1}^n\frac{2(\vec{1}^T_{d_i}\cdot\boldsymbol{h}_i)^3}{(1-\vec{1}^T_{d_i}\cdot \boldsymbol{x}_i)^3}-\sum_{i=1}^n\sum_{j=1}^{d_i}\frac{2h_{i,j}^3}{x_{i,j}^3}
    \end{align*}
    We check the first inequality in the condition $2$ of \defi{self concordant}.
    \begin{align*}
        2(\nabla^2 \Phi(\boldsymbol{x})[\boldsymbol{h},\boldsymbol{h}])^{3/2}&=2\left(\sum_{i=1}^n \frac{(\sum_{j=1}^{d_i}h_{i,j})^2}{(1-\vec{1}^T\cdot\boldsymbol{x}_{i_1})^2}+\sum_{i=1}^n\sum_{j=1}^{d_i}\frac{h_{i,j}^2}{x_{i,j}^2}\right)^{3/2}
        \\&=2\left(\sum_{i=1}^n \frac{(\sum_{j=1}^{d_i}h_{i,j})^2}{(1-\vec{1}^T\cdot\boldsymbol{x}_{i_1})^2}+\sum_{i=1}^n\sum_{j=1}^{d_i}\frac{h_{i,j}^2}{x_{i,j}^2}\right)\left(\sum_{i=1}^n \frac{(\sum_{j=1}^{d_i}h_{i,j})^2}{(1-\vec{1}^T\cdot\boldsymbol{x}_{i_1})^2}+\sum_{i=1}^n\sum_{j=1}^{d_i}\frac{h_{i,j}^2}{x_{i,j}^2}\right)^{1/2}
        \\&=2\left(\sum_{i=1}^n \frac{(\sum_{j=1}^{d_i}h_{i,j})^2}{(1-\vec{1}^T\cdot\boldsymbol{x}_{i_1})^2}\left(\sum_{i=1}^n \frac{(\sum_{j=1}^{d_i}h_{i,j})^2}{(1-\vec{1}^T\cdot\boldsymbol{x}_{i_1})^2}+\sum_{i=1}^n\sum_{j=1}^{d_i}\frac{h_{i,j}^2}{x_{i,j}^2}\right)^{1/2}\right.
        \\&\left.\quad+\sum_{i=1}^n\sum_{j=1}^{d_i}\frac{h_{i,j}^2}{x_{i,j}^2}\left(\sum_{i=1}^n \frac{(\sum_{j=1}^{d_i}h_{i,j})^2}{(1-\vec{1}^T\cdot\boldsymbol{x}_{i_1})^2}+\sum_{i=1}^n\sum_{j=1}^{d_i}\frac{h_{i,j}^2}{x_{i,j}^2}\right)^{1/2}\right)
        \\&\geq 2\left(\sum_{i=1}^n \left|\frac{(\sum_{j=1}^{d_i}h_{i,j})^3}{(1-\vec{1}^T\cdot\boldsymbol{x}_{i_1})^3}\right| +\sum_{i=1}^n\sum_{j=1}^{d_i}\left|\frac{h_{i,j}^3}{x_{i,j}^3}\right|\right)
        \\&\geq \left|\sum_{i=1}^n\frac{2(\vec{1}^T_{d_i}\cdot\boldsymbol{h}_i)^3}{(1-\vec{1}^T_{d_i}\cdot \boldsymbol{x}_i)^3}-\sum_{i=1}^n\sum_{j=1}^{d_i}\frac{2h_{i,j}^3}{x_{i,j}^3} \right|=|\nabla^3 \Phi(\boldsymbol{x})[\boldsymbol{h},\boldsymbol{h},\boldsymbol{h}]|
    \end{align*}
    Then we check the inequality between $\nabla \Phi(\boldsymbol{x})[\boldsymbol{h}]$ and $\nabla^2 \Phi(\boldsymbol{x})[\boldsymbol{h},\boldsymbol{h}]$.
    \begin{align*}
        |\nabla \Phi(\boldsymbol{x})[\boldsymbol{h}]| &= |\boldsymbol{h}^T \nabla \Phi(\boldsymbol{x})|
        \\&\leq\sum_{i=1}^n\left|\frac{\vec{1}^T_{d_i}\cdot\boldsymbol{h}_i}{1-\vec{1}^T\cdot\boldsymbol{x}_{i}}\right| + \sum_{i=1}^n\sum_{j=1}^{d_i}\left|\frac{h_{i,j}}{x_{i,j}}\right|
        \\&\leq \sqrt{\sum_{i=1}^n (d_i+1)} \left(\sum_{i=1}^n \frac{(\sum_{j=1}^{d_i}h_{i,j})^2}{(1-\vec{1}^T\cdot\boldsymbol{x}_{i_1})^2}+\sum_{i=1}^n\sum_{j=1}^{d_i}\frac{h_{i,j}^2}{x_{i,j}^2}\right)^{1/2}
        \\&=\left(\sum_{i=1}^n (d_i+1)\right)^{1/2}\left(\nabla^2 \Phi(\boldsymbol{x})[\boldsymbol{h},\boldsymbol{h}]\right)^{1/2}
    \end{align*}
    Therefore, $\Phi(\boldsymbol{x})$ is a $\left(\sum_{i=1}^n (d_i+1)\right)$-self-concordant barrier of $\mathcal{K}$.
\end{proof}

\section{Missing Proofs in \sect{reduction}}
\label{append:app}
The detailed pseudo-code of $\mathtt{BanditMLSM4PS}$ is shown in \algo{MLSM4PS}, the only difference between $\mathtt{BanditMLSM4PS}$ and $\mathtt{BanditMLSM}$ is the line $9$ to line $17$ in \algo{MLSM4PS}.
\begin{algorithm}[t]
	\caption{$\mathtt{BanditMLSM4PS}(\eta,L,\Phi)$}
	\label{algo:MLSM4PS}

	\textbf{Input}: block size $L$, block number $Q=T/L$, learning rate $\eta$, potential function $\Phi$
	
	\begin{algorithmic}[1]
	   \STATE initiate $\boldsymbol{x}_1\in \mbox{int} (\mathcal{K})$ such that $\nabla \Phi(\boldsymbol{x}_1) = 0$ 
	   \FOR{$q=1,2,\ldots,Q$}
	        \STATE Draw $t_q\sim \mbox{Unif}\{(q-1)L+1,(q-1)L+2,\ldots,qL\}$
	        \FOR{$t=(q-1)L+1, (q-1)L+2,\ldots, qL$}
	            \IF{$t=t_q$}
	                \STATE $\boldsymbol{H}_q = \left(\nabla^2 \Phi (\boldsymbol{x}_q)\right)^{-1/2}$
	                \STATE sample $z_q$ from $\mathbf{Z}$ where $P(\mathbf{Z}<z) =\int_{0}^{z} \frac{e^{u-1}}{1-e^{-1}}\mathbb{I}\left[u\in [0,1]\right]d u$
	                \STATE draw $\boldsymbol{v}_q \sim \mathbb{S}_{d-1}$
	                \IF{$z_q\geq \frac{1}{2}$}
                            \STATE draw $\boldsymbol{u}_q$ from $\{\boldsymbol{0},\boldsymbol{e}_1,\boldsymbol{e}_2,\ldots,\boldsymbol{e}_d\}$ with probability: $\Pr(\boldsymbol{u}_q = \boldsymbol{0})= \frac{1}{2}$, $\Pr(\boldsymbol{u}_q = \boldsymbol{e}_i) = \frac{1}{2d}$
                            \STATE $\boldsymbol{y}_{t_q} \leftarrow z_q\cdot \boldsymbol{x}_q+z_q\langle \boldsymbol{H}_q\boldsymbol{v}_q,\boldsymbol{u}_q\rangle\boldsymbol{u}_q$
                            \STATE $\widetilde{l}_q(\boldsymbol{H}_q\boldsymbol{v}_q)\leftarrow\left\{
                                    \begin{aligned}
                                    &-2(1-1/e)\frac{d}{z_q}\cdot f_{t_q}(\boldsymbol{y}_{t_q}) \quad \mbox{if }\boldsymbol{u}_q=\boldsymbol{0},\\
                                     &2(1-1/e)\frac{d}{z_q}\cdot f_{t_q}(\boldsymbol{y}_{t_q}) \quad \mbox{if } \boldsymbol{u}_q\neq \boldsymbol{0}.
                                    \end{aligned}
                                    \right.$
                        \ELSE
                            \STATE draw $\boldsymbol{u}_q$ from $\{\boldsymbol{e}_1,\boldsymbol{e}_2,\ldots,\boldsymbol{e}_d\}$ uniformly at random
                           \STATE  let $\boldsymbol{y}_{t_q}= z_q \boldsymbol{x}_q + \frac{1}{2}\boldsymbol{u}_q$ or $\boldsymbol{y}_{t_q}= z_q \boldsymbol{x}_q$ with equal probability.
                           \STATE play $\boldsymbol{y}_{t_q}$ and observe the feedback $f_{t_q}(\boldsymbol{y}_{t_q})$
                           \STATE $\widetilde{l}_q(\boldsymbol{H}_q\boldsymbol{v}_q)\leftarrow \left\{
                                    \begin{aligned}
                                     &-4(1-1/e)d\cdot \langle \boldsymbol{H}_q\boldsymbol{v}_q,\boldsymbol{u}_q\rangle f_{t_q}(\boldsymbol{y}_{t_q}) \quad \mbox{if }\boldsymbol{y}_{t_q}=z_q\boldsymbol{x}_q,\\
                                    &4(1-1/e)d\cdot \langle \boldsymbol{H}_q\boldsymbol{v}_q,\boldsymbol{u}_q\rangle f_{t_q}(\boldsymbol{y}_{t_q}) \quad \mbox{if } \boldsymbol{y}_{t_q}=z_q\boldsymbol{x}_q+\frac{1}{2}\boldsymbol{u}_q.
                                    \end{aligned}
                                    \right.
                                   $
                        \ENDIF
                        \STATE $\widetilde{\nabla} \overline{F}_q(\boldsymbol{x}_q) \gets d\cdot \widetilde{l}_q(\boldsymbol{H}_q\boldsymbol{v}_q)\boldsymbol{H}_q^{-1}\boldsymbol{v}_q$  
	                   \STATE $\boldsymbol{x}_{q+1} \gets \argmin\limits_{\boldsymbol{x}\in\mathcal{K}} \sum_{s=1}^q\langle -\eta           \widetilde{\nabla}F_s(\boldsymbol{x}_s),\boldsymbol{x}\rangle + \Phi(\boldsymbol{x})$
	            \ELSE
	                \STATE $\boldsymbol{y}_t \leftarrow \boldsymbol{x}_q$, 
                        \STATE sample $S_t$ from $\mbox{EXT}(\boldsymbol{y}_t)$ and play $S_t$.
	            \ENDIF
	        \ENDFOR
	   \ENDFOR
	\end{algorithmic}
\end{algorithm}

\subsection{Proof of \lem{reduction}} \label{append:reduction app}
In \algo{MLSMW}, the algorithm $\mathtt{MLSMWrapper}$ feeds $g_{t_q}(S_{t_q})$ back to $\mathtt{BanditMLSM4PS}$ to replace the value $f_{t_q}(\boldsymbol{y}_{t_q})$. Therefore $\mathtt{MLSMWrapper}$ are actually using a new estimator for $l_q(\boldsymbol{H}_q\boldsymbol{v}_q)$, we denote the new estimator $\widetilde{l}'(\boldsymbol{H}_q\boldsymbol{v}_q)$.
If $z_q\geq \frac{1}{2}$:
\begin{align*}
    \widetilde{l}'_q(\boldsymbol{H}_q\boldsymbol{v}_q):=\left\{\begin{aligned}
                                    &-2(1-1/e)\frac{d}{z_q}\cdot g_{t_q}(S_{t_q}) \quad \mbox{if }\boldsymbol{u}_q=\boldsymbol{0},\\
                                     &2(1-1/e)\frac{d}{z_q}\cdot g_{t_q}(S_{t_q}) \quad \mbox{if } \boldsymbol{u}_q\neq \boldsymbol{0}.
                                    \end{aligned}
                                    \right.
\end{align*}
else:
\begin{align}\label{eq:lq2}
    \widetilde{l}'_q(\boldsymbol{H}_q\boldsymbol{v}_q):=\left\{
        \begin{aligned}
            &-4(1-1/e)d\cdot \langle \boldsymbol{H}_q\boldsymbol{v}_q,\boldsymbol{u}_q\rangle g_{t_q}(S_{t_q}) \quad \mbox{if }\boldsymbol{y}_{t_q}=z_q\boldsymbol{x}_q,\\
            &4(1-1/e)d\cdot \langle \boldsymbol{H}_q\boldsymbol{v}_q,\boldsymbol{u}_q\rangle g_{t_q}(S_{t_q}) \quad \mbox{if } \boldsymbol{y}_{t_q}=z_q\boldsymbol{x}_q+\frac{1}{2}\boldsymbol{u}_q.
        \end{aligned}
        \right.
\end{align}
We first show $\widetilde{l}(\boldsymbol{H}_q\boldsymbol{v}_q)$ is an unbiased estimator of $l(\boldsymbol{H}_q\boldsymbol{v}_q)$.
\begin{lemma}
    The estimator $\widetilde{l}'_q(\boldsymbol{H}_q\boldsymbol{v}_q)$ is an unbiased estimator for $l_q(\boldsymbol{H}_q\boldsymbol{v}_q)$, that is, \[\E\left[\widetilde{l}'_q(\boldsymbol{H}_q\boldsymbol{v}_q)\mid \mathcal{H}_q,\boldsymbol{v}_q\right]=l_q(\boldsymbol{H}_q\boldsymbol{v}_q)\]
\end{lemma}
\begin{proof}
    Condition on $\mathcal{H}_{q-1},\boldsymbol{v}_q,t_q,z_q,\boldsymbol{u}_q$. If $z_q\geq \frac{1}{2}$ and $\boldsymbol{u}_q = 0$, 
    \begin{align*}
    \E\left[\widetilde{l}'_q(\boldsymbol{H}_q\boldsymbol{v}_q)\mid \mathcal{H}_{q-1},\boldsymbol{v}_q,t_q,\boldsymbol{z}_q,\boldsymbol{u}_q\right]&=-2(1-1/e)\frac{d}{z_q}\E[g_{t_q}(S_{t_q})\mid \mathcal{H}_{q-1},\boldsymbol{v}_q,t_q,\boldsymbol{z}_q,\boldsymbol{u}_q]
    \\&=-2(1-1/e)\frac{d}{z_q}f_{t_q}(z_q\boldsymbol{x}_q).
    \end{align*}
    The last equality is because that $S_{t_q}\sim \mbox{EXT}(z_q\boldsymbol{x}_q)$ and $f_{t_q}(\boldsymbol{x}) = \mathbb{E}_{S\sim \mbox{EXT}(\boldsymbol{x})}[g_{t_q}(S)]$. If $z_q\geq \frac{1}{2}$ and $\boldsymbol{u}_q \neq 0$,
    \begin{align*}
    \E\left[\widetilde{l}'_q(\boldsymbol{H}_q\boldsymbol{v}_q)\mid \mathcal{H}_{q-1},\boldsymbol{v}_q,t_q,\boldsymbol{z}_q,\boldsymbol{u}_q\right]&=2(1-1/e)\frac{d}{z_q}\E[g_{t_q}(S_{t_q})\mid \mathcal{H}_{q-1},\boldsymbol{v}_q,t_q,\boldsymbol{z}_q,\boldsymbol{u}_q]
    \\&=2(1-1/e)\frac{d}{z_q}f_{t_q}(z_q\boldsymbol{x}_q+z_q\langle\boldsymbol{H}_q\boldsymbol{v}_q,\boldsymbol{u}_q\rangle \boldsymbol{u}_q).
    \end{align*}

    Condition on $\mathcal{H}_{q-1},\boldsymbol{v}_q,t_q,z_q$, then 
    \begin{align*}
        \E\left[\widetilde{l}'_q(\boldsymbol{H}_q\boldsymbol{v}_q)\mid \mathcal{H}_{q-1},\boldsymbol{v}_q,t_q,\boldsymbol{z}_q\right]&= \frac{1}{2}(-2(1-1/e) \frac{d}{z_q}f_{t_q}(z_q\boldsymbol{x}_q)) +\sum_{i=1}^d \frac{1}{d} 2 (1-1/e)\frac{d}{z_q}f_{t_q}(z_q\boldsymbol{x}_q+z_q\langle\boldsymbol{H}_q\boldsymbol{v}_q,\boldsymbol{u}_q\rangle \boldsymbol{u}_q)
        \\&=(1-1/e)\langle \boldsymbol{H}_q\boldsymbol{v}_q,\nabla f_{t_q}(z_q\cdot \boldsymbol{x}_q)\rangle
    \end{align*}
    where the last equality is already proved in the proof of \lem{multi-linear estimator}.
    
    If $z_q<\frac{1}{2}$,
    \begin{align*}
        &\E\left[\widetilde{l}'_q(\boldsymbol{H}_q\boldsymbol{v}_q)\mid \mathcal{H}_{q-1},\boldsymbol{v}_q,z_q,t_q,\boldsymbol{u}_q,\boldsymbol{y}_{t_q}\right]\\&=\left\{
        \begin{aligned}
            &-4(1-1/e)d\cdot \langle \boldsymbol{H}_q\boldsymbol{v}_q,\boldsymbol{u}_q\rangle \E[g_{t_q}(S_{t_q})\mid \mathcal{H}_{q-1},\boldsymbol{v}_q,z_q,t_q,\boldsymbol{u}_q,\boldsymbol{y}_{t_q}] \quad \mbox{if }\boldsymbol{y}_{t_q}=z_q\boldsymbol{x}_q,\\
            &4(1-1/e)d\cdot \langle \boldsymbol{H}_q\boldsymbol{v}_q,\boldsymbol{u}_q\rangle \E[g_{t_q}(S_{t_q})\mid \mathcal{H}_{q-1},\boldsymbol{v}_q,z_q,t_q,\boldsymbol{u}_q,\boldsymbol{y}_{t_q}] \quad \mbox{if } \boldsymbol{y}_{t_q}=z_q\boldsymbol{x}_q+\frac{1}{2}\boldsymbol{u}_q.
        \end{aligned}
        \right.
        \\&=\left\{
        \begin{aligned}
            &-4(1-1/e)d\cdot \langle \boldsymbol{H}_q\boldsymbol{v}_q,\boldsymbol{u}_q\rangle f_{t_q}(\boldsymbol{y}_{t_q}) \quad \mbox{if }\boldsymbol{y}_{t_q}=z_q\boldsymbol{x}_q,\\
            &4(1-1/e)d\cdot \langle \boldsymbol{H}_q\boldsymbol{v}_q,\boldsymbol{u}_q\rangle f_{t_q}(\boldsymbol{y}_{t_q}) \quad \mbox{if } \boldsymbol{y}_{t_q}=z_q\boldsymbol{x}_q+\frac{1}{2}\boldsymbol{u}_q.
        \end{aligned}
        \right.
    \end{align*}

    Condition on $\mathcal{H}_{q-1},\boldsymbol{v}_q,t_q,z_q$,
    \begin{align*}
        &\E\left[\widetilde{l}'_q(\boldsymbol{H}_q\boldsymbol{v}_q)\mid \mathcal{H}_{q-1},\boldsymbol{v}_q,t_q,\boldsymbol{z}_q\right]\\&=\sum_{i=1}^d \frac{1}{d} \left(\frac{1}{2} \cdot (-4(1-1/e)d\cdot \langle \boldsymbol{H}_q\boldsymbol{v}_q,\boldsymbol{e}_i\rangle f(z_q\boldsymbol{x}_q))+\frac{1}{2}\cdot (4(1-1/e)d\cdot \langle \boldsymbol{H}_q\boldsymbol{v}_q,\boldsymbol{e}_i\rangle f(z_q\boldsymbol{x}_q+\frac{1}{2}\boldsymbol{e}_i))\right)
        \\&=\sum_{i=1}^d 2(1-1/e)\langle \boldsymbol{H}_q\boldsymbol{v}_q,\boldsymbol{e}_i\rangle \left(f(z_q\boldsymbol{x}_q+\frac{1}{2}\boldsymbol{e}_i)-f(z_q\boldsymbol{x}_q) \right)
        \\&=\sum_{i=1}^d 2(1-1/e)\langle \boldsymbol{H}_q\boldsymbol{v}_q,\boldsymbol{e}_i\rangle \frac{1}{2} \frac{\partial f}{\partial x_i}(z_q\boldsymbol{x}_q)
        \\&=(1-1/e)\sum_{i=1}^d\langle \boldsymbol{H}_q\boldsymbol{v}_q,\boldsymbol{e}_i\rangle \langle \boldsymbol{e}_i, \nabla f(z_q\boldsymbol{x}_q)\rangle
        \\&=(1-1/e)\langle \boldsymbol{H}_q\boldsymbol{v}_q,\nabla f_{t_q}(z_q\cdot \boldsymbol{x}_q)\rangle
    \end{align*}
    Combining with the case $z_q\geq \frac{1}{2}$, we proved this equation whatever $z_q$ is:
    \[\E\left[\widetilde{l}'_q(\boldsymbol{H}_q\boldsymbol{v}_q)\mid \mathcal{H}_{q-1},\boldsymbol{v}_q,t_q,\boldsymbol{z}_q\right]=(1-1/e)\langle \boldsymbol{H}_q\boldsymbol{v}_q,\nabla f_{t_q}(z_q\cdot \boldsymbol{x}_q)\rangle\]
    Then follow the calculation in \lem{multi-linear estimator}, we can prove
    \[\E[\widetilde{l}'_q(\boldsymbol{H}_q\boldsymbol{v_q})\mid \mathcal{H}_{q-1},\boldsymbol{v}_q] = l_q(\boldsymbol{H}_q\boldsymbol{v}_q).\]
\end{proof}

\begin{customlem}{5.1}
    For a finite set $\mathcal{S}$, and a function family $\mathcal{G}\subseteq \mathcal{S}^{\mathbb{R}_+}$, where $\mathcal{S}^{\mathbb{R}_+}$ is the set of all functions that map element in $\mathcal{S}$ to $\mathbb{R}^+$. If there is an extension mapping $\mbox{EXT}:\mathcal{K}\rightarrow \Delta(\mathcal{S})$  satisfying following conditions:
    \begin{enumerate}
        \item $\mathcal{K}\subseteq \mathbb{R}^d$ is a product of standard simplexes.
        \item For any $g\in \mathcal{G}$, $f(\boldsymbol{x})=\E_{S\in \mbox{EXT}(\boldsymbol{x})}[g(S)]$ is a multi-linear, monotone, DR-submodular function, and $f$ is $L_1$-lipschitz continuous, $f(\boldsymbol{0})=0$.
        \item For any $s\in \mathcal{S}$. Here exist $\boldsymbol{x}\in\mathcal{K}$ such that $\mbox{EXT}(\boldsymbol{x}) = \boldsymbol{1}_{s}$. Where $\boldsymbol{1}_{S}$ assign probability $1$ to $S$ and $0$ to other elements of $\mathcal{S}$.
    \end{enumerate}
     then the algorithm $\mathtt{MLSMWrapper}$ attains expected $(1-1/e)$-regret $\mathcal{R}_{1-1/e}(T)\leq O\left(d^{5/3}T^{2/3}\log (T)\right)$ on $(\mathcal{S},\mathcal{G})$-bandit.
\end{customlem}

\begin{proof}[Proof of \lem{reduction}]
    We first note that $g_t(S_t)$ is an unbiased estimator of $f_t(\boldsymbol{y}_t)$ by the definition of $f_t$.
    The analysis is the same as the analysis of \algo{MLSM} except that \algo{MLSMW} is actually using a new estimator $\widetilde{l}'_q(\boldsymbol{H}_q\boldsymbol{v}_q)$for $l(\boldsymbol{H}_q\boldsymbol{v}_q)$. We first bound this new estimator.

    If $z_q\geq \frac{1}{2}$, 
    \begin{align*}
        |\widetilde{l}'_q(\boldsymbol{H}_q\boldsymbol{v}_q)| &= 2(1-1/e)d\frac{g_{t_q}(S_{t_q})}{z_q}
        \\&\leq 4(1-1/e)dM
    \end{align*}
    If $z_q< \frac{1}{2}$, 
    \begin{align*}
        |\widetilde{l}'_q(\boldsymbol{H}_q\boldsymbol{v}_q)| &= 4(1-1/e)d\langle\boldsymbol{H}_q\boldsymbol{v}_q,\boldsymbol{u}_q\rangle g_{t_q}(S_{t_q})
        \\&\leq 4(1-1/e)dM
    \end{align*}
    The inequality is because that $z_q\boldsymbol{x}_q+\langle\boldsymbol{H}_q\boldsymbol{v}_q,\boldsymbol{u}_q\rangle\boldsymbol{u}_q\in \mathcal{K}$, and $\boldsymbol{u}_q$ is a basis vector. Therefore $\langle\boldsymbol{H}_q\boldsymbol{v}_q,\boldsymbol{u}_q\rangle\leq D_{\infty}$, here the $D_{\infty}$ is the $\infty$-norm diameter of $\mathcal{K}$. Since $\mathcal{K}$ is a cartesian product of standard simplexes, $D_{\infty}=1$.

    Let $\widetilde{\nabla} \overline{F}'_q(\boldsymbol{x}_q)$ be the estimator replacing $\widetilde{l}'_q(\boldsymbol{H}_q\boldsymbol{v}_q)$ with $\widetilde{l}_q(\boldsymbol{H}_q\boldsymbol{v}_q)$, that is
    \[\widetilde{\nabla} \overline{F}'_q(\boldsymbol{x}_q) = d\cdot \widetilde{l}'_q(\boldsymbol{H}_q\boldsymbol{v}_q)\boldsymbol{H}_q^{-1}\boldsymbol{v}_q\]
    We bound the dual local norm of $\widetilde{\nabla}\overline{F}'_q(\boldsymbol{x}_q)$
    \begin{align*}
        \E\left[\|\widetilde{\nabla}\overline{F}'_q(\boldsymbol{x}_q)\|_{\Phi,\boldsymbol{x}_q,*}\mid \mathcal{H}_{q-1}\right]&=\E\left[d^2\cdot (\widetilde{l}(\boldsymbol{H}_q\boldsymbol{v}_q))^2\boldsymbol{v}_s^T\boldsymbol{H}_q^{-1} \Phi(\boldsymbol{x}_{q})\boldsymbol{H}_q^{-1}\boldsymbol{v}_q\mid \mathcal{H}_{q-1}\right]
        \\&\leq 16(1-1/e)^2d^4M^2\E\left[\boldsymbol{v}_q^T\Phi(\boldsymbol{x}_q)^{-1/2}\Phi(\boldsymbol{x}_{q})\Phi(\boldsymbol{x}_q)^{-1/2}\boldsymbol{v}_q\mid \mathcal{H}_{q-1}\right]
        \\&\leq 16(1-1/e)^2d^4M^2\|\boldsymbol{v}_q\|^2
        \\&\leq 16(1-1/e)^2d^4M^2
    \end{align*}

    Since we have proved $\E[\widetilde{l}_q'(\boldsymbol{H}_q\boldsymbol{v}_q)\mid \mathcal{H}_{q-1},\boldsymbol{v}_q]=l_q(\boldsymbol{H}_q\boldsymbol{v}_q)$, then follow the proof of \lem{linear estimator} (i), we can prove 
    \[\E\left[\widetilde{\nabla} \overline{F}'_q (\boldsymbol{x}_q)\mid \mathcal{H}_{q-1}\right]= \nabla \overline{F}_q(\boldsymbol{x}_q)\]

    Then follow the proof of \thm{MLSM}, we have for any $\boldsymbol{x}^*\in \mathcal{K}$,
    \begin{align*}
        \E\left[\sum_{t=1}^T (1-1/e)f_t(\boldsymbol{x}^*)-f_t(\boldsymbol{y}_t)\right]&\leq \eta L\sum_{q=1}^Q\E \left[\|\widetilde{\nabla} \overline{F}_q'(\boldsymbol{x}_q)\|_{\Phi,\boldsymbol{x}_q,*}^2\mid \mathcal{H}_{q-1}\right]+(1-1/e)L_1 D +MQ +\frac{\nu L\log(\frac{1}{\delta})}{\eta}
        \\&\leq 16(1-1/e)^2M^2 d^4\eta T +(1-1/e)L_1 D+ M Q + \frac{\nu L\log(T)}{\eta}
    \end{align*}
    
    Let $S^* = \argmax_{S\in\mathcal{S}} \sum_{t=1}^T g_t(S)$, $\boldsymbol{x}^*$ be the point satisfies $\mbox{EXT}(\boldsymbol{x}^*)=\boldsymbol{1}_{S^*}$, then $f_t(\boldsymbol{x}^*) = g_t(S^*)$.
    
    Let $\mathcal{H}'_{t}$ be the history of the first $t$-rounds, including the realization of $\boldsymbol{v}_q, t_q, z_q, \boldsymbol{u}_q, \ \forall q\leq \lceil \frac{t}{L}\rceil$ and the realization of $S_{k},\ \forall k\leq t $. Since $f_t(\boldsymbol{y}_t) = \E\left[g_t(S_t)\mid \mathcal{H}_{t-1}\right]$, we have,
    \begin{align*}
        \E\left[\sum_{t=1}^T (1-1/e)g_t(S^*)-g_t(S_t)\right]& =\E\left[\sum_{t=1}^T(1-1/e)f_t(\boldsymbol{x}^*)-\E\left[g_t(S_t)\mid \mathcal{H}_{t-1}\right]\right]
        \\&=\E\left[\sum_{t=1}^T(1-1/e)f_t(\boldsymbol{x}^*)-f_t(\boldsymbol{y}_t)\right]
        \\&\leq 16(1-1/e)^2M^2 d^4\eta T+(1-1/e)L_1 + M Q + \frac{\nu L\log(T)}{\eta}
    \end{align*}
    If we use the self-concordant barrier described in \append{self concordant} as the input $\Phi$ here, by \lem{self concordant}, $\nu = O(d)$. Then we set $\eta=d^{-7/3}T^{-1/3}$, $L=d^{-5/3}T^{1/3}$, $Q=T/L=d^{5/3}T^{2/3}$. Then
    \[\mathcal{R}_{1-1/e}(T)=\E\left[\sum_{t=1}^T (1-1/e)g_t(S^*)-g_t(S_t)\right]=O\left(d^{5/3}T^{2/3}\log (T)\right).\]
\end{proof}

\subsection{Proof of \lem{EXTPM} and \cor{BMSMPM}}

Before proving \lem{EXTPM}, we first prove a useful lemma.
\begin{lemma}\label{lem:monotone submodular}
    Let $g:2^{G}\longrightarrow \mathbb{R}_+$ be a monotone submodular set function, $S$ is a subset of $G$, $s_1,s_2\in G$ and there is no any other restriction on $s_1$ and $s_2$, they may be the same element or not, and they may be in $S$ or not. Then $g(S\cup \{s_1,s_2\})-g(S\cup\{s_1\})\leq g(S\cup \{s_2\})-g(S)$.
\end{lemma}
\begin{proof}
    If $s_1\in S$ and $s_2\notin S$, then $g(S\cup \{s_1,s_2\})-g(S\cup\{s_1\})=g(S\cup \{s_2\})-g(S)$. If $s_1\notin S$ and $s_2\in S$, then $g(S\cup \{s_1,s_2\})-g(S\cup\{s_1\})=g(S\cup \{s_1\})-g(S\cup\{s_1\})=0$ and $g(S\cup \{s_2\})-g(S)=g(S)-g(S)=0 $. If $s_1\in S, s_2 \in S$, then $g(S\cup \{s_1,s_2\})-g(S\cup\{s_1\}) =g(S\cup \{s_2\})-g(S)=g(S)-g(S)=0$, the inequality holds.
    
    If $s_1\notin S, s_2\notin S$ and $s_1\neq s_2$, then the result holds due to the submodularity of $g$. If $s_1 = s_2\notin S$, then $g(S\cup \{s_1,s_2\})-g(S\cup\{s_1\})=0$ and $g(S\cup \{s_2\})-g(S)\geq 0$ by the monotonicity of $g$.
\end{proof}

\begin{customlem}{5.3}
    For $\mathcal{G}_{MS}$, the extension mapping $\mbox{EXT}_{PM}:\mathcal{K}\rightarrow \Delta(\mathcal{S}_{PM})$ satisfies the conditions in \lem{reduction}. Moreover, $\mathcal{K}$ is in a $\sum_{k=1}^K r_k|G_k| $ dimensional real vector space. For any $g\in \mathcal{G}_{MS}$, the continuous extension $f(\boldsymbol{x})=\E_{s\in \mbox{EXT}(\boldsymbol{x})}[g(s)]$ is $M\sqrt{\sum_{k=1}^K r_k|G_k|}$-lipschitz. 
\end{customlem}

\begin{proof}[Proof of \lem{EXTPM}]
     We first prove that for any $S\in\mathcal{S}_{PM}$, there is a $\boldsymbol{x}\in\mathcal{K}$ such that $\mbox{EXT}_{PM}(\boldsymbol{x})=\boldsymbol{1}_{S}$. For any $S\in\mathcal{S}_{PM}$, it can be partitioned into $S=\cup_{k=1}^K S_k$ such that $S_k\subseteq G_k$ and $|S_k|\leq r_k$. For any $k$, we select $|S_k|$ standard simplexes $\Delta_{G_k}^{k,i}, \ i\in [|S_k|]$, and we assign probability $1$ to the elements in $S_k$ respectively in these standard simplexes. Thus the condition $1$ and $3$ of \lem{reduction} is satisfied, the dimension of $\mathcal{K}$ is obvious. It's enough to show that $f(\boldsymbol{x})$ satisfies the condition $2$.

    $\mbox{EXT}_{PM}(\boldsymbol{0})$ assigns probability $1$ to the empty set, thus $f(\boldsymbol{0})=0$.
    
    Now we check the multi-linearity of $f$. Consider a sample $\omega \in \Omega$, the probability of $\omega$ is $\Pr(\omega) = \prod_{k=1}^K\prod_{i=1}^{r_k}\Pr(\omega_{k,i})$, and $\Pr(\omega_{k,i})\in \left\{x_{k,i,s}\mid s\in G_{k}\right\}\cup \{1-\sum_{s\in G_k}x_{k,i,s}\}$, thus $\Pr(\omega)$ is multi-linear with respect to the variables $x_{k,i,s}$. Then we write $f(\boldsymbol{x})$ as follows,
    \begin{align*}
        f(\boldsymbol{x})&= \E_{S\sim \mbox{EXT}_{PM}(\boldsymbol{x})} \left[g(S)\right]
        \\&=\sum_{S\in \mathcal{S}_{PM}} \Pr(S) g(S)
        \\&=\sum_{S\in \mathcal{S}_{PM}} \sum_{\omega \in \rho^{-1}(S)}\Pr(\omega) g(S)
        \\&=\sum_{\omega \in \Omega} \Pr(\omega) g(\rho(\omega))
    \end{align*}
    Since $g(\rho(\omega))$ is a constant independent from $\boldsymbol{x}$, $f(\boldsymbol{x})$ is a linear combination of multi-linear functions, thus $f(\boldsymbol{x})$ is also multi-linear.
    
    Since $f(\boldsymbol{x})$ is multi-linear, its partial derivative is 
    \[\frac{\partial f}{\partial x_{k,i,s}}(\boldsymbol{x})=\frac{f\left(\boldsymbol{x}\vee (1-\sum_{s'\in G_k , s'\neq s} x_{k,i,s'})\boldsymbol{e}_{k,i,s}\right)-f(\boldsymbol{x}\wedge  \bar{\boldsymbol{e}}_{k,i,s})}{1-\sum_{s'\in G_k , s'\neq s} x_{k,i,s'}}\]
    Here $\wedge$ is the coordinate-wise minimal, and $\vee$ is the coordinate-wise maximal. $\boldsymbol{e}_{k,i,s}$ is the basis vector which takes $1$ only for the component indexed $(k,i,s)$, and $0$ for the other components. $\bar{\boldsymbol{e}}_{k,i,s} $ is the vector that take $0$ for the component indexed $(k,i,s)$ and $1$ for the other components.
    
    We define two mappings $\rho_{\vee}^{k,i,s},\rho_{\wedge}^{k,i,s}: \Omega \longrightarrow \Omega$. For $\omega= (\omega_{k',i'})_{k',i'\in \Gamma}, \ \Gamma= \{(k',i')\mid 1\leq k'\leq K, 1\leq i'\leq r_k,\ i,k\in\mathbb{N}\}$. 
    $\rho_{\vee}^{k,i,s}$ and $\rho_{\wedge}^{k,i,s}$ only change the component indexed $(k,i)$, for any $(k',i')\neq (k,i)$, $(\rho_{\vee}^{k,i,s}(\omega))_{k',i'} = (\rho_{\wedge}^{k,i,s}(\omega))_{k',i'}=\omega_{k',i'}$. For the component indexed $(k,i)$, let
    \begin{align}
        \left(\rho_{\vee}^{k,i,s}(\omega)\right)_{k,i} = \left\{
        \begin{aligned}
            &s \quad \quad &\mbox{if } \omega_{k,i}=\circ\\
            &\omega_{k,i} &\mbox{if }\omega_{k,i}\neq \circ
        \end{aligned}
        \right.
    \end{align}
    and
    \begin{align}
        \left(\rho_{\wedge}^{k,i,s}(\omega)\right)_{k,i} = \left\{
        \begin{aligned}
            & \circ \quad \quad &\mbox{if } \omega_{k,i}=s\\
            &\omega_{k,i} &\mbox{if }\omega_{k,i}\neq s
        \end{aligned}
        \right.
    \end{align}
    We make the following important claim
    \begin{claim}\label{cla:coupling}
    \begin{equation}
        f\left(\boldsymbol{x}\vee \left(1-\sum_{s'\in G_k, s'\neq s} x_{k,i,s'}\right)\boldsymbol{e}_{k,i,s}\right) = \E_{\omega\sim \mbox{pre-EXT}_{PM}(\boldsymbol{x})}\left[g\left(\rho\left(\rho_{\vee}^{k,i,s}(\omega)\right)\right)\right]
    \end{equation}
    and 
    \begin{equation}\label{eq:cp2}
        f(\boldsymbol{x}\wedge  \bar{\boldsymbol{e}}_{k,i,s})=\E_{\omega\sim \mbox{pre-EXT}_{PM}(\boldsymbol{x})}\left[g\left(\rho\left(\rho_{\wedge}^{k,i,s}(\omega)\right)\right)\right].
    \end{equation}
    \end{claim}
    \begin{proof}[proof of \cla{coupling}]
        Let $\boldsymbol{x}^{\vee,k,i,s}:= \boldsymbol{x}\vee \left(1-\sum_{s'\in G_k, s'\neq s} x_{k,i,s'}\right)\boldsymbol{e}_{k,i,s}$.
        \[f(\boldsymbol{x}^{\vee,k,i,s}) = \E_{\omega\sim \mbox{pre-EXT}_{PM}(\boldsymbol{x}^{\vee,k,i,s})}\left[g(\rho(\omega))\right]\]
        Let $\omega \sim \mbox{pre-EXT}_{PM}(\boldsymbol{x})$, it's enough to show $\rho_{\vee}^{k,i,s}(\omega)\sim \mbox{pre-EXT}(\boldsymbol{x}^{\vee,k,i,s})$. Let $\omega_{\vee}\sim \mbox{pre-EXT}_{PM}(\boldsymbol{x}^{\vee,k,i,s}) $. For $(k',i')\neq (k,i)$ and $s'\in G_{k'}$, $\Pr((\omega_{\vee})_{k',i'}=s')=\Pr((\rho_{\vee}^{k,i,s})_{k',i'}=s') = x_{k',i',s'}$, and $\Pr((\omega_{\vee})_{k',i'}=\circ)=\Pr((\rho_{\vee}^{k,i,s}(\omega))_{k',i'}=\circ) = 1-\sum_{s'\in G_{k'}} x_{k',i',s'}$. For $(k,i)$-component and $s'\neq s$, $\Pr((\omega_{\vee})_{k,i}=s')=x_{k,i,s'} = \Pr((\rho_{\vee}^{k,i,s}(\omega))_{k,i}=s')$. For $s$ and $\circ$,
        \[\Pr((\omega_{\vee})_{k,i}=s)=1-\sum_{s'\in G_{k},s'\neq s}x_{k,i,s'}, \quad \quad \Pr((\omega_{\vee})_{k,i}=\circ)=0\]
        $(\rho_{\vee}^{k,i,s}(\omega))_{k,i}=s$ whenever $\omega \in \{s,\circ\}$, thus 
        \[\Pr((\rho_{\vee}^{k,i,s}(\omega))_{k,i}=s)=1-\sum_{s'\in G_{k},s'\neq s}x_{k,i,s'}=\Pr((\omega_{\vee})_{k,i}=s)\]
        Since $\rho_{\vee}^{k,i,s}(\omega))_{k,i}$ never be $\circ$, $\Pr(\rho_{\vee}^{k,i,s}(\omega))_{k,i}=\circ)=0$. Thus $\rho_{\vee}^{k,i,s}(\omega)\sim \mbox{pre-EXT}(\boldsymbol{x}^{\vee,k,i,s})$, and
        \[f(\boldsymbol{x}^{\vee,k,i,s}) = \E_{\omega\sim \mbox{pre-EXT}_{PM}(\boldsymbol{x}^{\vee,k,i,s})}\left[g(\rho(\omega))\right]= \E_{\omega\sim \mbox{pre-EXT}_{PM}(\boldsymbol{x})}\left[g\left(\rho\left(\rho_{\vee}^{k,i,s}(\omega)\right)\right)\right].\]
        In brief, whenever $\omega_{k,i}\in\{s,\circ\}$, $ \rho_{\vee}^{k,i,s}(\omega)=s$. That is, $\Pr((\rho_{\vee}^{k,i,s}(\omega))_{k,i}=s)= \Pr (\omega_{k,i}\in \{s,\circ\}) = 1-\Pr(\omega_{k,i}\notin \{s,\circ\})=1-\sum_{s'\neq s,s'\in G_k}x_{k,i,s'}$ which is the same as a sample in $\mbox{pre-EXT}_{PM}(\boldsymbol{x}^{\vee,k,i,s})$.

        For \eq{cp2}, we can define $\boldsymbol{x}^{\wedge,k,i,s}:=\boldsymbol{x}\bar{\boldsymbol{e}}_{k,i,s}$ and let $\omega_{\wedge}\sim \mbox{pre-EXT}_{PM}(\boldsymbol{x}^{\wedge,k,i,s})$. One can check that for $(k',i',s')\neq (k,i,s)$, $\Pr((\rho_{\wedge}^{k,i,s}(\omega))_{k',i'}=s')=\Pr((\omega_{\wedge})_{k',i'}=s')=x_{k',i',s'}$ and $\Pr((\rho_{\wedge}^{k,i,s}(\omega))_{k',i'}=\circ)=\Pr((\omega_{\wedge})_{k',i'}=\circ)=1-\sum_{s'\in G_{k'}}x_{k',i',s'}$. For $(k,i,s)$, 
        \[\Pr((\rho_{\wedge}^{k,i,s}(\omega))_{k,i}=\circ)= \Pr((\omega_{\wedge})_{k,i}=\circ)=1-\sum_{s'\in G_k,s'\neq s}x_{k,i,s'}\] and
        \[\Pr((\rho_{\wedge}^{k,i,s}(\omega))_{k,i}=s)=\Pr((\rho_{\wedge}^{k,i,s}(\omega))_{k,i}=s)=0.\]
        So $\rho_{\wedge}^{k,i,s}(\omega)\sim \mbox{pre-EXT}(\boldsymbol{x}^{\wedge,k,i,s})$ and \eq{cp2} holds.
    \end{proof}
    
    If $\omega_{k,i}=\circ$ or $\omega_{k,i}=s$, one can check that $\left(\rho_{\vee}^{k,i,s}(\omega)\right)_{k,i}=s$ and $\left(\rho_{\wedge}^{k,i,s}(\omega)\right)_{k,i}=\omega_{k,i}=\circ$, therefore $\rho(\rho_{\vee}^{k,i,s}(\omega))=\rho(\rho_{\wedge}^{k,i,s}(\omega))\cup \{s\}$, so $\rho\left(\rho_{\wedge}^{k,i,s}(\omega)\right)\subseteq \rho\left(\rho_{\vee}^{k,i,s}(\omega)\right)$. if $\omega_{k,i}\notin \{\circ,s\}$, then $\rho_{\vee}^{k,i,s}(\omega)=\rho_{\wedge}^{k,i,s}(\omega)$, so $\rho\left(\rho_{\wedge}^{k,i,s}(\omega)\right)= \rho\left(\rho_{\vee}^{k,i,s}(\omega)\right)$. Thus we proved $\rho\left(\rho_{\wedge}^{k,i,s}(\omega)\right)\subseteq \rho\left(\rho_{\vee}^{k,i,s}(\omega)\right)$ for any $\omega$. Since $g$ is monotone, we have
    \begin{align*}
        \frac{\partial f}{\partial x_{k,i,s}}(\boldsymbol{x})&=\frac{\E\limits_{\omega\sim \mbox{pre-EXT}_{PM}(\boldsymbol{x})}\left[g\left(\rho\left(\rho_{\vee}^{k,i,s}(\omega)\right)\right)-g\left(\rho\left(\rho_{\wedge}^{k,i,s}(\omega)\right)\right)\right]}{1-\sum_{s'\in G_k , s'\neq s} x_{k,i,s'}}
        \geq 0
    \end{align*}
    Note that $\rho_{\wedge}^{k,i,s}(\omega)\neq \rho_{\vee}^{k,i,s}(\omega)$ only happens when $\omega_{k,i} =s$ or $\omega_{k,i} =\circ$, thus,
    \begin{align*}
        \left|\frac{\partial f}{\partial x_{k,i,s}}(\boldsymbol{x})\right|&\leq \frac{M\Pr(\omega_{k,i}\in \{s,\circ\})}{1-\sum_{s'\in G_k , s'\neq s} x_{k,i,s'}}\\
        &= \frac{M(1-\sum_{s'\in G_k , s'\neq s} x_{k,i,s'})}{1-\sum_{s'\in G_k , s'\neq s} x_{k,i,s'}}=M
    \end{align*}
    which shows that $\|\nabla f\|_{\infty}\leq M$. Since $\nabla f\in \mathbb{R}^{\sum_{k=1}^K r_k|G_k|}$, $\|\nabla f\|_2 \leq \sqrt{\sum_{k=1}^K r_k|G_k|}\|\nabla f\|_{\infty}=M\sqrt{\sum_{k=1}^K r_k|G_k|}$. Therefore, $f$ is $M\sqrt{\sum_{k=1}^K r_k|G_k|}$-lipschitz continuous.
    
    Since the partial derivative of a multi-linear function is also multi-linear, $\frac{\partial f}{\partial x_{k,i,s}}(\boldsymbol{x})$ is multi-linear for any $k,i,s$. Then the second derivative of $f$ can be writen as
    \begin{align*}
        \frac{\partial^2 f}{\partial x_{k_1,i_1,s_1} \partial x_{k_2,i_2,s_2}}(\boldsymbol{x}) &= \frac{\frac{\partial f}{\partial x_{k_2,i_2,s_2}}\left(\boldsymbol{x}\vee (1-\sum_{s'\in G_{k_1} , s'\neq s_1} x_{k,i,s'})\boldsymbol{e}_{k_1,i_1,s_1}\right)-\frac{\partial f}{\partial x_{k_2,i_2,s_2}}\left(\boldsymbol{x}\wedge  \bar{\boldsymbol{e}}_{k_1,i_1,s_1}\right)}{1-\sum_{s'\in G_{k_1} , s'\neq s_1} x_{k_1,i_1,s'}}
        \\&=\frac{f\left(\boldsymbol{x}\vee (1-\sum_{s'\in G_{k_1} , s'\neq s_1} x_{k',i',s'})\boldsymbol{e}_{k_1,i_1,s_1}\vee (1-\sum_{s'\in G_{k_1} , s'\neq s_2} x_{k',i',s'})\boldsymbol{e}_{k_2,i_2,s_2}\right)}{(1-\sum_{s'\in G_{k_2} , s'\neq s_1} x_{k_1,i_1,s'})(1-\sum_{s'\in G_{k_2} , s'\neq s_2} x_{k_2,i_2,s'})}
        \\&\quad - \frac{f\left((\boldsymbol{x}\wedge  \bar{\boldsymbol{e}}_{k_2,i_2,s_2})\vee (1-\sum_{s'\in G_{k_1} , s'\neq s_2} x_{k',i',s'})\boldsymbol{e}_{k_1,i_1,s_1}\right)}{(1-\sum_{s'\in G_{k_1} , s'\neq s_1} x_{k_1,i_1,s'})(1-\sum_{s'\in G_{k_2} , s'\neq s_2} x_{k_2,i_2,s'})}
        \\&\quad - \frac{f\left((\boldsymbol{x}\vee (1-\sum_{s'\in G_{k_2} , s'\neq s_2} x_{k',i',s'})\boldsymbol{e}_{k_2,i_2,s_2})\wedge  \bar{\boldsymbol{e}}_{k_1,i_1,s_1}\right)}{(1-\sum_{s'\in G_{k_2} , s'\neq s_1} x_{k_1,i_1,s'})(1-\sum_{s'\in G_{k_2} , s'\neq s_2} x_{k_2,i_2,s'})}
        \\&\quad + \frac{f\left(\boldsymbol{x}\wedge  \bar{\boldsymbol{e}}_{k_2,i_2,s_2}\wedge  \bar{\boldsymbol{e}}_{k_1,i_1,s_1}\right)}{(1-\sum_{s'\in G_{k_2} , s'\neq s_1} x_{k_1,i_1,s'})(1-\sum_{s'\in G_{k_2} , s'\neq s_2} x_{k_2,i_2,s'})}
    \end{align*}
    
    To prove the DR-submodularity of $f$, We define $4$ mappings $\rho_{\vee,\vee}^{k_1,i_1,s_1,k_2,i_2,s_2},\rho_{\vee,\wedge}^{k_1,i_1,s_1,k_2,i_2,s_2},\rho_{\wedge,\vee}^{k_1,i_1,s_1,k_2,i_2,s_2}$, $\rho_{\wedge,\wedge}^{k_1,i_1,s_1,k_2,i_2,s_2}:\Omega\longrightarrow \Omega$.
    \begin{align*}
        &\rho_{\vee,\vee}^{k_1,i_1,s_1,k_2,i_2,s_2}(\omega)=\rho_{\vee}^{k_1,i_1,s_1}\left(\rho_{\vee}^{k_2,i_2,s_2}(\omega)\right) \quad &\rho_{\vee,\wedge}^{k_1,i_1,s_1,k_2,i_2,s_2}(\omega)=\rho_{\vee}^{k_1,i_1,s_1}\left(\rho_{\wedge}^{k_2,i_2,s_2}(\omega)\right)
        \\& \rho_{\wedge,\vee}^{k_1,i_1,s_1,k_2,i_2,s_2}(\omega)=\rho_{\wedge}^{k_1,i_1,s_1}\left(\rho_{\vee}^{k_2,i_2,s_2}(\omega)\right) \quad &\rho_{\wedge,\wedge}^{k_1,i_1,s_1,k_2,i_2,s_2}(\omega)=\rho_{\wedge}^{k_1,i_1,s_1}\left(\rho_{\wedge}^{k_2,i_2,s_2}(\omega)\right)
     \end{align*}
     Same as \cla{coupling}, we have
     \begin{align*}
         &\frac{\partial f}{\partial x_{k_1,i_1,s_1} \partial x_{k_2,i_2,s_2}}(\boldsymbol{x}) \\&= 
         \frac{\E\limits_{\omega\sim \mbox{pre-EXT}_{PM}(\boldsymbol{x})}\left[\rho\left(\rho_{\vee,\vee}^{k_1,i_1,s_1,k_2,i_2,s_2}(\omega)\right)-\rho\left(\rho_{\vee,\wedge}^{k_1,i_1,s_1,k_2,i_2,s_2}(\omega)\right)\right]}{(1-\sum_{s'\in G_{k_2} , s'\neq s_1} x_{k_1,i_1,s'})(1-\sum_{s'\in G_{k_2} , s'\neq s_2} x_{k_2,i_2,s'})}
         \\&\quad \quad \quad \quad\quad \quad +\frac{\E\limits_{\omega\sim \mbox{pre-EXT}_{PM}(\boldsymbol{x})}\left[-\rho\left(\rho_{\wedge,\vee}^{k_1,i_1,s_1,k_2,i_2,s_2}(\omega)\right)+\left(\rho_{\wedge,\wedge}^{k_1,i_1,s_1,k_2,i_2,s_2}(\omega)\right)\right]}{(1-\sum_{s'\in G_{k_2} , s'\neq s_1} x_{k_1,i_1,s'})(1-\sum_{s'\in G_{k_2} , s'\neq s_2} x_{k_2,i_2,s'})}.
     \end{align*}
     
     We first consider the situation where $k_1=k_2$ and $i_1=i_2$. Recall that $\Pr(\omega_{k_1,i_1})\in \left\{x_{k_1,i_1,s}\mid s\in G_{k}\right\}\cup \{1-\sum_{s\in G_{k_1}}x_{k_1,i_1,s}\} $, so $\frac{\partial f}{\partial x_{k_1,i_1,s_1} \partial x_{k_2,i_2,s_2}}\Pr(\omega_{k_1,i_1})=0$ when $k_1=k_2$ and $i_1=i_2$. In this situation,
     \begin{equation}\label{eq:2de}
          \begin{aligned}
         \frac{\partial f}{\partial x_{k_1,i_1,s_1} \partial x_{k_2,i_2,s_2}}(\boldsymbol{x}) &= 
         \sum_{\omega\in \Omega} g(\rho(\omega))\frac{\partial f}{\partial x_{k_1,i_1,s_1} \partial x_{k_2,i_2,s_2}}\Pr(\omega)
         \\&=\sum_{\omega\in \Omega} g(\rho(\omega))\left(\prod\limits_{\myatop{k'\in [K],i'\in [r_{k'}]}{k'\neq k_1 \mbox{ or }i'\neq i_1}}\Pr (\omega_{k',i'})\right) \frac{\partial f}{\partial x_{k_1,i_1,s_1} \partial x_{k_2,i_2,s_2}}\Pr(\omega_{k_1,i_1}) 
         \\& = 0 \quad\quad\quad \mbox{if } k_1= k_2 \mbox{ and } i_1= i_2
     \end{aligned}
     \end{equation}
     
    Then we consider the situation that $k_1\neq k_2$ or $i_1\neq i_2$, 
    
    \noindent \textbf{Case 1 : $k_1\neq k_2 \mbox{ or } i_1\neq i_2,\ \omega_{k_1,i_1}\in \{s_1,\circ\}$ and $\omega_{k_2,i_2}\in\{s_2,\circ\}$}. In this case,
    \begin{align*}
        &\left(\rho_{\vee,\vee}^{k_1,i_1,s_1,k_2,i_2,s_2}(\omega)\right)_{k_1,i_1}=s_1 \quad &\left(\rho_{\vee,\vee}^{k_1,i_1,s_1,k_2,i_2,s_2}(\omega)\right)_{k_2,i_2}=s_2\\ &\left(\rho_{\vee,\wedge}^{k_1,i_1,s_1,k_2,i_2,s_2}(\omega)\right)_{k_1,i_1}=s_1 & \left(\rho_{\vee,\wedge}^{k_1,i_1,s_1,k_2,i_2,s_2}(\omega)\right)_{k_2,i_2}=\circ
        \\& \left(\rho_{\wedge,\vee}^{k_1,i_1,s_1,k_2,i_2,s_2}(\omega)\right)_{k_1,i_1}=\circ &\left(\rho_{\wedge,\vee}^{k_1,i_1,s_1,k_2,i_2,s_2}(\omega)\right)_{k_2,i_2}=s_2
        \\&\left(\rho_{\wedge,\wedge}^{k_1,i_1,s_1,k_2,i_2,s_2}(\omega)\right)_{k_1,i_1}=\circ &\left(\rho_{\wedge,\wedge}^{k_1,i_1,s_1,k_2,i_2,s_2}(\omega)\right)_{k_2,i_2}=\circ
     \end{align*}
     $\rho_{\vee,\vee}^{k_1,i_1,s_1,k_2,i_2,s_2}(\omega),\rho_{\vee,\vee}^{k_1,i_1,s_1,k_2,i_2,s_2}(\omega),\rho_{\vee,\vee}^{k_1,i_1,s_1,k_2,i_2,s_2}(\omega),\rho_{\vee,\vee}^{k_1,i_1,s_1,k_2,i_2,s_2}(\omega)$ are equal in all but above two components $k_1,i_2$ and $k_2,i_2$. Thus
     \begin{align*}
         \rho\left(\rho_{\vee,\vee}^{k_1,i_1,s_1,k_2,i_2,s_2}(\omega)\right) &= \rho\left(\rho_{\wedge,\wedge}^{k_1,i_1,s_1,k_2,i_2,s_2}(\omega)\right)\cup \{s_1,s_2\} 
         \\\rho\left( \rho_{\vee,\wedge}^{k_1,i_1,s_1,k_2,i_2,s_2}(\omega)\right) &= \rho\left(\rho_{\wedge,\wedge}^{k_1,i_1,s_1,k_2,i_2,s_2}(\omega)\right)\cup \{s_1\}
         \\\rho\left(\rho_{\wedge,\vee}^{k_1,i_1,s_1,k_2,i_2,s_2}(\omega) \right)&= \rho\left(\rho_{\wedge,\wedge}^{k_1,i_1,s_1,k_2,i_2,s_2}(\omega)\right)\cup \{s_2\} 
     \end{align*}
     For monotone submodular $g$, by \lem{monotone submodular},
     \begin{align*}
         &g\left(\rho\left(\rho_{\wedge,\wedge}^{k_1,i_1,s_1,k_2,i_2,s_2}(\omega)\right)\cup \{s_1,s_2\} \right)-g\left(\rho\left( \rho_{\wedge,\wedge}^{k_1,i_1,s_1,k_2,i_2,s_2}(\omega)\right)\cup \{s_1\}\right)\\&\quad \quad \quad \geq g\left(\rho\left(\rho_{\wedge,\wedge}^{k_1,i_1,s_1,k_2,i_2,s_2}(\omega) \right)\cup\{s_2\}\right)-g\left(\rho\left(\rho_{\wedge,\wedge}^{k_1,i_1,s_1,k_2,i_2,s_2}(\omega) \right)\right)
     \end{align*}
     thus,
     \begin{align*}
         \rho\left(\rho_{\vee,\vee}^{k_1,i_1,s_1,k_2,i_2,s_2}(\omega)\right)-\rho\left(\rho_{\vee,\wedge}^{k_1,i_1,s_1,k_2,i_2,s_2}(\omega)\right)-\rho\left(\rho_{\wedge,\vee}^{k_1,i_1,s_1,k_2,i_2,s_2}(\omega)\right)+\left(\rho_{\wedge,\wedge}^{k_1,i_1,s_1,k_2,i_2,s_2}(\omega)\right)\leq 0
     \end{align*}

     \noindent \textbf{Case 2 : $k_1\neq k_2 \mbox{ or } i_1\neq i_2,\ \omega_{k_1,i_1}\in \{s_1,\circ\}\mbox{ and } \omega_{k_2,i_2}\notin\{s_2,\circ\}$}. In this case, 
     \begin{align*}
        \rho_{\vee}^{k_2,i_2,s_2}(\omega)=\rho_{\wedge}^{k_2,i_2,s_2}(\omega)=\omega
     \end{align*}
     thus,
     \begin{align*}
        &\rho_{\vee,\vee}^{k_1,i_1,s_1,k_2,i_2,s_2}(\omega)=\rho_{\vee}^{k_1,i_1,s_1}\left(\omega\right) \quad &\rho_{\vee,\wedge}^{k_1,i_1,s_1,k_2,i_2,s_2}(\omega)=\rho_{\vee}^{k_1,i_1,s_1}\left(\omega\right)
        \\& \rho_{\wedge,\vee}^{k_1,i_1,s_1,k_2,i_2,s_2}(\omega)=\rho_{\wedge}^{k_1,i_1,s_1}\left(\omega\right) \quad &\rho_{\wedge,\wedge}^{k_1,i_1,s_1,k_2,i_2,s_2}(\omega)=\rho_{\wedge}^{k_1,i_1,s_1}\left(\omega\right)
      \end{align*}
      and 
      \begin{align*}
          \rho\left(\rho_{\vee,\vee}^{k_1,i_1,s_1,k_2,i_2,s_2}(\omega)\right)-\rho\left(\rho_{\vee,\wedge}^{k_1,i_1,s_1,k_2,i_2,s_2}(\omega)\right)-\rho\left(\rho_{\wedge,\vee}^{k_1,i_1,s_1,k_2,i_2,s_2}(\omega)\right)+\left(\rho_{\wedge,\wedge}^{k_1,i_1,s_1,k_2,i_2,s_2}(\omega)\right)=0
      \end{align*}
      
      \noindent \textbf{Case 3 : $k_1\neq k_2 \mbox{ or } i_1\neq i_2,\ \omega_{k_1,i_1}\notin \{s_1,\circ\}\mbox{ and } \omega_{k_2,i_2}\in\{s_2,\circ\}$}. Since $(k_1,i_1)\neq(k_2,i_2)$, $\rho_{\vee}^{k_2,i_2,s_2}$ and $\rho_{\wedge}^{k_2,i_2,s_2}$ do not change the $k_1,i_1$ component of $\omega$, that is,
     \begin{align*}
        \left(\rho_{\vee}^{k_2,i_2,s_2}(\omega)\right)_{k_1,i_1}=\left(\rho_{\wedge}^{k_2,i_2,s_2}(\omega)\right)_{k_1,i_1} = \omega_{k_1,i_1}
     \end{align*}
     Therefore,
     \begin{align*}
        &\rho_{\vee,\vee}^{k_1,i_1,s_1,k_2,i_2,s_2}(\omega)=\rho_{\vee}^{k_2,i_2,s_2}\left(\omega\right) \quad &\rho_{\vee,\wedge}^{k_1,i_1,s_1,k_2,i_2,s_2}(\omega)=\rho_{\wedge}^{k_2,i_2,s_2}\left(\omega\right)
        \\& \rho_{\wedge,\vee}^{k_1,i_1,s_1,k_2,i_2,s_2}(\omega)=\rho_{\vee}^{k_2,i_2,s_2}\left(\omega\right) \quad &\rho_{\wedge,\wedge}^{k_1,i_1,s_1,k_2,i_2,s_2}(\omega)=\rho_{\wedge}^{k_2,i_2,s_2}\left(\omega\right)
      \end{align*}
      and 
      \begin{align*}
          \rho\left(\rho_{\vee,\vee}^{k_1,i_1,s_1,k_2,i_2,s_2}(\omega)\right)-\rho\left(\rho_{\vee,\wedge}^{k_1,i_1,s_1,k_2,i_2,s_2}(\omega)\right)-\rho\left(\rho_{\wedge,\vee}^{k_1,i_1,s_1,k_2,i_2,s_2}(\omega)\right)+\left(\rho_{\wedge,\wedge}^{k_1,i_1,s_1,k_2,i_2,s_2}(\omega)\right)=0
      \end{align*}
      
      \noindent \textbf{Case 4 : $k_1\neq k_2 \mbox{ or } i_1\neq i_2,\ \omega_{k_1,i_1}\notin \{s_1,\circ\}\mbox{ and } \omega_{k_2,i_2}\notin\{s_2,\circ\}$}. In this case,
      \[\rho_{\vee,\vee}^{k_1,i_1,s_1,k_2,i_2,s_2}(\omega)=\rho_{\vee,\wedge}^{k_1,i_1,s_1,k_2,i_2,s_2}(\omega)=\rho_{\wedge,\vee}^{k_1,i_1,s_1,k_2,i_2,s_2}(\omega)=\rho_{\wedge,\wedge}^{k_1,i_1,s_1,k_2,i_2,s_2}(\omega)=\omega\]
      thus
      \begin{align*}
          \rho\left(\rho_{\vee,\vee}^{k_1,i_1,s_1,k_2,i_2,s_2}(\omega)\right)-\rho\left(\rho_{\vee,\wedge}^{k_1,i_1,s_1,k_2,i_2,s_2}(\omega)\right)-\rho\left(\rho_{\wedge,\vee}^{k_1,i_1,s_1,k_2,i_2,s_2}(\omega)\right)+\left(\rho_{\wedge,\wedge}^{k_1,i_1,s_1,k_2,i_2,s_2}(\omega)\right)=0.
      \end{align*}
      
      In all $4$ cases above, whatever $\omega$ is,
      \begin{align*}
          \rho\left(\rho_{\vee,\vee}^{k_1,i_1,s_1,k_2,i_2,s_2}(\omega)\right)-\rho\left(\rho_{\vee,\wedge}^{k_1,i_1,s_1,k_2,i_2,s_2}(\omega)\right)-\rho\left(\rho_{\wedge,\vee}^{k_1,i_1,s_1,k_2,i_2,s_2}(\omega)\right)+\left(\rho_{\wedge,\wedge}^{k_1,i_1,s_1,k_2,i_2,s_2}(\omega)\right)\leq 0
      \end{align*}
      holds. Thus,
     \begin{align}\label{eq:2de2}
         \frac{\partial f}{\partial x_{k_1,i_1,s_1} \partial x_{k_2,i_2,s_2}}(\boldsymbol{x}) \leq 0 \quad \quad \mbox{if } k_1\neq k_2 \mbox{ or } i_1\neq i_2.
     \end{align}
     Combining \eq{2de} and \eq{2de2}, $\frac{\partial f}{\partial x_{k_1,i_1,s_1} \partial x_{k_2,i_2,s_2}}(\boldsymbol{x}) \leq 0 $ for any $k_1,i_1,s_1,k_2,i_2,s_2$, which shows the DR-submodularity of $f$.
     
\end{proof}

\begin{customcoro}{5.4}
    There is an algorithm attaining the expected $(1-1/e)$-regret of
    \[\mathcal{R}_{1-1/e}(T)\leq O\left(\left(\sum_{k=1}^K r_k|G_k|\right)^{5/3}T^{2/3}\log T\right)\]
    on any $(\mathcal{S}_{PM},\mathcal{G}_{MS})$-bandit.
\end{customcoro}

\begin{proof}[Proof of \cor{BMSMPM}]
    Since $\mbox{EXT}_{PM}$ satisfies the conditions in \lem{reduction} and the dimension of $\mathcal{K}$ is $d=\sum_{k=1}^K r_k|G_k|$. This is a direct corollary of \lem{reduction}.
\end{proof}

\subsection{Proof of \lem{BSSM} and \cor{BSSM}}

\begin{customlem}{5.5}
    For $\mathcal{G}_{SS}$, the extension mapping $\mbox{EXT}_{SS}:\mathcal{K}\rightarrow \Delta(\mathcal{S}_{OL})$ satisfies the conditions in \lem{reduction}. Moreover, $\mathcal{K}$ is in a $|G|^2-|G| $ dimensional real vector space. For any $g\in \mathcal{G}_{SS}$, the continuous extension $f(\boldsymbol{x})=\E_{s\in \mbox{EXT}(\boldsymbol{x})}[g(s)]$ is $M|G|$-lipschitz. 
\end{customlem}

\begin{proof}[Proof of \lem{BSSM}]
    The condition $1$ and $3$ of \lem{reduction} are obviously satisfied. Now we check the multi-linearity of $f$. For $\boldsymbol{x}\in\mathcal{K}$, we write $\boldsymbol{x}=(x_{i,s})_{i\in |G|,s\in G'}$. Consider $S \in \mathcal{S}_{OL}$, given $\boldsymbol{x}$, the probability of $S$ in distribution $\mbox{EXT}_{SS}(\boldsymbol{x})$ is $\Pr(S) =\prod_{i=1}^{|G|}\Pr(S_i)$. For any $1\leq i\leq |G|$, $\Pr(S_i)\in \left\{x_{i,s}\mid s\in G'\right\}\cup \{1-\sum_{s\in G'}x_{i,s}\}$, thus $\Pr(S)$ is multi-linear with respect to the variables $x_{i,s}$. Then we write $f(\boldsymbol{x})$ as follows,
    \begin{align*}
        f(\boldsymbol{x})&= \E_{S\sim \mbox{EXT}_{SS}(\boldsymbol{x})} \left[g(S)\right]
        \\&=\sum_{S\in \mathcal{S}_{OL}} \Pr(S) g(S)
    \end{align*}
    Since $g(S)$ is a constant independent from $\boldsymbol{x}$, $f(\boldsymbol{x})$ is a linear combination of multi-linear functions, thus $f(\boldsymbol{x})$ is also multi-linear.

    $\mbox{EXT}(\boldsymbol{0})$ assigns probability $1$ to the ordered list $\{\circ\}^{|G|}$, thus $f(\boldsymbol{0})=g(\{\circ\}^{|G|})=0$. Next we check the monotonicity and DR-submodularity of $f(\boldsymbol{x})$.
    
     Define $\rho_{\vee}^{i,s},\rho_{\wedge}^{i,s}:\mathcal{S}_{OL}\longrightarrow \mathcal{S}_{OL}$. $\rho_{\vee}^{i,s}(S)\neq S$ only when the $i$-th position of $S$ is $\circ$, $\rho_{\vee}^{i,s}(S)$ change the $i$-th position of $S$ to $s$ and keep other positions unchanged. $\rho_{\wedge}^{i,s}(S)\neq S$ only when the $i$-th position of $S$ is $s$, $\rho_{\vee}^{i,s}(S)$ change the $i$-th position of $S$ to $\circ$ and keep other positions unchanged. Then,
    \begin{align*}
        \frac{\partial f}{\partial x_{i,s}}(\boldsymbol{x}) = \frac{\E\limits_{S\sim \mbox{EXT}_{SS}(\boldsymbol{x})}\left[g(\rho_{\vee}^{i,s}(S))-g(\rho_{\wedge}^{i,s}(S))\right]}{1-\sum_{s'\in G, s'\neq s }x_{i,s'}}
    \end{align*}
    Let $S^{\leq k}$ be the set containing the first $k$ elements in the ordered list $S$. Then $ (\rho_{\wedge}^{i,s}(S))^{\leq k}\subseteq (\rho_{\vee}^{i,s}(S))^{\leq k},\ \forall k$, recall $g_k$ is monotone and $\lambda_k>0$ for any $k$,
    \begin{align*}
        \frac{\partial f}{\partial x_{i,s}}(\boldsymbol{x}) = \frac{\E\limits_{S\sim \mbox{EXT}_{SSM}(\boldsymbol{x})}\left[\sum_{k=1}^{|G_k|}\lambda_k \left(g_k((\rho_{\vee}^{i,s}(S))^{\leq k})-g((\rho_{\wedge}^{i,s}(S))^{\leq k})\right)\right]}{1-\sum_{s'\in G, s'\neq s }x_{i,s'}}\geq 0
    \end{align*}
    
    Thus $f$ is monotone. Since $\Pr(\rho_{\vee}^{i,s}(S)\neq \rho_{\wedge}^{i,s}(S))\leq 1-\sum_{s'\in G, s'\neq s }x_{i,s'}$ and $\frac{\partial f}{\partial x_{i,s}}(\boldsymbol{x})\leq M$. Thus $\|\nabla f(\boldsymbol{x})\|_{\infty}\leq M$, $\|\nabla f(\boldsymbol{x})\|_2\leq M\sqrt{|G|(|G|-1)}\leq M|G|$, $f(\boldsymbol{x})$ is $M|G|$-lipschitz.
    
    We then define
        \begin{align*}
        &\rho_{\vee,\vee}^{i_1,s_1,i_2,s_2}(S)=\rho_{\vee}^{i_1,s_1}\left(\rho_{\vee}^{i_2,s_2}(S)\right) \quad &\rho_{\vee,\wedge}^{i_1,s_1,i_2,s_2}(S)=\rho_{\vee}^{i_1,s_1}\left(\rho_{\wedge}^{i_2,s_2}(S)\right)
        \\& \rho_{\wedge,\vee}^{i_1,s_1,i_2,s_2}(S)=\rho_{\wedge}^{i_1,s_1}\left(\rho_{\vee}^{i_2,s_2}(S)\right) \quad &\rho_{\wedge,\wedge}^{i_1,s_1,i_2,s_2}(S)=\rho_{\wedge}^{i_1,s_1}\left(\rho_{\wedge}^{i_2,s_2}(S)\right)
     \end{align*}
    Then,
    \begin{align*}
        \frac{\partial f}{\partial x_{i_1,s_1}\partial x_{i_2,s_2}}(\boldsymbol{x})&=\frac{\E\limits_{S\sim \mbox{EXT}_{SSM}(\boldsymbol{x})}\left[g(\rho_{\vee,\vee}^{i_1,s_1,i_2,s_2}(S))-g(\rho_{\vee,\wedge}^{i_1,s_1,i_2,s_2}(S))-g(\rho_{\wedge,\vee}^{i_1,s_1,i_2,s_2}(S))+g(\rho_{\wedge,\wedge}^{i_1,s_1,i_2,s_2}(S))\right]}{(1-\sum_{s'\in G, s'\neq s_1 }x_{i_1,s'})(1-\sum_{s'\in G, s'\neq s_2 }x_{i_2,s'})}
    \end{align*}
    We then prove $g(\rho_{\vee,\vee}^{i_1,s_1,i_2,s_2}(S))-g(\rho_{\vee,\wedge}^{i_1,s_1,i_2,s_2}(S))-g(\rho_{\wedge,\vee}^{i_1,s_1,i_2,s_2}(S))+g(\rho_{\wedge,\wedge}^{i_1,s_1,i_2,s_2}(S))\leq 0$ for any $S\in \mathcal{S}$.  It's enough to prove $g_i((\rho_{\vee,\vee}^{i_1,s_1,i_2,s_2}(S))^{\leq i})-g_i((\rho_{\vee,\wedge}^{i_1,s_1,i_2,s_2}(S))^{\leq i})-g_i((\rho_{\wedge,\vee}^{i_1,s_1,i_2,s_2}(S))^{\leq i})+g_i((\rho_{\wedge,\wedge}^{i_1,s_1,i_2,s_2}(S))^{\leq i})\leq 0 $ for any $i\in [|G|]$. Note that if $\max\{i_1,i_2\}>i$ then $g_i((\rho_{\vee,\vee}^{i_1,s_1,i_2,s_2}(S))^{\leq i})-g_i((\rho_{\vee,\wedge}^{i_1,s_1,i_2,s_2}(S))^{\leq i})-g_i((\rho_{\wedge,\vee}^{i_1,s_1,i_2,s_2}(S))^{\leq i})+g_i((\rho_{\wedge,\wedge}^{i_1,s_1,i_2,s_2}(S))^{\leq i})= 0 $, so we now consider the case $\max\{i_1,i_2\}\leq i$. 
    
    \noindent \textbf{Case 1 : $i_1=i_2$}. In this case, since $\Pr(S_{i_1})\in  \left\{x_{i_1,s}\mid s\in G'\right\}\cup \{1-\sum_{s\in G'}x_{i_1,s}\}$, $\frac{\partial^2 \Pr(S_{i_1})}{\partial x_{i_1,s_1}\partial x_{i_2,s_2}}=0$ when $i_1=i_2$. we have
    \begin{align*}
        \frac{\partial^2 f}{\partial x_{i_1,s_1}\partial x_{i_2,s_2}}(\boldsymbol{x})&=\frac{\partial^2}{\partial x_{i_1,s_1}\partial x_{i_2,s_2}}\sum_{S\in\mathcal{S}_{OL}}g(S)\prod_{i=1}^{|G|}\Pr(S_i)
        \\&=\sum_{S\in \mathcal{S}_{OL}}g(S)\left(\prod_{i\neq i_1}\Pr(S_i) \right)\frac{\partial^2 \Pr(S_{i_1})}{\partial x_{i_1,s_1}\partial x_{i_2,s_2}} 
        \\&=0
    \end{align*}
    
    \noindent \textbf{Case 2 : $i_1\neq i_2, S_{i_1}\in\{s_1,\circ\}\mbox{ and } S_{i_2}\in\{s_2,\circ\}$}. In this case,
    \begin{align*}
        &\left(\rho_{\vee,\vee}^{i_1,s_1,i_2,s_2}(S)\right)_{i_1}=s_1 \quad &\left(\rho_{\vee,\vee}^{i_1,s_1,i_2,s_2}(S)\right)_{i_2}=s_2\\ &\left(\rho_{\vee,\wedge}^{i_1,s_1,i_2,s_2}(S)\right)_{i_1}=s_1 & \left(\rho_{\vee,\wedge}^{i_1,s_1,i_2,s_2}(S)\right)_{i_2}=\circ
        \\& \left(\rho_{\wedge,\vee}^{i_1,s_1,i_2,s_2}(S)\right)_{i_1}=\circ &\left(\rho_{\wedge,\vee}^{i_1,s_1,i_2,s_2}(S)\right)_{i_2}=s_2
        \\&\left(\rho_{\wedge,\wedge}^{i_1,s_1,i_2,s_2}(S)\right)_{i_1}=\circ &\left(\rho_{\wedge,\wedge}^{i_1,s_1,i_2,s_2}(S)\right)_{i_2}=\circ
     \end{align*}
     $\rho_{\vee,\vee}^{_1,s_1,i_2,s_2}(S),\rho_{\vee,\vee}^{i_1,s_1,i_2,s_2}(S),\rho_{\vee,\vee}^{i_1,s_1,i_2,s_2}(S),\rho_{\vee,\vee}^{i_1,s_1,i_2,s_2}(S)$ are equal in all but above two components $i_2$ and $i_2$. Thus,
     \begin{align*}
         \left(\rho_{\vee,\vee}^{i_1,s_1,i_2,s_2}(S)\right)^{\leq i} &= \left(\rho_{\wedge,\wedge}^{i_1,s_1,i_2,s_2}(S)\right)^{\leq i}\cup \{s_1,s_2\} 
         \\\left( \rho_{\vee,\wedge}^{i_1,s_1,i_2,s_2}(S)\right)^{\leq i} &= \left(\rho_{\wedge,\wedge}^{i_1,s_1,i_2,s_2}(S)\right)^{\leq i}\cup \{s_1\}
         \\\left(\rho_{\wedge,\vee}^{i_1,s_1,i_2,s_2}(S) \right)^{\leq i}&= \left(\rho_{\wedge,\wedge}^{i_1,s_1,i_2,s_2}(S)\right)^{\leq i}\cup \{s_2\} 
     \end{align*}
     By \lem{monotone submodular},
     \[g_i((\rho_{\vee,\vee}^{i_1,s_1,i_2,s_2}(S))^{\leq i})-g_i((\rho_{\vee,\wedge}^{i_1,s_1,i_2,s_2}(S))^{\leq i})-g_i((\rho_{\wedge,\vee}^{i_1,s_1,i_2,s_2}(S))^{\leq i})+g_i((\rho_{\wedge,\wedge}^{i_1,s_1,i_2,s_2}(S))^{\leq i})\leq 0\]
     Then $\frac{\partial^2 f}{\partial x_{i_1,s_1}\partial x_{i_2,s_2}}(\boldsymbol{x})\leq 0$.

     \noindent \textbf{Case 3 : $i_1\neq i_2, S_{i_1}\notin\{s_1,\circ\}\mbox{ or } S_{i_2}\notin\{s_2,\circ\}$}: If $S_{i_1}\notin\{s_1,\circ\}$, then $\rho_{\vee,\vee}^{i_1,s_1,i_2,s_2}(S) = \rho_{\wedge,\vee}^{i_1,s_1,i_2,s_2}(S)$ and $\rho_{\vee,\wedge}^{i_1,s_1,i_2,s_2}(S) = \rho_{\wedge,\wedge}^{i_1,s_1,i_2,s_2}(S)$. If $S_{i_2}\notin\{s_2,\circ\}$, then $\rho_{\vee,\vee}^{i_1,s_1,i_2,s_2}(S) = \rho_{\vee,\wedge}^{i_1,s_1,i_2,s_2}(S)$ and $\rho_{\wedge,\vee}^{i_1,s_1,i_2,s_2}(S) = \rho_{\wedge,\wedge}^{i_1,s_1,i_2,s_2}(S)$. Either way, we have,
     \[g_i((\rho_{\vee,\vee}^{i_1,s_1,i_2,s_2}(S))^{\leq i})-g_i((\rho_{\vee,\wedge}^{i_1,s_1,i_2,s_2}(S))^{\leq i})-g_i((\rho_{\wedge,\vee}^{i_1,s_1,i_2,s_2}(S))^{\leq i})+g_i((\rho_{\wedge,\wedge}^{i_1,s_1,i_2,s_2}(S))^{\leq i})= 0.\]
    Then $\frac{\partial^2 f}{\partial x_{i_1,s_1}\partial x_{i_2,s_2}}(\boldsymbol{x})= 0$.
    
     In all cases, $\frac{\partial f}{\partial x_{i_1,s_1}\partial x_{i_2,s_2}}(\boldsymbol{x})\leq 0$, thus $f(\boldsymbol{x})$ is DR-submodular.
\end{proof}

\begin{customcoro}{5.6}
    There is an algorithm for attaining the expected $(1-1/e)$-regret of
    \[\mathcal{R}_{1-1/e}(T)\leq O\left((|G|)^{10/3}T^{2/3}\log T\right)\] on any $(\mathcal{S}_{OL},\mathcal{G}_{SS})$-bandit.
\end{customcoro}
    
\begin{proof}[Proof of \cor{BSSM}]
    Since $\mbox{EXT}_{SS}$ satisfies the conditions in \lem{reduction} and the dimension of $\mathcal{K}$ is $d=|G|(|G|-1)=O(|G|^2)$. This is a direct corollary of \lem{reduction}.
\end{proof}

\section{Remark on the Stochastic Submodular Bandit}\label{append:remark}
In this section, we show how our algorithms for adversarial setting can be applied to the stochastic setting proposed by \citet{nie2022explore}. We take the stochastic monotone submodular bandit with cardinality constraint investigated in \cite{nie2022explore} as an example.

\paragraph{Stochastic submodular bandit model} In the stochastic model, there is an unknown distribution $\mathcal{D}$, its support is a set of set functions, we denote the set $supp(\mathcal{D})$. Any set function $g\in supp(\mathcal{D})$ is defined on the power set of the ground set $G$, mapping a subset of $G$ to a reward between $[0,1]$, that is, $g:2^{G}\rightarrow [0,1]$. In $t$-th round, the reward function $g'_t$ is drawn from $\mathcal{D}$ and we can only select a subset $S_t\subseteq G$ such that $|S_t|\leq k$, which is a cardinality constraint. Then we gain reward $g'_t(S_t)$. Note that the model does not need $g'_t$ to be a monotone submodular function, but requires $g=\mathbb{E}_{g'_t\sim \mathcal{D}}[g'_t]$ to be monotone submodular. Our goal is to minimize the $(1-1/e)$-regret:
\[\mathcal{R}^{sto}_{1-1/e}(T) = (1-\frac{1}{e})T\cdot \max_{S^*\subseteq G, |S^*|\leq k}g(S^*) - \mathbb{E}\left[\sum_{t=1}^T g'_t(S_t)\right]=(1-\frac{1}{e})T\cdot \max_{S^*\subseteq G, |S^*|\leq k}g(S^*) - \mathbb{E}\left[\sum_{t=1}^T g(S_t)\right]\]

The last equality is because the randomness of $S_t$ is independent of the randomness of $g'_t$.
\paragraph{Apply our algorithm on stochastic bandit model} Since cardinality constraint is a special case of the partition matroid constraint, we use our algorithm in \sect{BMSMPM}. While applying our algorithm on the stochastic model, we see $g_t = g =\mathbb{E}_{g'_t\sim \mathcal{D}}[g'_t]$ as the online function selected by the adversary, thus the online functions are monotone submodular. Note that, to obtain a regret bound w.r.t.~the online reward function $\{g_t\}_{t=1}^T$, we need to query the value of $g_t$ at some subset $S_t$ in round $t$. However, since the algorithm is actually running on the stochastically realized function sequence $\{g'_t\}_{t=1}^T$, if we query the function value of $S_t$, the feedback is $g'_t(S_t)$ rather than $g_t(S_t) = g(S_t)$. Fortunately, this is not a big issue since $g'_t(S_t)$ is an unbiased estimate of $g_t(S_t)$. Now we go back to the proof of \lem{reduction}, and replace all the $g_{t_q}(S_{t_q})=g(S_{t_q})$ with $g'_{t_q}(S_{t_q})$. Since $g'_{t_q}(S_{t_q})\leq 1$, it won't affect our bound for the dual local norm of the gradient estimator. And since the randomness of the stochastic function $g'_{t_q}$ is independent of all the randomness introduced in our algorithm and $\mathbb[E][g'_{t_q}(S_{t_q})] = g_{t_q}(S_{t_q})$, the new gradient estimator constructed by replacing $g_{t_q}(S_{t_q})$ with $g'_{t_q}(S_{t_q})$ is still an unbiased estimator. Thus, as the same as the proof of our algorithm for adversarial submodular bandit with partition matroid constraint, we have the same regret bound,
\[\mathcal{R}^{adv}_{1-1/e}(T) =\max_{S^*\subseteq G, |S^*|\leq k} \mathbb{E}\left[(1-\frac{1}{e})\sum_{t=1}^T g(S^*) - \sum_{t=1}^T g(S_t)\right]\leq O((k|G|)^{5/3}T^{2/3}\log T)\]
That is,
\[(1-\frac{1}{e})T\cdot \max_{S^*\subseteq G, |S^*|\leq k}g(S^*) - \mathbb{E}\left[\sum_{t=1}^T g(S_t)\right]\leq O((k|G|)^{5/3}T^{2/3}\log T)\]

\end{document}